\documentclass{article}

\usepackage[accepted]{icml2026}

\usepackage[english]{babel}
\usepackage{latexsym}
\usepackage{amsmath}
\usepackage{mathrsfs}
\usepackage{amssymb}
\usepackage{mathtools}
\usepackage[inline,shortlabels]{enumitem} 
\usepackage{bm}
\usepackage{datetime}

\usepackage{tikz}
\usepackage{listings}
\usepackage{mdframed}
\usepackage{pgfplots}
\usepackage{pgfplotstable}

\usepackage{dsfont}
\usepackage{color}
\usepackage{colortbl}
\usepackage{pifont}
\usepackage[font=small,singlelinecheck=off]{caption}
\usepackage{microtype} 
\usepackage{float}
\usepackage{xfrac} 
\usepackage{xspace}
\usepackage{booktabs}
\usepackage{blkarray}

\usepackage{hyperref}
\definecolor{mydarkblue}{rgb}{0,0.08,0.45}
\hypersetup{
    unicode=false,          
    pdftoolbar=true,        
    pdfmenubar=true,        
    pdffitwindow=false,     
    pdfstartview={FitH},    
    pdftitle={XXXXX},    
    pdfauthor={XXX},     
    pdfsubject={Bandits, Reinforcement Learning},   
    pdfcreator={Creator},   
    pdfproducer={Producer}, 
    pdfkeywords={bandits} {reinforcement learning} {policy gradient}, 
    pdfnewwindow=true,      
    colorlinks=true,       
    linkcolor=mydarkblue,          
    citecolor=mydarkblue,        
    filecolor=magenta,      
    urlcolor=cyan           
}
\usepackage{amsthm}
\usepackage{natbib}
\usepackage{nicefrac}
\usepackage{wrapfig}
\usepackage{pgfplots}

\usepackage[capitalize,noabbrev]{cleveref}

\usepackage{thmtools, thm-restate}

\declaretheorem{theorem}
\declaretheorem{proposition}

\declaretheorem{assumption}

\declaretheorem{remark}
\declaretheorem{definition}
\newtheorem{lemma}{Lemma}

\Crefname{equation}{Eq.}{Eqs.}
\Crefname{section}{Sec.}{Secs.}
\Crefname{appendix}{App.}{Apps.}
\Crefname{algorithm}{Alg.}{Algs.}
\Crefname{assumption}{Assn.}{Assn.}
\Crefname{theorem}{Thm.}{Thms.}
\Crefname{lemma}{Lemma.}{Lemmas.}
\Crefname{corollary}{Cor.}{Cor.}
\Crefname{proposition}{Prop.}{Prop.}

\allowdisplaybreaks

\newcommand{\ti}[1]{#1_{t,i}}
\newcommand{\tpi}[1]{#1_{t-1,i}}

\newcommand{\si}[1]{#1_{s,i}}

\newcommand{\normsq}[1]{\left\|#1 \right\|^{2}}
\newcommand{\norm}[1]{\left\|#1 \right\|}

\newcommand{\tht}{\theta_t}

\newcommand{\thtt}{\theta_{t+1}}

\def\1{\bm{1}}

\newcommand{\etat}{{\eta_t}}

\DeclareMathAlphabet{\mathsfit}{\encodingdefault}{\sfdefault}{m}{sl}
\SetMathAlphabet{\mathsfit}{bold}{\encodingdefault}{\sfdefault}{bx}{n}

\def\cI{{\mathcal{I}}}

\newcommand{\R}{\mathbb{R}}

\DeclareMathOperator*{\argmax}{arg\,max}

\def\FMAPG/{{FMA-PG}}
\def\DFMAPG/{{DFMA-PG}}
\def\SFMAPG/{{SFMA-PG}}

\newcommand{\NSD}{\texttt{NSD}}
\newcommand{\NormGD}{\texttt{Norm.GD}}
\newcommand{\SignGD}{\texttt{Sign GD}}
\newcommand{\SignCD}{\texttt{Sign CD-GS}}

\newcommand{\RMS}{\texttt{RMSProp}}
\newcommand{\Adam}{\texttt{Adam}}
\newcommand{\AdaGrad}{\texttt{AdaGrad}}
\newcommand{\AMSGrad}{\texttt{AMSGrad}}

\newcommand{\NS}{\textsf{NS}}
\newcommand{\NL}{\textsf{NL}}

\addtocontents{toc}{\protect\setcounter{tocdepth}{0}}

\icmltitlerunning{Convergence of Steepest Descent and Adam under Non-Uniform Smoothness}

\begin{document}

\twocolumn[
  \icmltitle{Convergence of Steepest Descent and Adam under Non-Uniform Smoothness}



  \icmlsetsymbol{equal}{*}

  \begin{icmlauthorlist}
    \icmlauthor{Sharan Vaswani}{SFU}
    \icmlauthor{Yifan Sun}{SB}
    \icmlauthor{Reza Babanezhad}{Samsung}
  \end{icmlauthorlist}

  \icmlaffiliation{SFU}{Simon Fraser University}
  \icmlaffiliation{SB}{Stony Brook University}
  \icmlaffiliation{Samsung}{Samsung AI, Montreal}

  \icmlcorrespondingauthor{Sharan Vaswani}{vaswani.sharan@gmail.com}
  
  \icmlkeywords{Adaptive Gradient Methods, Non-uniform Smoothness, Convergence Rates}

  \vskip 0.3in
]



\printAffiliationsAndNotice{}  

\begin{abstract}
Recent work has analyzed the convergence of first-order methods under non-uniform smoothness assumptions that better model the loss landscape in machine learning tasks. We generalize this assumption to objectives whose curvature is an affine function of the objective value. This property is satisfied by a broad class of problems, including logistic regression, generalized linear models with a logistic link function, softmax policy gradient in reinforcement learning, and a class of neural networks. Under this assumption and gradient domination conditions, we establish a general convergence rate for the steepest descent method, and deterministic, diagonal variants of RMSProp and Adam. Our results imply that for logistic regression on separable data and the softmax policy gradient objective, sign GD converges linearly and is provably faster than GD. Furthermore, we show that for a class of two-layer neural networks on separable data, RMSProp and Adam can converge at a linear rate with a constant step-size and momentum parameter. Finally, we present a lower bound demonstrating that, under our assumption, RMSProp and Adam are provably faster than AdaGrad, AMSGrad, gradient descent, and heavy-ball momentum.    
\end{abstract}
\section{Introduction}
\label{sec:introduction}
Recent works~\citep{zhang2019gradient,zhang2020improved, chen2023generalized,vankov2024optimizing,gorbunov2024methods,vaswani2025armijo,alimisis2025we} have studied the convergence of first-order methods under \textit{non-uniform smoothness}. Unlike the standard smoothness assumption, which imposes a global uniform upper-bound on the Hessian norm, non-uniform smoothness ($\NS$) upper-bounds the Hessian locally as a function of the parameter value. For example,~\citet{zhang2019gradient} proposed that the Hessian norm at any point can be bounded as an affine function of the gradient norm at that point. They empirically showed that this property holds during neural network training under standard loss functions, and leveraged this assumption to justify the role of gradient clipping.

Following this, there have been numerous works~\citep{chen2023generalized,li2023convex,vaswani2025armijo,alimisis2025we} relaxing this $\NS$ assumption, with the goal of better modeling the loss landscapes arising in machine learning problems. Under such $\NS$ assumptions, recent works have revisited the theoretical convergence rates of common first-order methods including gradient descent (GD) with adaptive step-sizes, heavy-ball momentum~\citep{vankov2024optimizing,gorbunov2024methods,vaswani2025armijo,hubler2024parameter}, and adaptive gradient methods such as Adam and RMSProp~\citep{li2023convergence,wang2024convergence,wang2024provable} (see~\cref{app:related} for a detailed review). 

We focus on a particular form of $\NS$ where the spectral norm of the Hessian is upper-bounded by an affine function of the objective value~\citep{vaswani2025armijo,alimisis2025we}. For twice-differentiable functions, this property strictly generalizes the $\NS$ assumption in~\citet{zhang2019gradient}. Furthermore, this property is satisfied by a broad class of problems, including logistic regression, generalized linear models with a logistic link~\citep{mei2021leveraging}, softmax policy gradient in reinforcement learning~\citep{mei2020global}, and certain two-layer neural networks~\citep{taheri2023fast,alimisis2025we}. 

Moreover, unlike the assumption in~\citet{zhang2019gradient}, the class of functions satisfying this alternative $\NS$ assumption is closed under finite-sums and affine transformations~\citep{alimisis2025we}. This enables an easier analysis of finite-sum losses prevalent in machine learning. We note that under this assumption,~\citet{vaswani2025armijo} analyzed GD and demonstrated the advantage of using an Armijo line-search over constant step-sizes. On the other hand,~\citet{alimisis2025we} have used this assumption to theoretically characterize the importance of learning rate warm-up. In this paper, we generalize this assumption to handle induced $(p,q)$ operator norms of the Hessian, and refer to the resulting property as $(H_0, H_1)$-$\NS$. 

In addition to this assumption, we consider objectives that satisfy a (possibly non-uniform) \L ojasiewicz ($\NL$) assumption. The $\NL$ condition is a gradient domination property implying that the gradient norm is lower-bounded in terms of the function sub-optimality. Hence, minimizing the gradient norm results in minimizing the function value. Consequently, the $\NL$ condition is widely used to analyze the global convergence of algorithms~\citep{karimi2016linear}. It is satisfied for convex objectives such as quadratics, the exponential loss on separable data, and structured non-convex losses including the softmax policy gradient objective~\citep{mei2020global}, matrix factorization~\citep{ward2023convergence} and sufficiently over-parameterized neural networks~\citep{liu2022loss,soltanolkotabi2018theoretical,zou2019improved,li2018learning}. 

Working under the $(H_0, H_1)$-$\NS$ and $\NL$ assumptions, we derive convergence guarantees for steepest descent, and deterministic, diagonal variants of $\RMS$ and $\Adam$. To our knowledge, these are the first such results. Our contributions are summarized as follows.

\textbf{Contribution 1:} In~\cref{sec:nus-structure}, we derive structural properties of functions satisfying $(H_0, H_1)$-$\NS$. In particular, we first generalize the results in~\citet{vaswani2025armijo,alimisis2025we} beyond Euclidean norms. This subsequently allows us to derive convergence rates for steepest descent algorithms without any dimension dependence. Furthermore, we prove that $(H_0, H_1)$-$\NS$ implies \textit{multiplicative Lipschitz} bounds for both the function and the gradient norm. This property enables a new framework for analyzing $\RMS$~\citep{tieleman2012lecture} and $\Adam$~\citep{kingma2014adam}. 

\textbf{Contribution 2:} In~\cref{sec:sd}, under our $\NS$ and $\NL$ assumptions, we analyze the convergence of normalized steepest descent ($\NSD$) methods including $\SignGD$ and \texttt{Normalized GD}. Our general theorem implies that for logistic regression on separable data and the softmax policy gradient objective, using $\SignGD$ with a constant step-size results in a dimension-free linear convergence rate matching GD with line-search~\citep{vaswani2025armijo}. For these applications, $\SignGD$ is provably faster than constant step-size GD~\citep{mei2020global,wu2024large}. Furthermore, our results demonstrate the effectiveness of $\SignCD$, a specific coordinate descent instantiation of $\NSD$. $\SignCD$ uses the Gauss–Southwell rule~\citep{nutini2015coordinate} to select the coordinate, and updates it using the sign of its gradient. When minimizing the exponential or logistic loss on separable data, $\SignCD$ converges linearly, matching the rate for a normalized variant of coordinate descent~\citep{axiotis2023gradient}. 

\textbf{Contribution 3}: In~\cref{sec:adaptive}, we derive convergence rates for the deterministic, diagonal variants of $\RMS$ and $\Adam$. Our results imply that for a class of two-layer neural networks on separable data, $\RMS$ and $\Adam$ can converge with a constant step-size, and do so at a dimension-free linear rate. Importantly, our analysis does not require any convexity or bounded gradient assumption. In contrast to our setting, previous works~\citep{li2023convergence,wang2024convergence,wang2024provable} consider non-convex functions satisfying the $\NS$ assumption in~\citet{zhang2019gradient}, and analyze both deterministic and stochastic variants of $\Adam$. For this class of functions, we show that our proof techniques can also be used to attain faster rates in the deterministic setting.  

\textbf{Contribution 4}: In~\cref{sec:lb}, we consider the one-dimensional logistic loss which satisfies our $\NL$ and $\NS$ assumptions,  and for which $\RMS$ and $\Adam$ achieve a linear convergence rate. For this example, we prove a sub-linear lower-bound for GD, heavy-ball momentum and deterministic variants of $\AdaGrad$~\citep{duchi2011adaptive} and $\AMSGrad$~\citep{reddi2019convergence}. Consequently, our results show that these methods are provably slower on this class of functions. This is the first such separation  for adaptive gradient methods, and provides theoretical justification for the practical dominance of \(\RMS\) and \(\Adam\) over \(\AdaGrad\) and \(\AMSGrad\).
\section{Problem Formulation}
\label{sec:problem}
We aim to solve the unconstrained minimization problem: $\min_{\theta \in \R^D} f(\theta)$. We define $\theta^* \in \arg\inf f(\theta)$ as an optimal solution and $f^* := \inf f(\theta)$ as the minimum function value. We make the following assumptions. 
\begin{assumption}
$f$ is twice-differentiable and non-negative i.e. for all $\theta$, $f(\theta) \geq 0$. 
\label{assn:non-negative}
\end{assumption}

\begin{assumption}{($(H_0, H_1)$-$\NS$)}
$f$ is $(H_0, H_1)$ non-uniform smooth if for constants $H_0 \geq 0, H_1 \geq 0$, and $p, q \geq 1$ s.t. $\frac{1}{p} + \frac{1}{q} = 1$, for all $\theta$, 
\begin{align}
\norm{\nabla^2 f(\theta)}_{p \to q} \leq H_0 + H_1 \, f(\theta) \,,    
\label{eq:nus}
\end{align}
where, for a matrix $A$, $\norm{A}_{p \to q} = \max_{\norm{x}_p \leq 1} \norm{A \, x}_{q}$. 
\label{assn:nus}
\end{assumption}
For twice-differentiable functions, recent work~\citep{vaswani2025armijo,alimisis2025we} has considered~\cref{assn:nus} with $(p,q) = (2,2)$, and proved that the resulting condition generalizes the non-uniform smoothness conditions in prior works~\citep{zhang2019gradient, zhang2020improved}. 

In particular, if a function is twice-differentiable and $(L_0, L_1)$-$\NS$ meaning that,
\begin{align}
\norm{\nabla^2 f(\theta)}_{2} \leq L_0 + L_1 \, \norm{\nabla f(\theta)}_{2}    \,,
\label{eq:zhang-nus}
\end{align}
then, it also satisfies~\cref{assn:nus}~\citep[Prop. 3]{vaswani2025armijo}.

Furthermore, note that if $H_1 = 0$ and $(p,q) = (2,2)$,~\cref{assn:nus} recovers the standard uniform smoothness condition as a special case. Consequently, common smooth objectives such as linear regression or logistic regression satisfy~\cref{assn:nus}. For example, if $X \in \R^{n \times D}$ is the feature matrix, and $y \in \R^n$ is the vector of measurements, then, the linear regression objective, $f(\theta) = \frac{1}{2n} \, \normsq{X \, \theta - y}$ is $(\frac{1}{n} \lambda_{\max}[X^T X], 0)$-$\NS$ where $\lambda_{\max} [A]$ is the maximum eigenvalue of the positive semi-definite matrix $A$. In this paper, we will be particularly interested in functions for which $H_1 > 0$. 

In addition to~\cref{assn:nus}, we consider functions that also satisfy a \L ojasiewicz or gradient domination condition. 
\begin{assumption}(\textsf{NL}) 
$f$ satisfies a non-uniform \L ojasiewicz condition if for $\tau \in (0,1]$, for all $\theta$, there exists $\mu(\theta) > 0$, such that, 
\begin{align}
\norm{\nabla f(\theta)}_{q} \geq \mu(\theta) \, [f(\theta) - f^*]^{\tau}.
\label{eq:pl}
\end{align}
\label{assn:pl}
\end{assumption}
\vspace{-4ex}
First, we note that~\cref{assn:pl} with $\tau = \frac{1}{2}$ and a uniform $\mu$ (for all $\theta$) is known as the Polyak \L ojasiewicz condition and generalizes the notion of strong-convexity. For $\tau = \frac{1}{2}$,~\cref{assn:pl} implies curvature near the optimum, and is related to the restricted secant inequality~\citep{zhang2013gradient} and error bound conditions~\citep{luo1993error} used to analyze the global convergence of algorithms~\citep{karimi2016linear} despite non-convexity.  

\subsection{Examples}
\label{sec:examples}
To motivate~\cref{assn:nus,assn:pl}, in~\cref{app:examples}, we prove that common convex objectives for supervised learning such as binary classification with the exponential loss, linear logistic regression and linear multi-class classification can satisfy these assumptions. While such examples have been studied in prior works, we generalize these results and present them to highlight the prevalence of our assumptions in machine learning problems. 

\begin{restatable}{proposition}{exponential}
Consider $n$ points where $x_i \in \R^d$ are the features and $y_i \in \{-1,1\}$ are the corresponding labels. Binary classification with an exponential loss,
\begin{align}
f(\theta) := \frac{1}{n} \, \sum_{i = 1}^{n} \exp(-y_i \langle x_i, \theta \rangle)    \,,
\end{align}
satisfies~\cref{assn:non-negative} and~\cref{assn:nus} with $H_0 = 0$ and $H_1 = \max_i \normsq{x_i}_q$. Furthermore, if the data is separable with a normalized margin $\gamma_p := \max_{\theta} \min_i \frac{y_i \, \langle x_i, \theta \rangle}{\norm{\theta}_p} > 0$, then, $f(\theta)$ satisfies~\cref{assn:pl} with $\tau = 1$, $\mu = \gamma_p$ and $f^* = 0$. 
\label{prop:exponential}
\end{restatable}
Note that the above example is not uniform smooth on an unbounded domain, but satisfies~\cref{assn:nus}. 
\begin{restatable}{proposition}{logistic}
Consider $n$ points where $x_i \in \R^D$ are the features and $y_i \in \{-1,1\}$ are the corresponding labels. Logistic regression with the objective, 
\begin{align}
f(\theta) := \frac{1}{n} \, \sum_{i = 1}^{n} \ln(1 + \exp(-y_i \langle x_i, \theta \rangle))    
\end{align}
satisfies~\cref{assn:non-negative} and~\cref{assn:nus} with $H_0 = 0$ and $H_1 =\max_i \normsq{x_i}_q$. Furthermore, if the data is separable with a normalized margin $\gamma_p := \max_{\theta} \min_i \frac{y_i \, \langle x_i, \theta \rangle}{\norm{\theta}_p} > 0$ then, 

\textbullet\ For $f(\theta) \leq \frac{\ln(2)}{n}$, $f$ satisfies~\cref{assn:pl} with $\tau = 1$, $\mu = \frac{\gamma_p}{2}$ and $f^* = 0$ 

\textbullet\ Else, if $f(\theta) > \frac{\ln(2)}{n}$, then, $\norm{\nabla f(\theta)}_q \geq \frac{\gamma_p}{3n}$. 
\label{prop:logistic}
\end{restatable}
Note that the logistic regression objective is also uniform smooth, meaning that it simultaneously satisfies~\cref{assn:nus} with $H_0 = \frac{1}{4n} \, \lambda_{\max}[X^T X]$ and $H_1 = 0$, where $X \in \R^{n \times D}$ is the corresponding feature matrix. The above results generalize those in~\citet{vaswani2025armijo} from $(p,q) = (2,2)$ to general pairs of dual norms. We defer the result for multi-class classification to~\cref{prop:multiclass} in~\cref{app:examples}.  

\cref{assn:nus,assn:pl} are also satisfied by certain non-convex functions such as the softmax policy gradient objective in reinforcement learning. In particular, we prove the following result for the multi-armed bandit problem with known deterministic rewards. This setting is often used as a testbed to analyze policy gradient methods~\citep{mei2020global,lu2024towards}.   

\begin{restatable}{proposition}{bandit}
Given a multi-armed bandit problem with $K$ arms and known deterministic rewards $r \in [0,1]^{K}$,  consider softmax policies $\pi_\theta \in \Delta_K$ parameterized by $\theta \in \R^K$ s.t. $\pi_\theta(a) = \nicefrac{\exp(\theta(a))}{\sum_{a'} \exp(\theta(a'))}$. The softmax policy gradient objective is given by 
\begin{align}
f(\theta) := r(a^*) - \langle \pi_\theta, r \rangle \,,
\end{align}
where $a^* := \argmax_{a \in [K]} r(a)$ is the optimal arm. $f(\theta)$ satisfies~\cref{assn:non-negative} and~\cref{assn:nus} with (i) $H_0 = 0$, $H_1 = 24$ for $p = q = 2$ and (ii) $H_0 = 0$, $H_1 = 6$ for $p = \infty, q = 1$ and $p = 1, q = \infty$. Furthermore, $f(\theta)$ satisfies~\cref{assn:pl} for all $q \geq 1$ with $\tau = 1$, $\mu(\theta) = \pi_\theta(a^*)$ and $f^* = 0$. 
\label{prop:bandits}
\end{restatable}
The above result generalizes that in~\citet{mei2021leveraging} beyond Euclidean norms. By following a similar analysis as~\citet[Proposition 7]{vaswani2025armijo}, we can extend this result from bandits to Markov decision processes. The proposition below shows that certain two-layer networks (with restrictions on the activation function) satisfy~\cref{assn:nus,assn:pl}.
\begin{restatable}{proposition}{nn}
Consider $n$ points where $x_i \in \R^D$ are the features and $y_i \in \{-1,1\}$ are the corresponding labels, and a neural network,
\begin{align}
\Phi(\theta,x) := \sum_{j = 1}^m a_j \, \sigma(\langle \theta_j, x \rangle) \,,
\end{align}
where $a_j$ are fixed, $m$ is the width of the layer and $\sigma$ is the activation function. Consider the case when $\sigma$ is smooth s.t. for all $t$, $|\sigma''(t)| \leq M$ and has bounded derivatives i.e. there exists positive constants $\alpha_1, \alpha_2$ such that $\alpha_1 \leq |\sigma'(t)| \leq \alpha_2$. Consider the loss 
\begin{align}
f(\theta) := \frac{1}{n} \sum_{i = 1}^{n} \, g\left(y_i \, \Phi(\theta, x_i \right)) \,, 
\end{align}
where, $g:\R\to\R$ is differentiable everywhere and for all $s$, $\quad g(s) \geq 0$, $g'(s)\leq 0$, $\frac{|g'(s)|}{g(s)} \in [c_1, c_2]$ and $|g''(s)| \leq c'_2 \, g(s)$.  $f$ satisfies~\cref{assn:non-negative} and~\cref{assn:nus} with  $H_0 = 0$ and $H_1 = \left[c_2 \,  M \, \norm{a}_1 + c'_2 \,  \alpha_2^2 \, \|a\|_q^2 \right] \, \max_i \normsq{x_i}_{q}$. Furthermore, if the data is linearly separable with a normalized margin $\gamma_p := \max_{\theta} \min_i \frac{y_i \, \langle x_i, \theta \rangle}{\norm{\theta}_p} > 0$ then, $f(\theta)$ satisfies~\cref{assn:pl} with $\tau = 1$, $\mu = \frac{c_1 \, \alpha_1 \, \gamma \, \|a\|_2^2}{\|a\|_p}$ and $f^* = 0$. 
\label{prop:nn} 
\end{restatable}
\vspace{-1ex}
The conditions on $\sigma(t)$ are satisfied for a smoothed variant of the leaky ReLU function, whereas the condition on $g$ is satisfied by the exponential loss. This result generalizes~\citet[Lemmas 3 \& 5]{taheri2023fast} beyond $\ell_2$ norms and the exponential loss.

Furthermore, recent work~\citep[Proposition 3.3]{alimisis2025we} shows that two-layer neural networks with $\ell_2$-regularization and weaker assumptions on the activation function satisfy~\cref{assn:nus} for $p = q = 2$, and $H_0 \neq 0$ and $H_1 \neq 0$. In addition, this work provides some empirical evidence verifying that~\cref{assn:nus} holds when training language models. On the other hand, we note that sufficiently over-parameterized neural networks~\citep{liu2022loss,soltanolkotabi2018theoretical,zou2019improved,li2018learning} are known to satisfy~\cref{assn:pl} with $\tau = \frac{1}{2}$ and $(p,q) = (2,2)$. 

Finally, in~\cref{prop:glm} in~\cref{app:examples}, we show that generalized linear models with the logistic link function also satisfy~\cref{assn:nus} with non-zero $H_0$ and $H_1$, and satisfy~\cref{assn:pl} with $\tau = \frac{1}{2}$. 
\section{Properties of $(H_0, H_1)$-$\NS$ Functions}
\label{sec:nus-structure}
In this section, we develop properties of $(H_0, H_1)$-$\NS$ functions that will be crucial in the subsequent analyses. We defer all the proofs to~\cref{app:nus-prop}. 

In~\cref{lemma:nus-grad-gen}, we first prove that for functions satisfying~\cref{assn:non-negative,assn:nus}, the gradient can be bounded in terms of the function value, i.e., for all $\theta$, 
\begin{align}
\norm{\nabla f(\theta)}_{q} \leq \sqrt{2 H_0 \, f(\theta) + H_1 \, [f(\theta)]^2}  \label{eq:grad-f}
\end{align}
If $H_1 = 0$ and $(p,q) = (2,2)$, the above inequality implies that the squared Euclidean norm of the gradient is bounded by the function value, a standard result for uniformly smooth functions. On the other hand, if $H_0 = 0$, then~\cref{eq:grad-f} simplifies to $\norm{\nabla f(\theta)}_q \leq \sqrt{H_1} \, f(\theta)$, implying that $\ln(f(\theta))$ is $\sqrt{H_1}$-Lipschitz. 

In~\cref{lemma:func-multi-lip-gen}, we generalize this property and prove that an appropriately shifted $f$ is uniformly Lipschitz in a \textit{multiplicative sense}, i.e. $\forall y,x$ and $H_1 > 0$,
\begin{equation}
f(y) + \frac{H_0}{H_1}  \leq \left(f(x) + \frac{H_0}{H_1}\right) \exp\left(\sqrt{H_1} \norm{y - x}_p \right)     
\label{eq:f-multi-lip}
\end{equation}
If $H_0 = 0$,~\cref{eq:f-multi-lip} implies that as $y \to x$, the ratio $\nicefrac{f(y)}{f(x)} \to 1$. This multiplicative Lipschitzness of the function will be helpful in the subsequent analysis. Moreover, it enables us to prove that the gradients are Lipschitz in the usual additive sense, and the gradient norms are Lipschitz in a multiplicative sense similar to~\cref{eq:f-multi-lip}. In particular, in~\cref{lemma:grad-lip-gen,lemma:grad-multi-lip-gen}, we prove that $\forall y,x$  such that $\norm{y - x}_p \leq \frac{1}{\sqrt{H_1}}$ and $c > 0$, 
\begin{align}
\norm{\nabla f(y) - \nabla f(x)}_{q} & \leq e \, [H_0 + H_1 \, f(x)] \,  \norm{y - x}_p \label{eq:grad-lip} 
\end{align}
\begin{align}
\left(\norm{\nabla f(y)}_q  + c\right) & \leq \left(\norm{\nabla f(x)}_q + c \right) \label{eq:grad-multi-lip}  \\
& \times \exp\left((H_0 + H_1  f(x)) \frac{e \, \norm{y - x}_p}{c} \right) \nonumber 
\end{align}
Hence,~\cref{eq:f-multi-lip,eq:grad-multi-lip} enable bounding both the (appropriately shifted) function and gradient norm at $y$ in terms of $x$, a nearby point in the $\ell_p$ norm. These properties will be particularly important when we analyze algorithms. Finally, we prove the \textit{descent lemma} for $\NS$ functions, implying that they can be upper-bounded in terms of a quadratic. In particular, if $\norm{y - x}_p \leq \frac{1}{\sqrt{H_1}}$, 
\begin{align}
f(y) \leq f(x) & + \langle \nabla f(x), y - x \rangle \nonumber \\ & + \left(H_0+ H_1 \, f(x) \right) \, \normsq{y - x}_p \label{eq:descent-ineq}
\end{align}
\cref{eq:descent-ineq} will serve as the starting point of all our analyses in~\cref{sec:sd,sec:adaptive}. Finally, we emphasize that unlike~\cref{eq:f-multi-lip},~\cref{eq:grad-lip,eq:grad-multi-lip,eq:descent-ineq} are non-uniform in that the constant depends on $f(x)$, and hold for $y,x$ that are sufficiently close in the $\ell_p$ norm. 
\section{Convergence of Steepest Descent}
\label{sec:sd}
In this section, we characterize the convergence of steepest descent methods on functions satisfying~\cref{assn:non-negative,assn:nus,assn:pl}. We then highlight some important practical consequences of our result to softmax policy gradient and logistic regression for separable data. 

We focus on normalized Steepest Descent ($\NSD$)~\citep{boyd2004convex} which has the following update:
\begin{align}
\thtt &= \tht - \etat \, d_t & (\NSD) \nonumber \\
\text{ where, } d_t & := \argmax_{\norm{d}_p \leq 1} \langle d, \nabla f(\tht) \rangle \label{eq:sd-update}
\end{align}
We will be particularly interested in $(p,q) = (\infty,1)$, $(2,2)$ and $(1, \infty)$. In these special cases, the update in~\cref{eq:sd-update} can be simplified and recovers sign gradient descent (Sign GD), normalized gradient descent (Norm. GD) and sign coordinate descent with the Gauss-Southwell rule (Sign CD-GS) respectively (see~\cref{prop:nsd-derivation} in~\cref{app:sd} for a proof). In particular, we define $\nabla_t := \nabla f(\tht)$ with $\nabla_{t,i}$ denoting coordinate $i$ of this vector, use $\text{sign}(\nabla_t) \in \{-1,0, 1\}^{D}$ to denote the element-wise sign operation with $\text{sign}(0) := 0$ and let $e_i$ denote the \(i\)-th standard basis vector. For, 

\textbullet\ $(p,q) = (\infty,1)$, $\thtt = \tht - \etat \, \text{sign}(\nabla_{t})$ ($\SignGD$)

\textbullet\ $(p,q) = (2,2)$, $\thtt = \tht - \etat \, \frac{\nabla_t}{\norm{\nabla_t}_2}$ ($\NormGD$)

\textbullet\ $(p,q) = (1,\infty)$, $\thtt = \tht - \etat \, \text{sign}(\nabla_{t,i_t}) \, e_{i_t}$, where, $i_t \in \argmax_{i \in [D]}|\nabla_{t,i}|$ ($\SignCD$)

We will subsequently use these special cases while discussing the practical implications of $\NSD$. Next, we present the theorem analyzing the convergence of $\NSD$ on functions satisfying~\cref{assn:non-negative,assn:nus} and~\cref{assn:pl} with $\mu(\theta) = \mu$ and $f^* = 0$. 
\begin{restatable}{theorem}{nsd}
Under~\cref{assn:non-negative},~\cref{assn:nus} and~\cref{assn:pl} with $\mu(\theta) = \mu$ for all $\theta$, $f^* = 0$, $\NSD$ converges as:

\textbullet\ if $\epsilon > \frac{H_0}{H_1}$, using $\etat = \eta = O(1)$ guarantees that after $T = O\left( \ln\left(\frac{1}{\epsilon}\right)\right)$ iterations, $f(\theta_{T+1}) \leq \epsilon$; 

\textbullet\ else, using $\etat = \eta = O(\epsilon^{\tau})$ guarantees that after $T = O \left( \frac{1}{\epsilon^{\tau}} \right)$ iterations, $f(\theta_{T+1}) \leq \epsilon$. 
\label{thm:nsd-pl}
\end{restatable}
In order to interpret the above theorem, let us first consider the setting corresponding to $H_0 > 0$ and $H_1 = 0$. This corresponds to uniform smoothness and implies that $\NSD$ attains an $O\left(\frac{1}{\epsilon^{\tau}}\right)$ rate. For strongly-convex quadratics, $\tau = \frac{1}{2}$, and in this case, the $\Theta\left(\frac{1}{\sqrt{\epsilon}}\right)$ convergence for $\NSD$ (see~\cref{prop:nsd-lb} for the corresponding lower-bound) is slower than the linear convergence rate of GD. On the other hand, if we consider the other extreme -- when $H_0 = 0$ and $H_1 > 0$, $\NSD$ converges at a linear rate for all values of $\tau$ and $\epsilon$. As mentioned in~\cref{sec:examples}, such a property is satisfied by binary classification with the exponential loss on separable data. In this case,  GD attains a sublinear $\Theta\left(\frac{1}{\epsilon}\right)$ rate~\citep{soudry2018implicit}, and is provably slower than $\NSD$. 

In general, for non-zero values of $H_0, H_1$, $\NSD$ results in a faster linear convergence rate when an $O\left(\frac{H_0}{H_1}\right)$ sub-optimality is acceptable. On the other hand, for $\epsilon < \frac{H_0}{H_1}$, $\NSD$ converges in two phases -- a fast, first phase when the method converges to an $O\left(\frac{H_0}{H_1}\right)$ neighbourhood, followed by a slower second phase. Intuitively, in the second phase, as the iterates get closer to the optimum, guaranteeing convergence requires using a smaller step-size so as not to ``overshoot'' the minimizer. The convergence in this second phase depends on the value of $\tau$ and is slower as $\tau \to 1$. On the other hand, for losses such as the exponential loss on separable data, the minimum is achieved as $\norm{\theta} \to \infty$. In this case, since there is no finite minimizer, $\NSD$ can use a constant step-size throughout, resulting in a faster linear rate. 

\paragraph{Proof Sketch:} Using~\cref{eq:descent-ineq} for the $\NSD$ update at iteration $t$ with $\eta \leq \frac{1}{\sqrt{H_1}}$ and noting that $\normsq{d_t}_p \leq 1$, 
\begin{align}
f(\thtt) & \leq f(\tht) - \eta \,  \langle \nabla f(\tht), d_t \rangle + \left(H_0 + H_1 \, f(\tht) \right) \, \eta^2 \nonumber
\end{align}
Using the definition of the dual norm to simplify $\langle \nabla_t, d_t \rangle = \norm{\nabla_t}_q$ and~\cref{assn:pl} with $f^* = 0$ to bound the gradient in terms of the function value, we get that,
\begin{align}
f(\thtt) & \leq f(\tht) - \eta \,  \mu \, [f(\tht)]^\tau +\left(H_0 + H_1 \, f(\tht) \right) \, \eta^2 \,.    \nonumber
\end{align}
The subsequent proof proceeds in two phases. We define $T_0$ as the first iteration s.t. $f(\theta_{T_0}) < \max\left\{\epsilon, \frac{H_0}{H_1} \right\}$. 

\textit{Phase 1:} For $t < T_0$, $f(\tht) \geq \max\left\{\epsilon, \frac{H_0}{H_1} \right\}$ and consequently, $H_0 + H_1 \, f(\tht) \leq 2 \, H_1 \, f(\tht)$. Using this relation to simplify the above inequality, 
\begin{align}
f(\thtt) & \leq f(\tht) - \eta \,  \mu \, [f(\tht)]^\tau + 2 \, H_1 \, f(\tht) \, \eta^2 \,.   
\label{eq:sd-main-inter-1}
\end{align}
By setting an appropriate $\eta = O(1)$, we inductively prove that $f(\tht) < f(\theta_1)$ for all $t < T_0$, and bound $[f(\tht)]^\tau$ in~\cref{eq:sd-main-inter-1} by $\nicefrac{f(\tht)}{[f(\theta_1)]^{1 - \tau}}$. Solving the resulting recursion using~\cref{lemma:sd-p1} immediately implies a linear convergence rate. 

\textit{Phase 2:} When $\epsilon < \frac{H_0}{H_1}$, consider $t > T_0$ s.t. $f(\tht) < \frac{H_0}{H_1}$. In this case, we bound $H_0 + H_1 \, f(\tht) \leq 2 \, H_0$, and show that setting an appropriate  $\eta = O(\epsilon^\tau)$ implies descent meaning that $f(\thtt) \leq f(\tht) \leq \frac{H_0}{H_1}$ for all $t > T_0$. Solving the resulting recursion using~\cref{lemma:sd-p2} with $r = \tau$ implies an $O(1/\epsilon^\tau)$ rate. \qed 

\begin{remark}
The $f^* = 0$ assumption is made w.l.o.g. In particular, if $f^* \neq 0$, we can analyze the shifted function $h(\theta) = f(\theta) - f^*$ s.t. $h^* = 0$ and $\nabla h(\theta) = \nabla f(\theta)$. Since $\NSD$ and all subsequent algorithms only depend on the gradients, they produce the same iterates on both $h$ and $f$. In~\cref{lemma:shift-nus}, we prove that if $f$ is $(H_0, H_1)$-$\NS$, then, $h$ is $(H_0 + H_1 \, f^*, H_1)$-$\NS$. 
\end{remark}

\subsection{Implications}
\label{sec:nsd-implications}
\cref{thm:nsd-pl} applies to all objectives in~\cref{sec:examples} and implies the efficacy of $\SignGD$, $\SignCD$ and $\NormGD$. Below, we highlight some practical consequences. 

\textit{Implication 1:} Recall from~\cref{sec:examples} that the softmax policy gradient objective satisfies~\cref{assn:nus} with $H_0 = 0$ and~\cref{assn:pl} with $\tau = 1$ and $f^* = 0$, though with a non-uniform $\mu(\theta)$. In~\cref{app:sd}, we handle this non-uniformity and prove the following corollary. 
\begin{restatable}{corollary}{bandits}
For the multi-armed bandit problem in~\cref{prop:bandits}, $\NSD$ with a uniform initialization i.e. $\forall a$, $\pi_{\theta_1}(a) = \nicefrac{1}{K}$ and $\eta = O(1)$ requires $T = O\left(\ln\left(\nicefrac{1}{\epsilon}\right) \right)$ iterations to guarantee $\langle \pi_{\theta_{T+1}}, r \rangle \geq r(a^*) - \epsilon$.
\label{cor:bandits}
\end{restatable}
The above result is in contrast to constant step-size GD which can only attain an $\Omega\left(\frac{1}{\epsilon}\right)$ convergence rate for this problem~\citep[Theorem 9]{mei2020global}. Since $\NSD$ includes $\NormGD$,~\cref{cor:bandits} recovers the result in~\citet{mei2021leveraging}. Moreover, since $\NSD$ also includes $\SignGD$ and $\SignCD$, the above result implies that these methods can match the convergence rate of algorithms designed for this specific problem, including GD with a line-search~\citep{lu2024towards,vaswani2025armijo}, GD with specific increasing step-sizes~\citep{liu2024elementary}, natural policy gradient~\citep{kakade2002approximately,xiao2022convergence} and mirror descent with a log-sum-exp mirror map~\citep{asad2024fast}. 

\textit{Implication 2:} Logistic regression on separable data is a canonical example in machine learning, and has been the focus of recent works~\citep{wu2024large,zhang2025minimax} analyzing the impact of large constant and adaptive step-sizes. More recently,~\citet{vaswani2025armijo} prove that GD with an Armijo line-search can attain a linear convergence rate for this problem. 

We now prove that $\NSD$ and consequently, $\SignGD$ and $\NormGD$ can match this linear rate while using a constant step-size. We first note that the objective in~\cref{prop:logistic} satisfies~\cref{assn:pl} with $\tau = 1$ only when the loss is below $\nicefrac{\ln(2)}{n}$. Consequently,~\cref{thm:nsd-pl} does not directly apply to this objective, and we handle this in~\cref{app:sd},  proving the following theorem.  
\begin{restatable}{theorem}{nsdlog}
For logistic regression on separable data with the margin $\gamma_p$ (see~\cref{prop:logistic}), $\NSD$ with $\theta_1 = 0$, $\etat = \eta = O(1)$ and $T = O\left(\frac{H_1}{\gamma_p^2} \, \left[n^2 + \ln\left(\frac{1}{n \epsilon}\right)\right]\right)$ iterations guarantees that  $f(\theta_{T+1}) \leq \epsilon$. 
\label{thm:nsd-logreg}
\end{restatable}
We note that after $O\left(\nicefrac{n^2}{\gamma_p^2}\right)$ burn-in iterations, the loss is below the $\frac{\ln(2)}{n}$ threshold, the objective satisfies~\cref{assn:pl} and consequently, $\NSD$ converges at a linear rate. Finally, we note that for this objective,~\citet{axiotis2023gradient} prove that a normalized variant of coordinate descent with greedy Gauss-Southwell selection achieves a linear convergence rate.~\cref{thm:nsd-logreg} implies that $\SignCD$, which uses the Gauss–Southwell rule to select the coordinate, and updates it using the sign of its gradient, is another example of a coordinate descent method that can achieve linear convergence. 
\section{Convergence of $\RMS$ \& $\Adam$}
\label{sec:adaptive}
In this section, we provide general convergence theorems for $\RMS$ and $\Adam$ on functions satisfying~\cref{assn:non-negative,assn:nus,assn:pl}. In~\cref{sec:rms}, we first analyze $\RMS$, highlighting the additional challenges compared to $\NSD$. In~\cref{sec:adam}, we turn to $\Adam$ and explain how to handle the additional momentum term. Subsequently, in~\cref{sec:adaptive-implications}, we highlight some practical implications of our results. Finally, in~\cref{sec:gnus}, we provide an additional result characterizing the stationary point convergence of $\Adam$ for non-convex, $(L_0, L_1)$ $\NS$ functions~\citep{zhang2019gradient} that do not not necessarily satisfy~\cref{assn:pl}. 

\subsection{$\RMS$}
\label{sec:rms}
We analyze $\RMS$~\citep{tieleman2012lecture} whose update is given as: if $\ti{g} := \nabla_{t,i}$, $v_{0,i} = 0$ for all $i \in [D]$ and for $\delta \geq 0$ and $t \geq 1$, 
\begin{align}
\thtt &= \tht - \etat \, d_t \, \text{ s.t., } \forall i \in [D]\,, \; d_{t,i} = \frac{\ti{g}}{\ti{\sqrt{v}} + \delta} \label{eq:rmsprop-update-coordinate}  \\
\ti{v} & = (1-\beta) \, \sum_{s = 1}^t \beta^{t -s} \, \si{g^2}  = \beta \, \tpi{v} + (1-\beta) \, \ti{g^2} \nonumber
\end{align}
We now analyze the convergence of $\RMS$.  
\begin{restatable}{theorem}{rms}
Under~\cref{assn:non-negative},~\cref{assn:nus} and~\cref{assn:pl} with $f^* = 0$, $\mu(\theta) = \mu$ for all $\theta$, $\RMS$ with $\delta = 0$ converges as:

\textbullet\ If $\epsilon > \frac{H_0}{H_1}$, using $\etat = \eta = O(1)$ guarantees that after $T = O\left( \ln\left(\frac{1}{\epsilon}\right)\right)$ iterations, $f(\theta_{T+1}) \leq \epsilon$. 

\textbullet\ Else, using $\etat = \eta = O(\epsilon^{2 \, \tau})$ guarantees that $f(\theta_{T+1}) \leq \epsilon$ after $T$ iterations, where,

\hspace{3ex} \textbullet\ If $\tau \leq \frac{1}{2}$, $T = O \left( \frac{1}{\epsilon^{2 \tau}} \right)$. 

\hspace{3ex} \textbullet Else, if $\tau > \frac{1}{2}$, $T = O \left( \frac{1}{\epsilon^{4 \tau - 1}} \right)$. 
\label{thm:rms-pl}
\end{restatable}
\paragraph{Proof Sketch:} We use~\cref{eq:descent-ineq} with $p = \infty$ and $q = 1$ for the $\RMS$ update at iteration $t$. Noting that $\norm{d_t}_\infty \leq \frac{1}{\sqrt{1 - \beta}}$ (see \cref{lemma:rms-descent}) and using an appropriate $\eta$ gives us the following inequality,  
\begin{align}
f(\thtt) & \leq f(\tht) - \eta \, \langle \nabla_t, d_t \rangle +  \bar{L}_t \, \frac{\eta^2}{1 - \beta} \,,   \nonumber
\end{align}
where $\bar{L}_t := H_0 + H_1 \, f(\tht)$. To lower-bound $\langle \nabla_t, d_t \rangle = \sum_i \frac{\ti{g^2}}{\ti{\sqrt{v}}}$, we prove~\cref{lemma:rms-precond-bound}, which uses the Cauchy-Schwarz inequality with
the $v_{t,i}$ update to get that,  
\begin{align}
\langle \nabla_t, d_t \rangle & \geq \norm{\nabla_t}_1 \, \underbrace{\frac{\norm{\nabla_t}_1}{\sqrt{1 - \beta} \, \sum_{j = 0}^{t-1} (\sqrt{\beta})^{j} \, \norm{\nabla_{t-j}}_1}}_{:= (*)}    \nonumber
\end{align}
Combining the above inequalities, we get that, 
\begin{align}
f(\thtt) & \leq f(\tht) - \eta \, \norm{\nabla_t}_{1} \, (*) + \bar{L}_t \, \frac{\eta^2}{1 - \beta}\,. \label{eq:rms-main-1}
\end{align}
The (*) term quantifies the effect of the $\RMS$ preconditioner, and involves the ratio of gradients evaluated at different points. To complete the proof, we split the analysis in two phases similar to~\cref{thm:nsd-pl}, and define $T_0$ as the first iteration s.t. $f(\theta_{T_0}) < \max\left\{\epsilon, \frac{H_0}{H_1} \right\}$.  

\textit{Phase 1:} Consider $t < T_0$. In this phase, we inductively prove that $\RMS$ results in descent and consequently, $f(\tht) \leq f(\theta_1)$ for all $t < T_0$. In particular, given the inductive hypothesis at iteration $t$, we use~\cref{eq:grad-multi-lip} for $x = \tht$, $y = \theta_{t-1}$, $c > 0$ to relate $\norm{\nabla_t}_1$ and $\norm{\nabla_{t-1}}_1$,
\begin{align}
\norm{\nabla_{t-1}}_1 + c & \leq (\norm{\nabla_{t}}_1 + c) \, \nonumber\\ & \times \exp\left((H_0 + H_1 \, f(\theta_1)) \, \frac{e \, \norm{\theta_{t-1} - \tht}_\infty}{c} \right) \,.    \nonumber
\end{align}
Using a sufficiently small $\eta$ and successively using the above inequality for $j \in \{0,1,\ldots,t-1\}$ enables us to obtain the following bound for a constant $a > 0$ that depends on $\beta$,
\begin{align}
\norm{\nabla_{t-j}}_1 & \leq (\norm{\nabla_{t}}_1 + c) \, \exp(a \, j) \nonumber
\end{align}
\begin{align}
\implies &  \sum_{j = 0}^{t-1} (\sqrt{\beta})^{j} \, \norm{\nabla_{t-j}}_1 \leq (\norm{\nabla_{t}}_1 + c)  \, \sum_{j = 0}^{t-1} (\sqrt{\beta} \, e^a)^j  \nonumber
\end{align}
Ensuring that $\sqrt{\beta} \, \exp(a) < 1$ enables us to bound the geometric series. Simplifying and using~\cref{assn:pl} with the appropriate value of $c$ gives the final bound on (*),
\begin{align}
(*) \geq \Omega \left( \frac{\norm{\nabla_t}_1}{\norm{\nabla_t}_1 + c} \right) \geq \Omega \left( \left[1 + \left(\frac{H_0}{H_1} \, \frac{1}{f(\tht)}\right)^{\tau} \right]^{-1} \right)\,. \nonumber
\end{align}
Since $f(\tht) \geq \nicefrac{H_0}{H_1}$ for all $t$ in Phase 1, $(*) = \Omega(1)$ and $\bar{L}_t = O(f(\tht))$. Simplifying~\cref{eq:rms-main-1} for Phase 1, and using~\cref{assn:pl} to simplify the $\norm{\nabla_t}_1$ term, we get that,  
\begin{align}
f(\thtt) & \leq f(\tht) - \norm{\nabla_t}_1 \, O(\eta) + O(\eta^2) \, f(\tht) \nonumber\\
\implies f(\thtt) & \leq  f(\tht) - O(\eta) \, [f(\tht)]^\tau + O(\eta^2) \, f(\tht) \nonumber
\end{align}
Setting an appropriate $\eta$ ensures descent and completes the induction. The rest of the proof for Phase 1 is the same for~\cref{thm:nsd-pl}, and results in linear convergence. 

\textit{Phase 2}: Similar to Phase 1, we inductively prove that for all $t > T_0$, $f(\thtt) \leq f(\tht) \leq f(\theta_{T_0}) \leq \frac{H_0}{H_1} \leq f(\theta_1)$. For this phase, since $f(\tht) \leq \nicefrac{H_0}{H_1}$ by the inductive hypothesis, we get that $(*) = \Omega([f(\tht)]^\tau)$ and $\bar{L}_t = O(1)$. Simplifying~\cref{eq:rms-main-1} analogous to~\cref{thm:nsd-pl},
\begin{align}
f(\thtt) & \leq f(\tht) - O(\eta) \, [f(\tht)]^{2\tau} + O(\eta^2) \nonumber
\end{align}
The rest of the proof for Phase 2 proceeds as that for~\cref{thm:nsd-pl}, with the only change being the $2 \tau$ exponent on $f(\tht)$ (rather than $\tau$ for the $\NSD$ proof). Solving the above recursion using~\cref{lemma:sd-p2} finishes the proof. \qed

Before we interpret the above result, we first show how to handle the momentum term in $\Adam$, and prove that it attains the same convergence rate as $\RMS$. 

\subsection{Adam}
\label{sec:adam}
We are now ready to analyze $\Adam$~\citep{kingma2014adam} without bias correction\footnote{The effect of bias correction decays exponentially fast.}. The resulting update is given as follows: if $\ti{g} := \nabla_{t,i}$, $v_{0,i} = 0$ and $m_{0,i} = g_{1,i}$ for all $i \in [D]$, then,  for $\delta \geq 0$ and $t \geq 1$,
\begin{align}
\thtt &= \tht - \etat \, d_t \, \text{ s.t., } \forall i \in [D] \;, d_{t,i} = \frac{\ti{m}}{\ti{\sqrt{v}} + \delta} \label{eq:adam-update-coordinate}  \\
\ti{m} & = (1-\beta_1) \, \sum_{s = 1}^t \beta_1^{t -s} \, \si{g}  = \beta_1 \, \tpi{m} + (1-\beta_1) \, \ti{g} \nonumber \\
\ti{v} & = (1-\beta_2) \, \sum_{s = 1}^t \beta_2^{t -s} \, \si{g^2}  = \beta_2 \, \tpi{v} + (1-\beta_2) \, \ti{g^2} \nonumber
\end{align}
In~\cref{sec:adam}, we characterize the convergence of $\Adam$ and prove the following theorem. 
\begin{restatable}{theorem}{adam}
Under~\cref{assn:non-negative},~\cref{assn:nus} with $H_0 \geq 0$ and $H_1 > 0$, and~\cref{assn:pl} with $f^* = 0$ and $\mu(\theta) = \mu$ for all $\theta$, $\Adam$ with the update in~\cref{eq:adam-update-coordinate} with $\beta_1 \leq \beta_2$ has the following convergence rate: 

\textbullet\ If $\epsilon > \frac{H_0}{H_1}$, using $\etat = \eta = O(1)$ guarantees that after $T = O\left( \ln\left(\frac{1}{\epsilon}\right)\right)$ iterations, $f(\theta_{T+1}) \leq \epsilon$. 

\textbullet\ Else, using $\etat = \eta = O(\epsilon^{2 \, \tau})$ guarantees that $f(\theta_{T+1}) \leq \epsilon$ after $T$ iterations, where,

\textbullet\ If $\tau \leq \frac{1}{2}$, $T = O \left( \frac{1}{\epsilon^{2 \tau}} \right)$. 

\textbullet Else, if $\tau > \frac{1}{2}$, $T = O \left( \frac{1}{\epsilon^{4 \tau - 1}}  \right)$. 

\label{thm:adam-pl}
\end{restatable}

\paragraph{Proof Sketch:} As with~\cref{thm:rms-pl}, we begin with~\cref{eq:descent-ineq} with $p = \infty$ and $q = 1$. In~\cref{lemma:adam-descent}, we show that if $\beta_1 \leq \beta_2$, $\norm{d_t}_\infty = O(1)$. We define  $\bar{L}_t := H_0 + H_1 \, f(\tht)$ and obtain the following inequality, 
\begin{align}
f(\thtt) & \leq f(\tht) - \eta \, \langle \nabla_t, d_t \rangle + \bar{L}_t \, O(\eta^2)
\label{eq:adam-main-1}
\end{align}
Bounding $- \langle \nabla_t, d_t \rangle$,
\begin{align*}
& \langle \nabla_t, d_t \rangle = \sum_i \frac{\ti{g} \, \ti{m}}{\ti{\sqrt{v}}} = \sum_i \frac{\ti{g} \, (\ti{m} - \ti{g})}{\ti{\sqrt{v}}} + \frac{\ti{g^2}}{\ti{\sqrt{v}}} \\
& \implies - \langle \nabla_t, d_t \rangle \leq \underbrace{\sum_i \frac{|\ti{g}| \, |\ti{m} - \ti{g}|}{\ti{\sqrt{v}}}}_{:= \text{Term (i)}} - \underbrace{\sum_i \frac{\ti{g^2}}{\ti{\sqrt{v}}}}_{:= \text{Term (ii)}}
\end{align*}
Term (i) depends on the momentum, and we bound it using~\cref{lemma:adam-mom-bound}. In particular, by lower-bounding $\ti{v}$ in terms of $\ti{g}$, we obtain that, 
\begin{align}
\text{Term (i)} &\leq \frac{\norm{m_t - g_t}_1}{\sqrt{1 - \beta_2}}\,. \label{eq:adam-main-inter-1}
\end{align}
By using the momentum update, we get that,
\begin{align}
\norm{m_t - g_t}_1 & \leq \beta_1 \norm{m_{t-1} - g_{t-1}}_1 + \beta_1 \norm{\nabla_t - \nabla_{t-1}}_1 \nonumber
\end{align}
Using~\cref{eq:grad-lip} to bound the difference between consecutive gradients with an appropriate step-size, 
\begin{align}
\norm{\nabla_t - \nabla_{t-1}}_1 & \leq O(\eta) \, \bar{L}_t \nonumber
\end{align}
Combining the above two inequalities, we get that, 
\begin{align}
\frac{\norm{m_t - g_t}_1}{\bar{L}_t} & \leq \beta_1 \frac{\norm{m_{t-1} - g_{t-1}}_1}{\bar{L}_t} + \beta_1 \, O(\eta) \nonumber
\end{align}
We now use~\cref{eq:f-multi-lip} to write $\bar{L}_t$ in terms of $\bar{L}_{t-1}$, 
\begin{align}
\bar{L}_t = H_0 + H_1 \, f(\tht) & \geq \left(H_0 + H_1 \, f(\theta_{t-1}) \right) \nonumber\\ 
& \times \exp\left(-\sqrt{H_1} \, \norm{\tht - \theta_{t-1}}_{\infty} \right) \nonumber
\end{align}
With an appropriate step-size, $\frac{1}{\bar{L}_t} \leq \frac{1}{\bar{L}_{t-1}} \, \frac{1}{\sqrt{\beta_1}}$, 
\begin{align}
\implies \frac{\norm{m_t - g_t}_1}{\bar{L}_t} & \leq \sqrt{\beta_1} \, \frac{\norm{m_{t-1} - g_{t-1}}_1}{\bar{L}_{t-1}} + \beta_1 \, O(\eta) \nonumber
\end{align}
Solving the above recursion and using~\cref{eq:adam-main-inter-1}, we get that $\text{Term (i)} \leq \bar{L}_t \, O(\eta)$.
\begin{align}
\implies - \eta \, \langle \nabla_t, d_t \rangle & \leq \bar{L}_t \, O(\eta^2) - \eta \, (\text{Term (ii)}) \nonumber
\end{align}
Combining with~\cref{eq:adam-main-1}, 
\begin{align}
f(\thtt) \leq f(\tht)  - \eta \, (\text{Term (ii)}) +  \bar{L}_t \, O(\eta^2) \nonumber
\end{align}
The term that depends on the momentum is of the same order as the third term in~\cref{eq:adam-main-1}, and is absorbed into it. Term (ii) does not depend on momentum, and is the same as in the $\RMS$ proof. Up to constants, the remaining proof is exactly the same as that for $\RMS$, and results in the same rate. \qed

Importantly, our analysis does not require convexity or bounded gradients, and the resulting rate is dimension-independent. The above results show that, similar to $\NSD$, $\RMS$ and $\Adam$ result in linear convergence rates when $\epsilon = O\left(\nicefrac{H_0}{H_1}\right)$. In this regime, these methods can converge to the desired sub-optimality with an $O(1)$ step-size and $\delta = 0$. For $\epsilon < \nicefrac{H_0}{H_1}$, $\RMS$ and $\Adam$ can still converge with $\delta = 0$, provided that the step-size is sufficiently small (of the order $\epsilon^{2 \tau}$). In this case, these methods attain a slower sub-linear rate. 

\subsection{Implications}
\label{sec:adaptive-implications}
The above results apply to all objectives in~\cref{sec:examples}. Below, we highlight some practical consequences. 

\textit{Implication 1:} For neural network objectives satisfying~\cref{assn:nus} and the PL condition (\cref{assn:pl} with $\tau = \nicefrac{1}{2}$), $\Adam$ 
and $\RMS$ (corresponding to $\Adam$ with $\beta_1 = 0$) can converge at a linear rate for $\epsilon = O\left(\nicefrac{H_0}{H_1}\right)$, and at an $O\left(\nicefrac{1}{\epsilon}\right)$ rate for smaller $\epsilon$. As mentioned in~\cref{sec:examples}, this includes over-parameterized neural networks. 

\textit{Implication 2:} For binary classification using either a linear model (\cref{prop:exponential}) or a two-layer neural network (\cref{prop:nn}), when minimizing the exponential loss on separable data, $\RMS$ and $\Adam$ can use an $O(1)$ step-size and result in a linear convergence rate.

\textit{Implication 3:} In~\cref{thm:adam-logreg}, we prove that both $\RMS$ and $\Adam$ can achieve linear convergence for logistic regression on separable data. Given the sub-linear lower-bounds for GD, this provides a concrete setting in which $\RMS$ and $\Adam$ are provably faster. Furthermore, by using a similar argument as in~\cref{cor:bandits}, we conjecture that $\RMS$ and $\Adam$ can attain a linear rate for the softmax policy gradient objective. 

\subsection{Handling $(L_0, L_1)$-$\NS$ Functions}
\label{sec:gnus}
In~\cref{app:gen-nus-adam}, we analyze the convergence of $\Adam$ for non-convex functions satisfying the $(L_0,L_1)$-$\NS$ condition in~\cref{eq:zhang-nus}. Unlike the above result, the next theorem does not rely on the $\NL$-assumption in~\cref{assn:pl}. 

\begin{restatable}{theorem}{adamgnus}
Under~\cref{assn:non-negative} and~\cref{assn:gnus} with $L_0 \geq 0$ and $L_1 > 0$, and  $f^* = 0$, Adam with the update in~\cref{eq:adam-update-coordinate} has the following convergence rate:  

\hspace{3ex} \textbullet If $\epsilon \geq \frac{L_0}{L_1}$, using $\etat = \eta = O(1)$ guarantees that after $T = O\left(\frac{1}{\epsilon}\right)$ iterations, $\norm{\nabla f(\theta_{T})}_1 \leq \epsilon$. 

\hspace{3ex} \textbullet Else, using $\etat = \eta = O(\epsilon^2)$ guarantees that after $T = O\left(\frac{1}{\epsilon^2} \right)$ iterations, $\min_{t \leq T} \norm{\nabla f(\theta_t)}_1 \leq \epsilon$. 

\label{thm:adam-coordinate-gnus}
\end{restatable}
In contrast to~\citet{wang2024convergence} that only considers the scalar or norm version of the update, we consider the standard, diagonal variant of $\Adam$. The above result requires~\cref{eq:zhang-nus}, the generalization of the assumption in~\citet{zhang2019gradient} to $p, q$ norms. In particular, for large $\epsilon \geq \frac{L_0}{L_1}$, the above rate is faster than the $O\left(\frac{1}{\epsilon^2}\right)$ rate in~\citet[Theorem 1]{wang2024convergence}. For small $\epsilon < \frac{L_0}{L_1}$, the resulting $O\left(\frac{1}{\epsilon^2}\right)$ matches that of~\citet{wang2024convergence} and the standard non-convex convergence rate for GD. 
\section{Inefficiency of $\AdaGrad$ \& $\AMSGrad$}
\label{sec:lb}
In~\cref{sec:adaptive}, we have identified $(0, H_1)$-$\NS$ functions satisfying~\cref{assn:pl} as examples where $\RMS$ and $\Adam$ achieve linear convergence rates. The one-dimensional logistic loss function, $f(\theta) = \ln(1 + \exp(-\theta))$ is such an example, and satisfies~\cref{assn:non-negative},~\cref{assn:nus} with $H_0 = 0$, $H_1 = 1$ and~\cref{assn:pl} with $\tau = 1$, $\mu = 1$, $f^*= 0$. For this loss function, we now prove that other adaptive gradient methods such as $\AdaGrad$~\citep{duchi2011adaptive} and \texttt{AMSGrad}~\citep{reddi2019convergence} cannot achieve this faster linear convergence. In particular, we present the following $\Omega\left(\frac{1}{T}\right)$ lower-bound (proved in~\cref{app:lb}).

\begin{restatable}{theorem}{lb}
Starting from $\theta_1 = 0$, consider $T$ iterations of the form $\thtt = \tht - \etat \, m_t$ s.t. (i) the effective step-size $\etat$ is bounded, non-increasing and independent of $T$, (ii) for $\beta \in [0,1)$, $m_t = (1 - \beta) \, \sum_{s = 1}^{t} \beta^{t-s} \, \nabla f(\theta_s)$ is the momentum vector and (iii) $\eta_1 \leq \ln\left(\nicefrac{1}{\beta}\right)$. When minimizing the logistic loss with these update restrictions, the convergence rate is lower-bounded by $\Omega(1/T)$. 
\label{thm:slow}
\end{restatable}
In~\cref{app:lb-examples}, we show that constant step-size GD and heavy-ball momentum with an appropriate constant step-size satisfy the conditions in~\cref{thm:slow}. Our result complements the lower-bound for GD in~\citet{wu2024large}, and implies that for this class of functions, heavy-ball momentum cannot improve the convergence. 

Furthermore, we show that $\AdaGrad$~\citep{duchi2011adaptive} with any $O(1)$ step-size as well as $\AMSGrad$~\citep{reddi2019convergence} with an appropriate constant step-size also satisfy the conditions in~\cref{thm:slow}, and consequently have an $\Omega(1/T)$ lower-bound. Unlike these methods, $\RMS$ and $\Adam$ do not ensure that the effective step-size is non-increasing and hence do not satisfy condition (i) in~\cref{thm:slow}. Consequently, they do not suffer from this slower convergence. In fact, in~\cref{thm:adam-logreg}, we have shown that $\RMS$ and $\Adam$ (with similar restrictions on the step-size as in~\cref{thm:slow}) achieve an $O(\exp(-T))$ convergence rate. Hence, we have identified a natural class of functions where $\RMS$ and $\Adam$ are provably faster than $\AdaGrad$ and $\AMSGrad$.

\section{Conclusion}
\label{sec:conclusion}
We analyzed the convergence of steepest descent, and deterministic $\RMS$ and $\Adam$ on non-uniform smooth functions satisfying gradient domination. Our results imply the fast convergence of these methods for certain neural networks and policy gradient objectives. Furthermore, we proved a separation, identifying practical objectives where $\RMS$ and $\Adam$ are provably faster than GD, $\AdaGrad$ and $\AMSGrad$. 

We believe that the same techniques can be used to analyze the convergence of related methods such as steepest descent with momentum~\citep{bernstein2018signsgd} and deterministic variants of generalized sign GD~\citep{crawshaw2022robustness} and Lion~\citep{chen2023symbolic}. In the future, we plan to generalize our results to the stochastic setting, and study a broader class of non-convex functions. 
\section*{Impact Statement}
This paper presents work whose goal is to advance the field of Machine
Learning. There are many potential societal consequences of our work, none
which we feel must be specifically highlighted here.

\bibliographystyle{icml2026}
\bibliography{ref}

\appendix
\onecolumn
\renewcommand{\contentsname}{Appendix}
\addtocontents{toc}{\protect\setcounter{tocdepth}{3}}
{
  \hypersetup{hidelinks}
  \tableofcontents
}
\clearpage
\section{Related Work}
\label{app:related}
In this section, we present a more detailed review of the related work. Previous work can be characterized into two broad categories:

    \begin{itemize}

        \item \textbf{Convergence under $(H_0, H_1)$ non-uniform smoothness}: The most relevant papers are~\citet{vaswani2025armijo,alimisis2025we} that consider the same non-uniform assumption as we do.~\citet{vaswani2025armijo} use this assumption to justify the use of Armijo line-search, while~\citet{alimisis2025we} use this assumption to justify the importance of learning-rate warm-up. In terms of convergence guarantees,~\citet{vaswani2025armijo} analyze GD with Armijo line-search and~\citet{alimisis2025we} analyze normalized GD. Both papers derive similar convergence guarantees as~\cref{thm:nsd-pl} in our paper. However, both these papers are inherently in the Euclidean setting, and do not derive dimension-free guarantees for general (normalized) steepest descent, nor do they consider RMSProp or Adam. In~\cref{app:ls}, we generalize the result in~\citet{vaswani2025armijo} to steepest descent with Armijo line-search, while substantially simplifying their proof.      
        
        \item \textbf{Convergence under $(L_c, L_g)$ non-uniform smoothness}:~\citep{zhang2019gradient,gorbunov2024methods,vankov2024optimizing,li2023convergence,wang2024convergence,wang2024provable} consider a different non-uniform smoothness assumption in~\cref{eq:zhang-nus}. This assumption only considers Euclidean norms, and is in general, stronger than our $(H_0, H_1)$ assumption~\citep{alimisis2025we}. Importantly, this assumption cannot model even the exponential or logistic regression loss~\citep{vaswani2025armijo}. 
        
        Under this assumption~\citep{zhang2019gradient,gorbunov2024methods,vankov2024optimizing} have analyzed the convergence of variants of normalized gradient descent for general non-convex functions (that do not necessarily satisfy Assumption 3) and convex functions. However, unlike~\cref{thm:nsd-pl}, these papers do not derive dimension-free guarantees for general normalized steepest descent. 
                
        Furthermore, under this assumption,~\citet{li2023convergence,wang2024convergence,wang2024provable} analyze the scalar or norm version of Adam and derive a $O(1/\epsilon^2)$ stationary-point convergence for general non-convex functions. In~\cref{thm:adam-coordinate-gnus}, we specialize our proof technique to the $(L_0, L_1)$ assumption, and derive faster convergence rates.

    \end{itemize}

    \textbf{Other work:} Apart from these main bodies of work, there are papers that focus on specific examples of $(H_0, H_1)$ functions and analyze the convergence of specific normalized steepest descent methods and derive similar rates as in our Theorem 1. For example,~\citet{mei2021leveraging} use normalized gradient descent for the softmax policy gradient objective,~\citet{taheri2023fast} use it for 2 layer neural networks and~\citet{axiotis2023gradient} use greedy coordinate descent (corresponding to $p = 1, q = \infty$ in our Theorem 1) and analyze its convergence on logistic regression. 
\section{Examples}
\label{app:examples}

\exponential*
\begin{proof}
Clearly, $f_i(\theta) \geq 0$ and hence $f(\theta) \geq 0$ for all $\theta$. Calculating the gradient and hessian for $f_i(\theta) := \exp(-y_i \langle x_i, \theta \rangle)$, 
\begin{align*}
\nabla f_i(\theta) &=  -\exp(-y_i \langle x_i, \theta \rangle) y_i \, x_i \\
\nabla^2 f_i(\theta) & = \exp(-y_i \langle x_i, \theta \rangle) \, y_i^2 \, x_i \, x_i^T = \exp(-y_i \langle x_i, \theta \rangle) \, x_i \, x_i^T \tag{$y_i^2 = 1$} \\
\implies \norm{\nabla^2 f_i(\theta)}_{p \to q} & \leq \exp(-y_i \langle x_i, \theta \rangle) \, \norm{x_i \, x_i^T}_{p \to q} = f_i(\theta) \, \normsq{x_i}_{q} \tag{Since for a rank one matrix $\norm{u u^T}_{p \to q} = \normsq{u}_q$} \\
\norm{\nabla^2 f(\theta)}_{p \to q} & \leq \frac{1}{n} \,\sum_{i = 1}^{n} \norm{\nabla^2 f_i(\theta)}_{p \to q} \leq f(\theta) \, \max_i \normsq{x_i}_q \tag{Using triangle inequality and above relation} 
\end{align*}
Verifying~\cref{assn:pl} for separable data, 
\begin{align*}
\norm{\nabla f(\theta)}_q &= \sup_{\norm{u}_p \leq 1} \langle u, \nabla f(\theta) \rangle \\  
& = \sup_{\norm{u}_p \leq 1} \left[ \frac{1}{n} \, \left\langle -\sum_{i = 1}^{n} \exp(-y_i \langle x_i, \theta \rangle) y_i \, x_i, u \right \rangle \right] \\
& \geq \left[ \frac{1}{n} \, \left\langle -\sum_{i = 1}^{n} \exp(-y_i \langle x_i, \theta \rangle) y_i \, x_i, u \right \rangle \right] \tag{for all $u$ s.t. $\norm{u}_p \leq 1$} \\
\implies \norm{\nabla f(\theta)}_q & \geq \left[ \frac{1}{n} \, \left\langle \sum_{i = 1}^{n} \exp(-y_i \langle x_i, \theta \rangle) y_i \, x_i, \frac{\theta^*}{\norm{\theta^*}_p} \right \rangle \right] \tag{Setting $u = -\frac{\theta^*}{\norm{\theta^*}_p}$} \\
& \geq \min_{j} \frac{y_j \, \langle x_j, \theta^* \rangle}{\norm{\theta^*}_p} \, \left[ \frac{1}{n} \, \sum_{i = 1}^{n} \exp\left(-y_i \, \langle x_i, \theta \rangle \right) \right] \\
\implies \norm{\nabla f(\theta)}_q & \geq \gamma_p \, f(\theta) \tag{By definition of $\gamma_p$ and $f(\theta)$}
\end{align*}
\end{proof}

\logistic*
\begin{proof}
Clearly, $f_i(\theta) \geq 0$ and hence $f(\theta) \geq 0$ for all $\theta$. Calculating the gradient and hessian for $f_i(\theta) := \ln(1 + \exp(-y_i \langle x_i, \theta \rangle)$,
\begin{align*}
\nabla f_i(\theta) & = \frac{-\exp(-y_i \langle x_i, \theta \rangle)}{1 + \exp(-y_i \langle\, x_i, \theta \rangle)} y_i \, x_i \quad \text{;} \quad \nabla^2 f_i(\theta) = \frac{1}{1 + \exp(-y_i \, \langle x_i, \theta \rangle)} \, \frac{\exp(-y_i \, \langle x_i, \theta \rangle)}{1 + \exp(-y_i \, \langle x_i, \theta \rangle)} \, y_i^2 \, x_i \, x_i^T \\
\end{align*}
Bounding the Hessian,
\begin{align*}
\norm{\nabla^2 f_i(\theta)}_{p \to q} & \leq \norm{\frac{1}{1 + \exp(-y_i \, \langle x_i, \theta \rangle)} \, \frac{\exp(-y_i \, \langle x_i, \theta \rangle)}{1 + \exp(-y_i \, \langle x_i, \theta \rangle)} \, y_i^2 \, x_i \, x_i^T}_{p \to q} \\
& = \frac{1}{1 + \exp(-y_i \, \langle x_i, \theta \rangle)} \, \frac{\exp(-y_i \, \langle x_i, \theta \rangle)}{1 + \exp(-y_i \, \langle x_i, \theta \rangle)} \, \norm{x_i \, x_i^T}_{p \to q} \tag{Since $y_i^2 = 1$} \\
& \leq \frac{1}{1 + \exp(-y_i \, \langle x_i, \theta \rangle)} \, \frac{\exp(-y_i \, \langle x_i, \theta \rangle)}{1 + \exp(-y_i \, \langle x_i, \theta \rangle)} \normsq{x_i}_{q} \tag{Since for a rank one matrix $\norm{u u^T}_{p \to q} = \normsq{u}_q$} \\
& \leq \frac{\exp(-y_i \, \langle x_i, \theta \rangle)}{1 + \exp(-y_i \, \langle x_i, \theta \rangle)} \,   \normsq{x_i}_{q} \tag{For all $x$, $\frac{1}{1 + e^x} \leq 1$} \\
& \leq \ln(1 + \exp(-y_i \, \langle x_i, \theta \rangle)) \,   \normsq{x_i}_{q}  \tag{For all $x\geq 0$, $\frac{x}{1+x} \leq \ln(1+x)$} \\
\implies \norm{\nabla^2 f_i(\theta)}_{p \to q} & \leq  f_i(\theta) \,  \normsq{x_i}_{q} \\
\implies \norm{\nabla^2 f(\theta)}_{p \to q} & \leq \frac{1}{n} \,\sum_{i = 1}^{n} \norm{\nabla^2 f_i(\theta)}_{p \to q} \leq f(\theta) \, \max_i \normsq{x_i}_q \tag{Using triangle inequality and above relation} 
\end{align*} 
Verifying~\cref{assn:pl} for separable data, 
\begin{align*}
\norm{\nabla f(\theta)}_q &= \sup_{\norm{u}_p \leq 1} \langle u, \nabla f(\theta) \rangle \\  
& = \sup_{\norm{u}_p \leq 1} \left[ \frac{1}{n} \, \left\langle -\sum_{i = 1}^{n} \frac{1}{1 + \exp(y_i \langle x_i, \theta \rangle)} y_i \, x_i, u \right \rangle \right] \\
& \geq \left[ \frac{1}{n} \, \left\langle -\sum_{i = 1}^{n} \frac{1}{1 + \exp(y_i \langle x_i, \theta \rangle)} \, y_i \, x_i, u \right \rangle \right] \tag{for all $u$ s.t. $\norm{u}_p \leq 1$} \\
\implies \norm{\nabla f(\theta)}_q & \geq \left[ \frac{1}{n} \, \left\langle \sum_{i = 1}^{n} \frac{1}{1 + \exp(y_i \langle x_i, \theta \rangle)} y_i \, x_i, \frac{\theta^*}{\norm{\theta^*}_p} \right \rangle \right] \tag{Setting $u = -\frac{\theta^*}{\norm{\theta^*}_p}$} \\
& \geq \min_{j} \frac{y_j \, \langle x_j, \theta^* \rangle}{\norm{\theta^*}_p} \, \left[ \frac{1}{n} \, \sum_{i = 1}^{n} \frac{1}{1 + \exp(y_i \langle x_i, \theta \rangle)} \right] \\
\implies \norm{\nabla f(\theta)}_q & \geq \gamma_p \, \left[ \frac{1}{n} \, \sum_{i = 1}^{n} \frac{1}{1 + \exp(y_i \langle x_i, \theta \rangle)} \right] \tag{By definition of $\gamma_p$}
\end{align*}
\textbf{Case (1)}: If $f(\theta) \leq \frac{\ln(2)}{n}$, then, 
\begin{align*}
\implies \sum_i \ln(1 + \exp(-y_i \langle x_i, \theta \rangle) ) & \leq \ln(2) \implies \ln(1 + \exp(-y_i \langle x_i, \theta \rangle) ) \leq \ln(2) \tag{Since $f_i(\theta) > 0$} \\
\implies y_i \langle x_i, \theta \rangle \geq 0
\end{align*}
Hence, all examples are classified with a positive margin. In this case, we use the following inequality, 
\begin{align*}
\frac{1}{1 + \exp(z)} & = \frac{\exp(-z)}{1 + \exp(-z)} \geq \frac{\exp(-z)}{2} \tag{For $z \geq 0$} \\
& \geq \frac{\ln(1+ \exp(-z))}{2} \tag{For all $z$, $\ln(1+z) \leq z$}
\end{align*}
Since $y_i \langle x_i, \theta \rangle > 0$ for all $i \leq [n]$, 
\begin{align*}
\frac{1}{1 + \exp(y_i \langle x_i, \theta \rangle)} & \geq \frac{\ln(1 + \exp(-y_i \langle x_i, \theta \rangle))}{2} \\
\implies \norm{\nabla f(\theta)}_q & \geq \frac{\gamma_p}{2} \, \left[ \frac{1}{n} \, \sum_{i = 1}^{n} \ln(1 + \exp(-y_i \langle x_i, \theta \rangle)) \right] = \frac{\gamma_p}{2} \, f(\theta) 
\end{align*}
\textbf{Case (2)}: On the other hand, if $f(\theta) > \frac{\ln(2)}{n}$ and $u_i := \exp(-y_i \, \langle x_i \theta \rangle) > 0$, then,  
\begin{align*}
\sum_{i = 1}^{n} \ln(1 + \exp(- y_i \langle x_i, \theta \rangle)) & > \ln(2) \implies \sum_{i = 1}^{n} u_i > \ln(2)  \tag{Since $\ln(1+u) < u$ for all $u > 0$} \\
\implies \frac{1}{1 + \exp(y_i \langle x_i, \theta \rangle)} &= \frac{u_i}{1 + u_i} \geq \frac{u_i}{1 + \sum_{j = 1}^{n} u_j} \tag{Since $u_j > 0$} \\
\implies \sum_{i = 1}^{n} \frac{1}{1 + \exp(y_i \langle x_i, \theta \rangle)} &\geq \frac{\sum_{i = 1}^{n} u_i}{1 + \sum_{j = 1}^n u_j} \geq \frac{\ln(2)}{1 + \ln(2)} \tag{Since $\frac{u}{1+u}$ is increasing for $u \geq 0$}
\end{align*}
Using the above lower-bound on the gradient norm, 
\begin{align*}
\norm{\nabla f(\theta)}_q & \geq \gamma_p \, \left[ \frac{1}{n} \, \sum_{i = 1}^{n} \frac{1}{1 + \exp(y_i \langle x_i, \theta \rangle)} \right] \geq \frac{\gamma_p}{n} \, \frac{\ln(2)}{1 + \ln(2)} \geq \frac{\gamma_p}{3n}
\end{align*}
\end{proof}

\bandit*
\begin{proof}
Calculating the gradient, 
\begin{align}
\nabla f(\theta) &= - \nabla [\langle \pi_\theta, r \rangle ] = - G_\theta \, r \quad \text{where } \quad G_\theta \in \R^{K \times K} = \text{diag}(\pi_\theta) - \pi_\theta \, [\pi_\theta]^T \quad \text{and} \quad G_\theta[i,j] = \frac{\partial \pi_\theta(i)}{\partial \theta_j} \nonumber\\
\implies \nabla f(\theta) &= - \text{diag}(\pi_\theta) \, r + \langle \pi_\theta, r \rangle \, \pi_\theta \implies [\nabla f(\theta)]_j = -\pi_\theta(j) \, [r_j - \langle \pi_\theta, r \rangle ] \label{eq:pg-grad-inter-1}
\end{align}
Calculating the Hessian, 
\begin{align*}
\frac{\partial^2 f}{\partial \theta_k \,\partial  \theta_j} & = \frac{\partial [\nabla f(\theta)]_j}{\partial \theta_k} = -\frac{\partial}{\partial \theta_k} [ \pi_\theta(j) \, (r_j - \langle \pi_\theta, r \rangle )]   \\
& = - \left[ G_{j,k} \, [r_j - \langle \pi_\theta, r \rangle] - \pi_\theta(j) \,\left(\pi_\theta(k) \, [r_k - \langle \pi_\theta, r \rangle ] \right) \right] \\
\implies \frac{\partial^2 f}{\partial \theta_k \, \partial \theta_j} & = - G_{j,k} \, [r_j - \langle \pi_\theta, r \rangle] + \pi_\theta(j) \, \pi_\theta(k) \,  [r_k - \langle \pi_\theta, r \rangle ] \\
\implies \nabla^2 f(\theta) &= -\underbrace{\text{diag}(r - \langle \pi_\theta, r \rangle \, \mathbf{1})}_{K \times K} \, \underbrace{G_\theta}_{K \times K} + \underbrace{\pi_\theta}_{K \times 1} \, \underbrace{[G_\theta \, r]^T}_{1 \times K} 
\end{align*}
Define $w := G_\theta r = \text{diag}(\pi_\theta) \, (r - \langle \pi_\theta, r \rangle \mathbf{1})$, and simplifying the first term, 
\begin{align*}
\text{diag}(r - \langle \pi_\theta, r \rangle \, \mathbf{1}) \, G_\theta &= \text{diag}(r - \langle \pi_\theta, r \rangle \, \mathbf{1})  \left(\text{diag}(\pi_\theta) - \pi_\theta \, [\pi_\theta]^T\right) \\
& = \text{diag}(\pi_\theta \circ (r - \langle \pi_\theta, r \rangle ) ) - [\pi_\theta \circ (r - \langle \pi_\theta, r \rangle )] [\pi_\theta]^T \\
& = \text{diag}(w) - w [\pi_\theta]^T 
\end{align*}
Combining the above relations, 
\begin{align}
\nabla^2 f(\theta) &= -\text{diag}(w) + w [\pi_\theta]^T + \pi_\theta [w]^T \label{eq:pg-hessian-inter-1}     
\end{align}

\textbf{Case 1}: Consider $p = q = 2$. In this case, 
\begin{align*}
\norm{\nabla^2 f(\theta)}_{2 \to 2} & \leq \norm{\text{diag}(w)}_{2 \to 2} + \norm{w [\pi_\theta]^T}_{2 \to 2} +  \norm{\pi_\theta [w]^T}_{2 \to 2} \tag{Using triangle inequality} \\    
& \leq \norm{w}_2 + 2 \, \norm{w}_2 \, \norm{\pi_\theta}_2 \leq 3 \, \norm{w}_2 \tag{Since $\norm{\pi_\theta}_2 \leq \norm{\pi_\theta}_1 = 1$} \\
\implies \norm{\nabla^2 f(\theta)}_{2 \to 2} & \leq 3 \, \norm{w}_2 = 3 \, \norm{G_\theta r} = 3 \, \norm{\nabla f(\theta)}_{2} \tag{Using~\cref{eq:pg-grad-inter-1}}
\end{align*}
Using~\citep[Proposition 3]{vaswani2025armijo}, if $\norm{\nabla^2 f(\theta)}_{2} \leq 3 \norm{\nabla f(\theta)}_2$ and $f^* = 0$, 
\begin{align*}
\norm{\nabla f(\theta)}_2 & \leq 8 \, f(\theta)   
\end{align*}
Combining the above equations, 
\begin{align*}
\norm{\nabla^2 f(\theta)}_{2 \to 2} & \leq 24 \, f(\theta)    
\end{align*}

\textbf{Case 2}: Consider $p = \infty, q = 1$. In this case, 
\begin{align}
\norm{\nabla^2 f(\theta)}_{\infty \to 1} & \leq \norm{\text{diag}(w)}_{\infty \to 1} + \norm{w [\pi_\theta]^T}_{\infty \to 1} +  \norm{\pi_\theta [w]^T}_{\infty \to 1} \tag{Using triangle inequality} \\
& \leq \max_{\norm{v}_\infty \leq 1} \norm{\text{diag}(w) \, v} + \norm{w}_{1} \max_{\norm{v}_\infty \leq 1} \langle \pi_\theta, v \rangle + \max_{\norm{v}_\infty \leq 1} \langle w, v \rangle \, \sum_i \pi_\theta(i) \nonumber \\
& \leq 3 \, \norm{w}_{1} = 3 \, \norm{G_\theta \, r}_{1} \tag{Since $\norm{\pi_{\theta}}_1 = \sum_i \pi_\theta(i) = 1$} \\
\implies \norm{\nabla^2 f(\theta)}_{\infty \to 1} & \leq 3 \, \norm{G_\theta \, r}_{1} \label{eq:pg-hessian-inter-2}
\end{align}
Calculating $\norm{G_\theta \, r}_{1}$,
\begin{align*}
\norm{G_\theta \, r}_{1} &= \sum_{a} \pi_\theta(a) \, |r(a) - \langle \pi_\theta, r \rangle | = \sum_{a} \pi_\theta(a) \, |r(a) - r(a^*) + r(a^*) - \langle \pi_\theta, r \rangle | \\
& = \sum_{a} \pi_\theta(a) \, |r(a) - r(a^*) + f(\theta)| \leq f(\theta) + \sum_{a} \pi_\theta(a) |r(a) - r(a^*)| \tag{Using triangle inequality and definition of $f(\theta)$} \\
& = f(\theta) + \sum_{a} \pi_\theta(a) [r(a^*) - r(a)] \tag{Since $r(a^*) > r(a)$ by definition of the optimal arm $a^*$} \\
& = f(\theta) + (r(a^*) - \langle \pi_\theta, r \rangle 
\implies \norm{G_\theta \, r}_{1} &\leq 2 \, f(\theta) 
\end{align*}
Combining the above inequality with~\cref{eq:pg-hessian-inter-2},
\begin{align*}
\norm{\nabla^2 f(\theta)}_{\infty \to 1} & \leq 6 \, f(\theta)  
\end{align*}

\textbf{Case 3:} Consider $p = 1, q = \infty$. Since $\norm{\cdot}_{1,\infty} \leq \norm{\infty,1}$, using the result for case 2 immediately gives us:
\begin{align*}
\norm{\nabla^2 f(\theta)}_{1 \to \infty} & \leq 6 \, f(\theta)      
\end{align*}

Verifying~\cref{assn:pl}, for any $q \geq 1$,
\begin{align*}
\norm{\nabla f(\theta)}_q &= \norm{G_\theta \, r}_q \geq \vert \pi_\theta(a^*) \, [r(a^*) - \langle \pi_\theta, r \rangle] \vert = \pi_\theta(a^*) f(\theta)
\end{align*}
\end{proof}

\nn*
This proposition is a generalization of Lemmas 3 and 5 in \citet{taheri2023fast}, which is specific to the exponential loss, with $p = q = 2$.

\begin{proof}
Clearly $f$ satisfies~\cref{assn:non-negative} by construction. Moreover, for $\theta = [\theta_1 , \theta_2 , \ldots , \theta_m] \in \R^{m \, D}$ where for each $j \in [m]$, $\theta_j \in \R^{D}$, 
\begin{align}
f(\theta) & = \frac{1}{n} \sum_{i=1}^n g(y_i \, \Phi(\theta, x_i )) \implies \nabla f(\theta) = \frac{1}{n} \sum_{i=1}^n g'(y_i \Phi(\theta, x_i ))\, y_i\nabla \Phi(\theta, x_i ) \nonumber \\
\Phi(\theta,x) & = \sum_{j=1}^m a_j \sigma(\langle \theta_j,x\rangle) \implies \nabla \Phi(\theta,x) = [a_1 \sigma'(\langle \theta_1,x\rangle) x, a_2 \sigma'(\langle \theta_2,x\rangle) x, \ldots, a_m \sigma'(\langle \theta_m,x\rangle) x] \nonumber \\
\implies \|\nabla \Phi(\theta,x)\|_q &\leq  \alpha_2 \|a\|_q \|x\|_q\label{eq:phigrad_upperbound}  
\end{align}
$\nabla^2 \Phi(\theta, x)$ is a $\R^{m D \times m D}$ block-diagonal matrix where block $j$ is an $D \times D$ matrix equal to $a_j \, \sigma''(\langle \theta_j,x\rangle) x x^{T}$. Hence, 
\begin{align}
\norm{\nabla^2 \Phi(\theta, x)}_{p \to q}  & \leq \sum_{j = 1}^{m} \norm{a_j \, \sigma''(\langle \theta_j,x\rangle) x x^{T}}_{p \to q} \leq \sum_{j = 1}^{m} |a_j| \, M \, \normsq{x}_q = M \, \norm{a}_1 \, \normsq{x}_{q} \label{eq:phihess_upperbound}  
\end{align}
Calculating the Hessian of $f$,
\begin{align*}
\nabla^2 f(\theta) &= \frac{1}{n} \, \sum_{i = 1}^{n} \left[ g'(y_i \, \Phi(\theta, x_i )) \, \nabla^2 \Phi(\theta, x_i) + g''(y_i \, \Phi(\theta, x_i )) \, [\nabla \Phi(\theta, x_i)] \, [\nabla \Phi(\theta, x_i)]^{T} \right] \\
\implies \norm{\nabla^2 f(\theta)}_{p \to q} & \leq \frac{1}{n} \, \sum_{i = 1}^{n} \left[ \vert g'(y_i \, \Phi(\theta, x_i )) \vert \norm{\nabla^2 \Phi(\theta, x_i)}_{p \to q} + \vert g''(y_i \, \Phi(\theta, x_i )) \vert \, \norm{[\nabla \Phi(\theta, x_i)] \, [\nabla \Phi(\theta, x_i)]^{T}}_{p \to q} \right] \tag{By triangle inequality} \\
& \leq \frac{1}{n} \, \sum_{i = 1}^{n}  g(y_i \, \Phi(\theta, x_i )) \, \left[c_2 \, \norm{\nabla^2 \Phi(\theta, x_i)}_{p \to q} + c'_2 \norm{[\nabla \Phi(\theta, x_i)] \, [\nabla \Phi(\theta, x_i)]^{T}}_{p \to q} \right] \tag{Since $|g'(s)| \leq c_2 \, g(s)$ and $|g''(s)| \leq c_2' \, g(s)$ for all $s$} \\
& \leq \frac{1}{n} \, \sum_{i = 1}^{n}  g(y_i \, \Phi(\theta, x_i )) \, \left[c_2 \,  L \, \norm{a}_1 \, \normsq{x_i}_{q} + c'_2 \,  \alpha_2^2 \, \|a\|_q^2 \, \norm{x_i}_q^2 \right] \\
& = \left[c_2 \,  M \, \norm{a}_1 + c'_2 \,  \alpha_2^2 \, \|a\|_q^2 \right] \, \max_i \normsq{x_i}_{q} \, f(\theta) 
\end{align*}
Hence, $f$ satisfies~\cref{assn:nus} with $H_0=0$ and $H_1 = \left[c_2 \,  M \, \norm{a}_1 + c'_2 \,  \alpha_2^2 \, \|a\|_q^2 \right] \, \max_i \normsq{x_i}_{q}$. 

Next, using the construction $v_j = a_j \, w^*$ and $v = (v_1,v_2,...) = [a_1 \, w^*, a_2 \, w^* , \ldots, a_m \, w^*]$,  where $w^*$ is the max-margin separator that satisfies $\frac{y_i\langle x_i,w^*\rangle}{\|w^*\|_p}\geq \gamma$, we get that, 
\begin{align*}
    \|\nabla f(\theta) \|_q &=\sup_u \, \frac{\langle -\nabla f(\theta), u\rangle}{\|u\|_p} \tag{By definition of dual norm} \\
    &= \sup_u\,  -\frac{1}{\|u\|_p} \frac{1}{n} \sum_{i=1}^n g'(y_i \Phi(\theta, x_i ))\, y_i\, \langle\nabla \Phi(\theta, x_i ), u \rangle \tag{Substituting $\nabla f(\theta) = \frac{1}{n} \sum_{i} g'(y_i \Phi(\theta, x_i ))\, y_i\nabla \Phi(\theta, x_i )$}\\
    &\geq  -\frac{1}{\|v\|_p} \frac{1}{n} \sum_{i=1}^n g'(y_i \Phi(\theta, x_i ))\, y_i\, \langle\nabla \Phi(\theta, x_i ), v \rangle \tag{Substituting $u = v$}\\
    &=  -\frac{1}{n \|v\|_p} \sum_{i=1}^n g'(y_i \Phi(\theta, x_i )) \, y_i \, \sum_{j=1}^m a_j \, \sigma'(\langle \theta_j,x_i\rangle) \,  \langle x_i, v_j \rangle \tag{Substituting $\nabla \Phi(\theta,x) = \sum_{j} a_j \sigma'(\langle \theta_j,x\rangle) x$}\\
    &\geq \frac{1}{n \, \norm{v}_p} \sum_{i=1}^n |g'(y_i \Phi(\theta, x_i )) |\, \sum_{j=1}^m a^2_j  \, \sigma'(\langle \theta_j, x_i \rangle) \, \langle y_i x_i, w^* \rangle \tag{Since $v_i = a_i \, w^*$ and $g'(s) \leq 0$ for all $s$}\\
    & = \frac{1}{n \|a\|_p} \sum_{i=1}^n |g'(y_i \Phi(\theta, x_i )) |\, \sum_{j=1}^m a^2_j \, \sigma'(\langle \theta_j, x_i \rangle) \, \frac{ \langle y_i x_i, w^* \rangle }{\|w^*\|_p} \\
    &\geq  \frac{\alpha_1}{n \|a\|_p} \sum_{i=1}^n |g'(y_i \Phi(\theta, x_i )) |\, \underbrace{\sum_{j=1}^m a^2_j }_{\|a\|_2^2} \ \underbrace{ \frac{ \langle y_i x_i, w^* \rangle }{\|w^*\|_p}}_{\geq \gamma}\tag{Since $y_i \langle x_i, w^* \rangle > 0$ and $\sigma'(t) \geq \alpha_1$ for all $t$}\\
     &\geq \frac{\alpha_1 \gamma  \|a\|_2^2}{\|a\|_p} \frac{1}{n}\sum_{i=1}^n |g'(y_i \Phi(\theta, x_i ))| = \frac{\alpha_1 \gamma  \|a\|_2^2}{\|a\|_p} \frac{1}{n}\sum_{i=1}^n -g'(y_i \Phi(\theta, x_i )) \tag{Since $g'(s) \leq 0$ for all $s$} \\
     &\geq \frac{\alpha_1 \gamma  \|a\|_2^2 c_1}{\|a\|_p} \underbrace{\frac{1}{n}\sum_{i=1}^n g(y_i \Phi(\theta, x_i ))}_{f(\theta)}\tag{Since $-g'(s) \geq c_1 \, g(s)$ for all $s$}
\end{align*}
Hence, $f$ satisfies~\cref{assn:pl} with $\mu = \frac{\alpha_1 \gamma  \|a\|_2^2 c_1}{\|a\|_p}$
\end{proof}

\begin{restatable}{proposition}{multiclass_gennorm}
Consider $n$ points where $x^{(i)} \in \R^D$ are the features and $y^{(i)} \in \{0,1\}^K$ are the corresponding one-hot label vectors for $K$ classes. For each $i \in [n]$, $c_i \in [K]$ is the index of the true label such that $y^{(i)}(c_i) = 1$ and for all $j \neq c_i$, $y^{(i)}(j) = 0$. The loss for multi-class classification with the cross-entropy objective  (multiclass logistic regression) is given as:
\begin{align*}
f(\theta) & = \frac{1}{n} \sum_{i = 1}^{n} \text{KL}(y^{(i)} || \pi^{(i)}_\theta) \,, 
\text{ where $\forall i \in [n]$, $\pi^{(i)}_\theta \in \Delta_{K}$ s.t. $\forall c \in [K]$, } \pi^{(i)}_\theta(c) = \frac{\exp(\langle x^{(i)}, \theta_c \rangle)}{\sum_{k = 1}^{K} \exp(\langle x^{(i)}, \theta_k \rangle)} \,,
\end{align*}
where $\theta_c \in \R^D$ for $c \in [K]$ and $\theta = [\theta_1, \theta_2, \ldots, \theta_K]$. Multi-class logistic regression satisfies~\cref{assn:non-negative} and~\cref{assn:nus} with $H_0 = 0$ and $H_1 = 4 \, \max_i \normsq{x_i}_q$.

Furthermore, if the data is separable with a normalized margin $\gamma_p:= \max_\theta \min_{i\in [n]} \langle x_i, \theta_{c_i}\rangle - \min_{\substack{k \in [K] \\ k \neq c_i}} \langle x_i, \theta_k \rangle >0$, then,
\begin{itemize}
    \item For $f(\theta)\leq \frac{\ln(2)}{n}$, $f$ satisfies \cref{assn:pl} with $\tau = 1$, $\mu =\frac{\gamma_p}{K}$ and $f^* = 0$
    \item Else, if $f(\theta) > \frac{\ln(2)}{n}$ then $\|\nabla f(\theta)\|_q \geq \frac{\gamma_p}{2n}$.
\end{itemize}
\label{prop:multiclass}    
\end{restatable}
\begin{proof}
For the first part of the proof, note that, 
\begin{align*}
f_i(\theta) = \text{KL}(y^{(i)} || \pi^{(i)}_\theta) \text{ where $\pi^{(i)}_\theta \in \Delta_{K}$ s.t. } \pi^{(i)}_\theta(c) = \frac{\exp(\langle x^{(i)}, \theta_c \rangle)}{\sum_{k = 1}^{K} \exp(\langle x^{(i)}, \theta_k \rangle)} \,,
\end{align*}
where $\theta_c \in \R^D$ for $c \in [K]$ and $\theta = [\theta_1, \theta_2, \ldots, \theta_K]$. Since $y^{(i)}$ is a one-hot vector, let $c_{i}$ be the index corresponding to the non-zero entry. Hence, $y^{(i)}_{c_{i}} = 1$ and for all $j \neq c_{i}$, $y^{(i)}_{j} = 0$, and hence, 
\begin{align*}
f_i(\theta) &= -\ln\left(\pi_\theta^{(i)}(c_{i})\right)    
\end{align*}
The Hessian can be written as a Kronecker product of a $K \times K$ matrix which corresponds to the Jacobian of the softmax function, and a $d \times d$ rank-one matrix formed using the features. Specifically, 
\begin{align}
\nabla^2 f_i(\theta) &= \underbrace{G^{(i)}(\theta)}_{K \times K} \,\otimes\, \underbrace{x^{(i)} [x^{(i)}]^{T}}_{d \times d} \text{ where, } G^{(i)}(\theta) := \text{diag}(\pi^i_\theta) - \pi^i_\theta \, [\pi_\theta{^{(i)}}]^T 
\end{align}
Using~\cref{lemma:krondecomp} for the fixed $\theta$, and using $G^{(i)}$ as a shorthand for $G^{(i)}(\theta)$, for $q \in [1,\infty)$, 
\begin{align*}
\norm{\nabla^2 f_i(\theta)}_{p \to q} & \leq \left(\sum_{k,l} |G^{(i)}_{k,l}|^{q} \right)^{1/q}  \, \left(\sum_{k,l} \left(x^{(i)}_{k} \, x^{(i)}_{l}\right)^{q} \right)^{1/q} \\  
& = \left(\sum_{k,l} |G^{(i)}_{k,l}|^{q} \right)^{1/q} \, \normsq{x^{(i)}}_{q} \\
& \leq \sum_{k,l} |G^{(i)}_{k,l}| \, \normsq{x^{(i)}}_{q} \tag{Since $\norm{\cdot}_{q} \leq \norm{\cdot}_1$ for all $q \geq 1$}
\end{align*}

On the other hand, using~\cref{lemma:krondecomp} for $q = \infty$, 
\begin{align*}
\norm{\nabla^2 f_i(\theta)}_{1 \to \infty} & \leq  \max_{k,l}| G^{(i)}_{k,l} | \, (\max_{k} | x^{(i)}_k |)^{2} \leq \sum_{k,l} |G^{(i)}_{k,l}| \, \normsq{x^{(i)}}_\infty \tag{Since $\max_i v_i \leq \sum_i v_i$}
\end{align*}
Hence, in both cases, 
\begin{align*}
\norm{\nabla^2 f_i(\theta)}_{p \to q} & \leq  \sum_{k,l} |G^{(i)}_{k,l}| \, \normsq{x^{(i)}}_q
\end{align*}

Bounding $\sum_{k,l} |G^{(i)}_{k,l}| $, if $z := \pi_\theta^{(i)}$, then, 
\begin{align*}
\sum_{k,l} |G^{(i)}_{k,l}| & = \sum_k z_k \, (1-z_k) - \sum_{l \neq k} z_l \, z_k = 2 \, \sum_k z_k \, (1-z_k) \tag{Since $\sum_k z_k = 1$, $\sum_{l \neq k} z_l = 1 - z_k$} \\
& = 2 \, (1 - \norm{z}_2^2) \leq 2 \, (1 - [z(c_{i})]^2)  \tag{Since $\forall c \in [K]$, $1 - \normsq{p} \leq 1 - [p(c)]^2$} \nonumber \\
& \leq 4 \, \left(1 - z(c_{i}) \right) \tag{Since for all $z$, $1 - z^2 \leq 2 (1 - z)$} \nonumber \\
& \leq 4  \, \ln\left(\frac{1}{z(c_{i})} \right) \tag{For all $z \in [0,1]$, $1 - z \leq \ln(1/z)$} \nonumber \\
& = 4  \, \text{KL}(y^{(i)} || \pi_\theta^{(i)}) \tag{Using that $y^{(i)}$ is a one-hot vector and $z =  \pi_\theta^{(i)}$} \\
& = 4 \, f_i(\theta)
\end{align*}
Combining the above inequalities, 
\begin{align*}
\norm{\nabla^2 f_i(\theta)}_{p \to q} & \leq  4 \, \normsq{x^{(i)}}_q \, f_i(\theta)     \\
\implies \norm{\nabla^2 f(\theta)}_{p \to q} & \leq \frac{1}{n} \,\sum_{i = 1}^{n} \norm{\nabla^2 f_i(\theta)}_{p \to q} \leq  4 \, \max_i\, \normsq{x_i}_q \, f(\theta) \tag{Using triangle inequality and above relation} 
\end{align*}
This proves~\cref{assn:nus}. Note that, for all $i \in [n]$ and $j \in [K]$, if $\theta^* = [\theta_1^*, \theta_2^*, \ldots, \theta^*_K]$ with $\norm{\theta^*}_p \leq 1$ is the max-margin solution, then, 
\begin{align*}
\frac{\partial f_i(\theta)}{\partial \theta_{j}} & = \left( \pi_\theta^{(i)}(j) - \cI\{j = c_{i} \} \right) x^{(i)}  \\
\implies -\langle \theta^*, \nabla f_i(\theta) \rangle & = \sum_{j = 1}^{K} \left( \cI\{j = c_{i} \} - \pi_\theta^{(i)}(j)  \right) \langle x^{(i)}, \theta^*_j \rangle = \langle x^{(i)}, \theta^*_{c_{i}} \rangle - \sum_{j = 1}^{K} \pi_\theta^{(i)}(j) \, \langle x^{(i)}, \theta^*_j \rangle \\
& = \sum_{j = 1}^{K} \pi_\theta^{(i)}(j) \left[ \langle x^{(i)}, \theta^*_{c_{i}}  - \theta^*_j \rangle \right] =  \sum_{j = 1, j \neq c_{i}}^{K} \pi_\theta^{(i)}(j) \left[ \langle x^{(i)}, \theta^*_{c_{i}}  - \theta^*_j \rangle \right] \\
& \geq \gamma_p \, \sum_{j = 1, j \neq c_{i}}^{K} \pi_\theta^{(i)}(j) \tag{By definition of the margin $\gamma_p$} \\
\implies -\langle \theta^*, \nabla f_i(\theta) \rangle & \geq \gamma_p \, \left(1 -  \pi_\theta^{(i)}(c_{i}) \right) \\
\implies -\langle \theta^*, \nabla f(\theta) \rangle & \geq \frac{\gamma_p}{n} \,  \sum_{i = 1}^{n} \left(1 -  \pi_\theta^{(i)}(c_{i}) \right) \\
\intertext{By definition of the dual norm and using the above relation,}
\norm{\nabla f(\theta)}_{q} &= \max_{\norm{v}_p \leq 1} \langle v, -\nabla f(\theta) \rangle \geq \frac{\gamma_p}{n} \,  \sum_{i = 1}^{n} \left(1 -  \pi_\theta^{(i)}(c_{i}) \right)
\end{align*}
We will now relate $\pi_\theta^{(i)}(c_{i})$ to the loss. 

\textbf{Case (i)}: If $f(\theta) \leq \frac{\ln(2)}{n}$, then for all $i \in [n]$, 
\begin{align*}
f_i(\theta) & \leq \ln(2) \implies \pi_\theta^{(i)}(c_{i}) \geq \frac{1}{2} \implies 1 -  \pi_\theta^{(i)}(c_{i}) \geq \frac{-\ln(\pi_\theta^{(i)}(c_{i}))}{2} = \frac{f_i(\theta)}{2}  \tag{Since for all $u \in \left[\frac{1}{2},1\right]$, $1 - u \geq -\frac{\ln(u)}{2}$} \\
\implies \norm{\nabla f(\theta)}_{q} & \geq \frac{\gamma_p}{n} \,  \sum_{i = 1}^{n} \left(1 -  \pi_\theta^{(i)}(c_{i}) \right) \geq \frac{\gamma_p}{2} \, f(\theta)   
\end{align*}

\textbf{Case (ii)}: If $f(\theta) > \frac{\ln(2)}{n}$, 
\begin{align*}
\sum_{i = 1}^{n} f_i(\theta) \geq \ln(2) \implies - \sum_{i = 1}^{n} \ln(\pi_\theta^{(i)}(c_i)) \geq \ln(2)  & \implies \pi_\theta^{(1)}(c_1) \times \pi_\theta^{(2)}(c_2)  \times \ldots \times \pi_\theta^{(n)}(c_n) \leq \frac{1}{2} \\
\implies 1 - \sum_{i = 1}^{n} [1 - \pi_\theta^{(i)}(c_i)] & \leq \frac{1}{2} \implies \sum_{i = 1}^{n} [1 - \pi_\theta^{(i)}(c_i)] \geq \frac{1}{2} \tag{Since for probabilities $p_i$, $\prod_{i = 1}^{n} p_i \geq 1 - \sum_{i = 1}^{n} [1 - p_i]$} \\
\implies \norm{\nabla f(\theta)}_{q} \geq \frac{\gamma_p}{n} \,  \sum_{i = 1}^{n} \left(1 -  \pi_\theta^{(i)}(c_{i}) \right) & \geq \frac{\gamma_p}{2 \, n}
\end{align*}
 \end{proof}
\begin{restatable}{proposition}{glm}
Consider $n$ points where $x_i \in \R^D$ are the features and $y_i \in [0,1]$ are the corresponding labels. If $\pi_i(\theta) = \sigma(\langle x_i, \theta \rangle) := \frac{1}{1 + \exp(-\langle x_i, \theta \rangle)}$, the GLM objective,
\begin{align}
f(\theta) = \frac{1}{2n} \sum_{i = 1}^{n} \left(\pi_i(\theta) - y_i \right)^2 \,,
\label{eq:glm}    
\end{align}
satisfies~\cref{assn:non-negative} and~\cref{assn:nus} with $H_0 = \frac{3 \, \max_i \left\{\normsq{x_i}_{q}\right\}}{16}$ and $H_1 =  \frac{\max_i \left\{\normsq{x_i}_{q}\right\}}{4} $. Furthermore, assuming that for all $i \in [n]$, $\norm{x_i}_2 \leq 1$, $y_i = \pi_i(\theta^*)$ such that $\norm{\theta^*} \leq D < \infty$ and $\upsilon(\theta) := \min_{i \in [n]}{ \left\{ \pi_i(\theta) \cdot \left( 1 - \pi_i(\theta) \right) \right\} }$, then the GLM objective in~\cref{eq:glm}  satisfies~\cref{assn:pl} with $\tau = \frac{1}{2}$ and $\mu(\theta) = 64 \, [\upsilon(\theta)]^2 \, [\min \{\upsilon(\theta), \upsilon(\theta^*) \}]^2$. 
\label{prop:glm}
\end{restatable}
\begin{proof}
Clearly, $f_i(\theta) \geq 0$ and hence $f(\theta) \geq 0$ for all $\theta$. $f(\theta)$ is a finite-sum objective. Calculating the gradient and hessian for $f_i(\theta) = \frac{1}{2} \, \left(\pi_i(\theta) - y_i \right)^2$, 
\begin{align*}
\nabla f_i(\theta) &= \left(\pi_i(\theta) - y_i \right) \, \frac{1}{1 + \exp(-\langle x_i, \theta \rangle)} \, \frac{\exp(-\langle x_i, \theta \rangle)}{1 + \exp(-\langle x_i, \theta \rangle)} \, x_i = \left(\pi_i(\theta) - y_i \right) \, \pi_i(\theta) \, (1 - \pi_i(\theta)) \, x_i \\
\nabla^2 f_i(\theta) &= 
[1 - 2 \, \pi_i(\theta)] \, \pi_i(\theta) \, [1 - \pi_i(\theta)] \, [\pi_i(\theta) - y_i] \, x_i \, x_i^T + [\pi_i(\theta)]^2 \, [1 - \pi_i(\theta)]^2 \, x_i \, x_i^T \\
\implies \norm{\nabla^2 f_i(\theta)}_{p \to q} &= \underbrace{\left\vert[1 - 2 \, \pi_i(\theta)] \, \pi_i(\theta) \, [1 - \pi_i(\theta)] \, [\pi_i(\theta) - y_i] + [\pi_i(\theta)]^2 \, [1 - \pi_i(\theta)]^2 \right\vert}_{:= (*)} \, \normsq{x_i}_{q} \tag{Since for a rank one matrix $\norm{u u^T}_{p \to q} = \normsq{u}_q$} 
\end{align*}
To bound (*), define $a_i := \pi_i(\theta) - y_i$ and $b_i := \pi_i(\theta) \, [1 - \pi_i(\theta)]$, 
\begin{align*}
(*) &= \left\vert \, [1 - 2 \, \pi_i(\theta)] \, a_i \, b_i + b_i^2 \right\vert \leq \underbrace{|1 - 2 \pi_i(\theta)|}_{\leq 1} \, |a_i| \underbrace{|b_i|}_{\leq \frac{1}{4}} + \underbrace{|b_i|^2}_{\leq \frac{1}{16}} \\
& \leq \frac{|a_i|}{4} + \frac{1}{16} \leq \frac{a_i^2}{8} + \frac{1}{8} + \frac{1}{16} \tag{Since for all $x$, $|x| \leq \frac{x^2 + 1}{2}$} \\
& = \frac{a_i^2}{8} + \frac{3}{16} \\
(*) & \leq \frac{1}{4} \, \frac{1}{2} \, \left(\pi_i(\theta) - y_i \right)^2 + \frac{3}{16} = \frac{1}{4} \, f_i(\theta) + \frac{3}{16} \\
\implies \norm{\nabla^2 f_i(\theta)}_{p \to q} & \leq \frac{1}{4} \,  \normsq{x_i}_{q}f_i(\theta) + \frac{3}{16} \,  \normsq{x_i}_{q} \\
\implies \norm{\nabla^2 f(\theta)}_{p \to q} & \leq \frac{1}{n} \,\sum_{i = 1}^{n} \norm{\nabla^2 f_i(\theta)}_{p \to q} \leq  \frac{1}{4} \, \frac{\max_i \left\{\normsq{x_i}_{q}\right\}}{n} \, \sum_{i = 1}^{n} f_i(\theta) + \frac{3 \, \max_i \left\{\normsq{x_i}_{q}\right\}}{16} \tag{Using triangle inequality and above relation} \\
& = \frac{\max_i \left\{\normsq{x_i}_{q}\right\}}{4}f(\theta) + \frac{3 \, \max_i \left\{\normsq{x_i}_{q}\right\}}{16} 
\end{align*}
\cref{assn:pl} follows from~\citep[Lemma 9]{mei2021leveraging}. 
\end{proof}

\subsection{Helper Lemmas}
\label{app:example-helper}

\begin{lemma}
For $A \in \mathbb{R}^{m \times n}$ and $B \in \mathbb{R}^{r \times s}$ and $p, q$ such that $\frac{1}{p} + \frac{1}{q} = 1$,
\begin{itemize}
    \item If $q \in [1,\infty)$, $\norm{A \otimes B}_{p \to q} \leq \left(\sum_{i,j} |A_{i,j}|^{q} \right)^{1/q}  \, \left(\sum_{k,l} |B_{k,l}|^{q} \right)^{1/q}$

    \item Else if $q = \infty$, $\norm{A \otimes B}_{p \to q} \leq \max_{i,j }|A_{ij}| \, \max_{k,l} |B_{k,l}|$
\end{itemize}
\label{lemma:krondecomp}
\end{lemma}

\begin{proof}
For proving the first part, note that, for any matrix $M$ and vector $x$, for $q \in [1, \infty)$, 
\begin{align}
\norm{Mx}_q^q &= \sum_{i}\left|\sum_{j}M_{ij} \, x_j\right|^q =  \sum_{i} \left|\langle M_{i,:} , x \rangle \right|^q \tag{By definition of $\norm{\cdot}_q$}\\
&\leq \sum_{i} \norm{M_{i,:}}^{q}_{q} \, \norm{x}^{q}_p \tag{Using Holder's inequality} \\
\implies \norm{Mx}_q & \leq \left(\sum_{i} \norm{M_{i,:}}^{q}_{q} \, \norm{x}^{q}_p \right)^{1/q} = \norm{x}_p \, \left(\sum_{i} \norm{M_{i,:}}^{q}_{q} \right)^{1/q} \\
\implies \norm{M}_{p \to q} &= \sup_{x\neq 0}\frac{\norm{Mx}_q}{\norm{x}_p} \leq \left(\sum_{i} \norm{M_{i,:}}^{q}_{q} \right)^{1/q} = \left(\sum_{i,j} |M_{i,j}|^{q} \right)^{1/q} \label{eq:kron-inter-1}
\end{align}
By definition of the Kronecker product and the above inequality, 
\begin{align*}
\norm{A\otimes B}_{p \to q} & = \left(\sum_{i,j} \sum_{k,l} |A_{i,j}|^{q} \, |B_{k,l}|^{q} \right)^{1/q} \leq  \left(\sum_{i,j} |A_{i,j}|^{q} \right)^{1/q}  \, \left(\sum_{k,l} |B_{k,l}|^{q} \right)^{1/q} 
\end{align*}

For the second part, note that, for any matrix $M$, if $q  = \infty$, $p = 1$ and hence, 
\begin{align}
\norm{M}_{1\to\infty} &= \sup_{\norm{x}_1\leq 1}\norm{Mx}_\infty = \sup_{\norm{x}_1\leq 1}\max_i\left|\sum_j M_{ij}x_j\right| \tag{By definition of $\norm{\cdot}_\infty$}\\
&= \max_i\sup_{\norm{x}_1\leq 1}\left|\sum_j M_{ij}x_j\right| = \max_i\sup_{\norm{x}_1\leq 1}|\langle M_{i,:},x\rangle| = \max_i \norm{M_{i,:}}_\infty \tag{By definition of dual norm}\\
\implies \norm{M}_{1\to\infty} &= \max_i\max_j |M_{ij}| \label{eq:kron-inter-2}
\end{align}
By definition of the Kronecker product and the above inequality, 
\begin{align*}
\norm{A\otimes B}_{1 \to \infty} & = \max_{i,j} \, \max_{k,l} |A_{ij} B_{k,l}| \leq \max_{i,j }|A_{ij}| \, \max_{k,l} |B_{k,l}|  
\end{align*}
\end{proof}
\section{Properties of $(H_0, H_1)$-$\NS$ Functions}
\label{app:nus-prop}
\begin{lemma}
Consider a function $g: \R \to \R$. If $g(t) \geq 0$ for all $t$, and for constants $H_0 \geq 0, H_1 > 0$, $g''(t) \leq H_0 + H_1 \, g(t)$, then, $g'(0) \leq \sqrt{2 H_0 \, g(0) + H_1 \, g(0)^2}$. 
\label{lemma:1d-domination-gen}    
\end{lemma}
\begin{proof}
For convenience, we define $\lambda = \sqrt{H_1}$ and $c = \frac{H_0}{H_1}$. Define the function $$\phi(t) := (g(0) + c) \, \cosh(\lambda t) + \frac{g'(0)}{\lambda} \, \sinh(\lambda t) - c,$$ where $\sinh(x) := \frac{\exp(x) - \exp(-x)}{2}$ and $\cosh(x) = \frac{\exp(x) + \exp(-x)}{2}$ are hyperbolic sine and cosine functions. We can verify the following relations: $\phi(0) = g(0)$, and 
\begin{align*}
\phi'(t) &= \frac{(g(0) + c) \, \lambda}{2} \, (\exp(\lambda \, t) - \exp(-\lambda \, t)) + \frac{g'(0)}{2} \, (\exp(\lambda t) + \exp(-\lambda t)) \implies \phi'(0) = g'(0) \\
\phi''(t) &= \frac{(g(0) + c) \, \lambda^2}{2} \, (\exp(\lambda \, t) + \exp(-\lambda \, t)) + \frac{g'(0) \, \lambda}{2} \, (\exp(\lambda t) - \exp(-\lambda t)) \implies \phi''(t) = \lambda^2 \, (\phi(t) + c) = H_0 + H_1 \, \phi(t)
\end{align*}
Hence, $\phi(t)$ satisfies a similar condition as $g$, but with an equality. We now show that $\phi(t)$ upper-bounds $g(t)$. For this define $$\Delta(t) = \phi(t) - g(t)$$ and verify that (i) $\Delta''(t) \geq \lambda^2 \, \Delta(t)$, (ii) $\Delta'(0) = 0$ and (iii) $\Delta(0) = 0$. We define $H(t)$ and calculate its derivative:
\begin{align*}
H(t) := \exp(-\lambda \, t) \left[ \Delta'(t) + \lambda \, \Delta(t) \right] \quad \text{;} \quad H'(t) = \exp(-\lambda t) \left[ \Delta''(t) - \lambda^2 \Delta(t) \right] \geq 0  
\end{align*}
Hence, $H(t)$ is a non-decreasing function, and hence for all $t \geq 0$, $H(t) \geq H(0) = 0$. Hence, for all $t \geq 0$,
\begin{align*}
& H(t) = \exp(-\lambda \, t) \left[ \Delta'(t) + \lambda \, \Delta(t) \right] \geq 0 \implies \exp(2 \lambda(t) ) \, H(t) = \exp(\lambda \, t) \left[ \Delta'(t) + \lambda \, \Delta(t) \right] \geq 0 \\
\implies & \frac{\partial \exp(\lambda t) \, \Delta(t)}{\partial t} \geq 0
\end{align*}
Hence, $h(t) := \exp(\lambda t) \, \Delta(t)$ is a non-decreasing function, and hence, for all $t \geq 0$, $h(t) \geq h(0) = 0$. This implies that $\Delta(t) \geq 0$ for all $t \geq 0$. \\ 
Hence, for $t \geq 0$, $\phi(t) \geq g(t)$. By defining $\tilde{H}(t) := \exp(\lambda \, t) \left[ \Delta'(t) - \lambda \, \Delta(t) \right]$ and following an analogous argument, we prove that for all $t \leq 0$, $\phi(t) \geq g(t)$. 
\begin{align*}
\tilde{H}(t) := \exp(\lambda \, t) \left[ \Delta'(t) - \lambda \, \Delta(t) \right] \quad \text{;} \quad \tilde{H}'(t) = \exp(\lambda t) \left[ \Delta''(t) - \lambda^2 \Delta(t) \right] \geq 0  
\end{align*}
Hence, $\tilde{H}(t)$ is a non-decreasing function, and hence for all $t \leq 0$, $\tilde{H}(t) \leq \tilde{H}(0) = 0$. Hence, for all $t \leq 0$,
\begin{align*}
& \tilde{H}(t) = \exp(\lambda \, t) \left[ \Delta'(t) - \lambda \, \Delta(t) \right] \leq 0 \implies \exp(-2 \lambda(t) ) \, \tilde{H}(t) = \exp(-\lambda \, t) \left[ \Delta'(t) - \lambda \, \Delta(t) \right] \leq 0 \\
\implies & \frac{\partial \exp(-\lambda t) \, \Delta(t)}{\partial t} \leq 0    
\end{align*}
Hence, $\tilde{h}(t) := \exp(-\lambda t) \, \Delta(t)$ is a non-increasing function, and hence, for all $t \leq 0$, $\tilde{h}(t) \geq \tilde{h}(0) = 0$. This implies that $\Delta(t) \geq 0$ for all $t \leq 0$. \\ 
Hence, for all $t \in \R$, $\phi(t) \geq g(t) \geq 0 \implies \min_{t \in \R} \phi(t) \geq 0$. Therefore, $\phi''(t)\geq 0$ for all $t$ and $\phi(t)$ is convex. \\ 
Therefore, we can minimize $\phi(t)$ by setting $\phi'(t) = 0$:
\[
\phi'(t) = (g(0)+c)\lambda \sinh(\lambda t) + g'(0) \cosh(\lambda t) = 0 \iff \tanh(\lambda t^*) = -\frac{g'(0)}{(g(0)+c)\lambda }
\]
so since $\cosh(\tanh^{-1}(x))=\frac{1}{\sqrt{1-x^2}}$
then
\[
\min_{t\in R}\phi(t)
= \left((g(0)+c) - \frac{g'(0)^2}{\lambda^2(g(0)+c)}\right) \cosh(\lambda t^*) - c 
= \frac{1}{\lambda}\sqrt{(g(0)+c)^2\lambda^2 - g'(0)^2} - c.
\] 
Therefore
\begin{align*}
(g(0) + c)^2 - \frac{g'(0)^2}{\lambda^2} \geq c^2 \tag{since $\phi(t) \geq 0$ for all $t \in \R$} \\
\implies & g'(0) \leq \lambda \, [\sqrt{2 c g(0) + g(0)^2}] = \sqrt{2 H_0 \, g(0) + H_1 \, g(0)^2}
\end{align*}
\end{proof}

\begin{lemma}
If~\cref{assn:non-negative,assn:nus} hold with $H_0 \geq 0, H_1 \geq 0$, then, for all $\theta$, $\norm{\nabla f(\theta)}_{q} \leq \sqrt{2 H_0 \, f(\theta) + H_1 \, [f(\theta)]^2}$.
\label{lemma:nus-grad-gen}
\end{lemma}
\begin{proof}
If $H_1 = 0$, then using the standard descent lemma and the non-negativity of $f$ gives us that, for all $\theta$, 
\begin{align*}
\normsq{\nabla f(\theta)}_{q} & \leq 2 H_0 \, f(\theta)
\end{align*}

For the case where $H_1 > 0$, we define $g(t) := f(\theta + t u)$ s.t. $\norm{u}_p \leq 1$. From assumption \ref{assn:non-negative}, we know that $g(t) \geq 0$ for all $t$. Furthermore, 
\begin{align*}
g'(t) & = \langle \nabla f(\theta + t \, u), u \rangle \\
g''(t) &= u^T \, \nabla^2 f(\theta + t \, u) \, u \leq \norm{u}_{p} \, \norm{\nabla^2 f(\theta + t \, u) \, u}_{q} \tag{Holder's inequality} \\
& = \normsq{u}_p \, \norm{\nabla^2 f(\theta + t \, u)}_{p \to q} \tag{Definition of matrix norm} \\
& \leq \normsq{u}_p \, (H_0 + H_1 \, f(\theta + t \, u))  \tag{Using assumption \ref{assn:nus}} \\
\implies g''(t) & \leq H_0 + H_1 \, g(t) \tag{Since $\norm{u}_p \leq 1$ and using assumption \ref{assn:non-negative} }\\
\end{align*}
Using~\cref{lemma:1d-domination-gen} for $H_0 \geq 0$ and $H_1 > 0$, we get that, $g(t) \geq 0$ and $g''(t) \leq H_0 + H_1 \, g(t)$, then, $g'(0) \leq \sqrt{2 H_0 \, g(0) + H_1 \, g(0)^2}$. Hence, 
\begin{align*}
\langle \nabla f(\theta), u \rangle & \leq \sqrt{2 H_0 \, f(\theta) + H_1 \, [f(\theta)]^2} \implies \max_{\norm{u}_p \leq 1} \langle \nabla f(\theta), u \rangle \leq \sqrt{2 H_0 \, f(\theta) + H_1 \, [f(\theta)]^2} \\
\implies \norm{\nabla f(\theta)}_{q} & \leq \sqrt{2 H_0 \, f(\theta) + H_1 \, [f(\theta)]^2} \tag{By definition of the dual norm}
\end{align*}
\end{proof}

\begin{lemma}
If~\cref{assn:non-negative,assn:nus} hold with $H_0 = 0$, then, $f(y) \leq f(x) \exp\left(\sqrt{H_1} \norm{y - x}_p \right)$.     
\label{lemma:func-multi-lip}
\end{lemma}
\begin{proof}
Define $g(\theta) := \ln(f(\theta))$, and note that $\norm{\nabla g(\theta)}_q = \frac{\norm{\nabla f(\theta)}_{q}}{f(\theta)} \leq \sqrt{H_1}$, as a consequence of assumption \ref{assn:nus} and Lemma \ref{lemma:nus-grad-gen}. Hence, $g$ is $\sqrt{H_1}$-Lipschitz. Hence, for all  $y, x$
\begin{align*}
g(y) - g(x) \leq \sqrt{H_1} \, \norm{y - x}_{p} \implies \ln\left(\frac{f(y)}{f(x)}\right) \leq \sqrt{H_1} \, \norm{y - x}_p \implies f(y) \leq f(x) \, \exp\left( \sqrt{H_1} \, \norm{y - x}_p \right).  
\end{align*}
\end{proof}

\begin{lemma}
If~\cref{assn:non-negative,assn:nus} hold with $H_1 > 0$, then, $f(y) \leq \left(f(x) + \frac{H_0}{H_1}\right) \exp\left(\sqrt{H_1} \norm{y - x}_p \right) - \frac{H_0}{H_1} \leq \left(f(x) + \frac{H_0}{H_1}\right) \exp\left(\sqrt{H_1} \norm{y - x}_p \right)$.     
\label{lemma:func-multi-lip-gen}
\end{lemma}
\begin{proof}
Define $\tilde{f}(\theta):= f(\theta) + \frac{H_0}{H_1}$. We will show that if $f$ satisfies assumption \ref{assn:nus} with $H_0 > 0, H_1 > 0$, then $\tilde{f}$ satisfies assumption \ref{assn:nus} with $H_0 = 0, H_1 > 0$. For all $\theta$, 
\begin{align*}
\norm{\nabla^2 \tilde{f}(\theta)}_{p \to q} &= \norm{\nabla^2 f(\theta)}_{p \to q} \leq H_0 + H_1 \, f(\theta) \tag{Using assumption \ref{assn:nus} for $f$} \\
& = H_0 + H_1 \, \left( \tilde{f}(\theta) - \frac{H_0}{H_1} \right) = H_0 + H_1 \, \tilde{f}(\theta) - H_0 = H_1 \, \tilde{f}(\theta) \tag{Since $\tilde{f}(\theta) = f(\theta) + \frac{H_0}{H_1}$} \\
\implies \norm{\nabla^2 \tilde{f}(\theta)}_{p \to q} & \leq H_1 \tilde{f}(\theta)
\end{align*}
Hence, $\tilde{f}(\theta)$ satisfies assumption \ref{assn:nus} with $H_0 = 0$. Using~\cref{lemma:func-multi-lip} for $\tilde{f}$, 
\begin{align*}
\tilde{f}(y) & \leq \tilde{f}(x) \exp\left(\sqrt{H_1} \norm{y - x}_p \right) \\
\implies f(y) + \frac{H_0}{H_1} & \leq \left(f(x) + \frac{H_0}{H_1}\right) \exp\left(\sqrt{H_1} \norm{y - x}_p \right) \\
\implies f(y) & \leq \left(f(x) + \frac{H_0}{H_1}\right) \exp\left(\sqrt{H_1} \norm{y - x}_p \right) - \frac{H_0}{H_1}
\end{align*}
\end{proof}

\begin{lemma}
If~\cref{assn:non-negative,assn:nus} hold with $H_0 \geq 0, H_1 \geq 0$, then, for all $y,x$ s.t. $\norm{y - x}_p \leq \frac{1}{\sqrt{H_1}}$, 
\[
\norm{\nabla f(y) - \nabla f(x)}_{q} \leq [H_0 + H_1 \, f(x)] \, e \, \norm{y - x}_p.
\]
\label{lemma:grad-lip-gen}
\end{lemma}
\begin{proof}
By the fundamental theorem of calculus, 
\begin{align*}
\nabla f(y) - \nabla f(x) & = \int_{t = 0}^{1} \nabla^2 f((1-t) \, x + t \, y) \, (y - x) \, dt \\
\implies \norm{\nabla f(y) - \nabla f(x)}_q & = \norm{\int_{t = 0}^{1} \nabla^2 f((1-t) \, x + t \, y) \, (y - x) \, dt}_q \\
& \leq \int_{t = 0}^{1} \, \norm{\nabla^2 f((1-t) \, x + t \, y) \, (y - x)}_{q} dt \tag{Triangle inequality} \\
& \leq \int_{t = 0}^{1} \, \norm{\nabla^2 f((1-t) \, x + t \, y)}_{p \to q} \, \norm{y - x}_p \, dt \tag{By definition of matrix norm} \\
& \leq \norm{y - x}_p \, \left[\int_{t = 0}^{1} \, H_0 + H_1 \, f((1-t) \, x + t \, y) \, dt \right] \tag{Using assumption \ref{assn:nus}}\\
& = \norm{y - x}_p \, \left[H_0 + H_1 \, \int_{t = 0}^{1} \, f((1-t) \, x + t \, y) \, dt \right] \\
& \leq \norm{y - x}_p \, \left[ H_1 \, \int_{t = 0}^{1} \, \left(f(x) + \frac{H_0}{H_1} \right)\, \exp(\sqrt{H_1} \, t \, \norm{y - x}) \, dt \right] \tag{Using~\cref{lemma:func-multi-lip-gen} with $\theta = ty + (1-t)x$ and $\theta' = x$} \\
& = \norm{y - x}_p \, \left[H_1 \left(f(x) + \frac{H_0}{H_1} \right) \, \int_{t = 0}^{1} \, \exp(\sqrt{H_1} \, t \, \norm{y - x}) \, dt \right]  \\
& \leq \norm{y - x}_p \, \left[H_1 \left(f(x) + \frac{H_0}{H_1} \right) \, \int_{t = 0}^{1} \, \exp(t) \, dt \right]  \tag{Since $\norm{y - x}_p \leq \frac{1}{\sqrt{H_1}}$} \\
& = \left[H_1 \left(f(x) + \frac{H_0}{H_1} \right) \, (e-1) \right]\, \norm{y - x}_p \\
& \leq [H_0 + H_1 f(x)] \,  e \, \norm{y - x}_p 
\end{align*}
\end{proof}
\begin{lemma}
If~\cref{assn:non-negative,assn:nus} hold with $H_0 \geq 0, H_1 \geq 0$, for all $y, x$ s.t. $\norm{y - x} \leq \frac{1}{\sqrt{H_1}}$, 
\[
f(y) \leq f(x) + \langle \nabla f(x), y - x \rangle + \left(H_0+ H_1 \, f(x) \right) \, \normsq{y - x}_p
\] 
\label{lemma:descent-ineq-gen}
\end{lemma}
\begin{proof}
Define $u(t) = (1-t) x + t y$ and $g(t) = f(u(t))$. Use Taylor's theorem for $g$,
\begin{align*}
g(b) = g(a) + (b - a) \, g'(a) + \int_{t = a}^{b} (b - t) \, g''(t) \, dt \\
\implies g(1) = g(0) + g'(0) + \int_{t = 0}^{1} (1 - t) \, g''(t) \, dt \tag{Substituting $a = 0$ and $b = 1$} 
\end{align*}
We know that, 
\begin{align*}
g'(t) &= \frac{\partial f(u(t))}{\partial t} = \langle \nabla f(u(t)), y - x \rangle \quad \text{;} \quad g''(t) = \frac{\partial^2 f(u(t))}{\partial t^2} = (y - x)^T \nabla^2 f(u(t))  (y - x)
\end{align*}
Combining the above relations, and using that $g(1) = f(y)$, $g(0) = f(x)$, $g'(0)  = \langle \nabla f(x), y - x \rangle$. 
\begin{align*}
f(y) &= f(x) + \langle \nabla f(x), y - x \rangle + (y-x)^T \left[\int_{t = 0}^{1} (1-t) \, \nabla^2 f(t \, y + (1-t) \, x) \, dt \right]\, (y-x) \\
\intertext{Simplifying the last term,}
& (y-x)^T \left[\int_{t = 0}^{1} (1-t) \, \nabla^2 f(t \, y + (1-t) \, x) \, dt \right]\, (y-x) \\
&\leq \norm{y - x}_p \, \norm{\left[\int_{t = 0}^{1} (1-t) \, \nabla^2 f(t \, y + (1-t) \, x) \, dt \right]\, (y-x)}_{q} \tag{Holder's inequality} \\
& \leq \norm{y - x}_p \, \left[\int_{t = 0}^{1} (1-t) \, \norm{\nabla^2 f(t \, y + (1-t) \, x) \, (y-x)}_{q} \, dt \right] \tag{Triangle inequality} \\
& \leq \normsq{y - x}_p \,  \int_{t = 0}^{1} (1-t) \, \norm{\nabla^2 f(t \, y + (1-t) \, x)}_{p \to q} \, dt \tag{By definition of matrix norm} \\
& \leq \normsq{y - x}_p \,  \int_{t = 0}^{1} (1-t) \, [H_0 + H_1 \, f(t \, y + (1-t) \, x)] \, dt \tag{Using assumption \ref{assn:nus}} \\
& \leq \normsq{y - x}_p \,  \int_{t = 0}^{1} (1-t) \, \left(H_0 + H_1 \, \left[ \left(f(x) + \frac{H_0}{H_1} \right) \exp(\sqrt{H_1} \, t \, \norm{y - x}_p)  - \frac{H_0}{H_1} \right] \right)\, dt \tag{Using~\cref{lemma:func-multi-lip-gen} with $\theta = ty + (1-t)x$ and $\theta' = x$} \\
& = \normsq{y - x}_p \,  \int_{t = 0}^{1} (1-t) \, \left((H_0 + H_1 \, f(x)) \, \exp(\sqrt{H_1} \, t \, \norm{y - x}_p)  \right)\, dt \\
& = \normsq{y - x}_p \, (H_0 + H_1 \, f(x))\, \int_{t = 0}^{1} (1-t) \, \left( \, \exp(\sqrt{H_1} \, t \, \norm{y - x}_p)  \right)\, dt \\
& \leq  \normsq{y - x}_p \, (H_0 + H_1 \, f(x))\, \int_{t = 0}^{1} (1-t) \, \left( \, \exp(  t   )  \right)\, dt \tag{Since $\norm{y - x}_p \leq \frac{1}{\sqrt{H_1}}$} \\
& \leq \normsq{y - x}_p \, (H_0 + H_1 \, f(x))\, \left[\int_{t = 0}^{1} \exp(t) \, dt - \int_{t = 0}^{1} t \, \exp(t) \, dt \right]  \\
& \leq \normsq{y - x}_p \, (H_0 + H_1 \, f(x))\, (e-2)  \\
& \leq \normsq{y - x}_p \, (H_0 + H_1 \, f(x))\\
\end{align*}
Putting everything together, 
\begin{align*}
f(y) &\leq f(x) + \langle \nabla f(x), y - x \rangle + \left(H_0 + H_1 \, f(x) \right) \, \normsq{y - x}_p
\end{align*}
\end{proof}

\begin{lemma}
If~\cref{assn:non-negative,assn:nus} hold with $H_0 = 0$, then, for all $x, y$, $\norm{\frac{\nabla f(y)}{f(y)} - \frac{\nabla f(x)}{f(x)}}_q \leq 2 \, H_1 \, \norm{y - x}_p$.     
\label{lemma:log-smoothness}
\end{lemma}
\begin{proof}
\begin{align*}
\nabla^2 \ln(f(\theta)) &= \frac{\nabla^2 f(\theta)}{f(\theta)} - \frac{[\nabla f(\theta)] \, [\nabla f(\theta)]^{T}}{[f(\theta)]^2} \\
\implies \norm{\nabla^2 \ln(f(\theta))}_{p \to q} & \leq \norm{\frac{\nabla^2 f(\theta)}{f(\theta)}}_{p \to q} + \norm{\frac{[\nabla f(\theta)] \, [\nabla f(\theta)]^{T}}{[f(\theta)]^2}}_{p \to q} \tag{Triangle inequality} \\
& \leq H_1 + \frac{1}{[f(\theta)]^2} \, \norm{[\nabla f(\theta)] \, [\nabla f(\theta)]^{T}}_{p \to q} \tag{Using assumption \ref{assn:nus}} \\
& \leq H_1 + \frac{\normsq{\nabla f(\theta)}_q}{[f(\theta)]^2} \tag{Since for any vector $x$, $\norm{x x^T}_{p \to q} = \normsq{x}_{q}$} \\
\implies \norm{\nabla^2 \ln(f(\theta))}_{p \to q}  & \leq 2 \,  H_1 \tag{Using~\cref{lemma:nus-grad-gen} with $H_0 = 0$}
\end{align*}    
By the fundamental theorem of calculus, 
\begin{align*}
\nabla \ln(f(y)) - \nabla \ln(f(x)) & = \int_{t = 0}^{1} \nabla^2 \ln(f(t \, y + (1-t) x)) (y-x) \, dt \\
\implies \norm{\nabla \ln(f(y)) - \nabla \ln(f(x))}_{q} & \leq \int_{t = 0}^{1} \norm{\nabla^2 \ln(f(t \, y + (1-t) x)) (y-x)}_{q} \, dt \tag{Triangle inequality} \\
& \leq \int_{t = 0}^{1} \, \norm{\nabla^2 \ln(f(t \, y + (1-t) x)) }_{p \to q} \, \norm{y-x}_p \, dt \tag{By definition of matrix norm} \\
& \leq 2 \, H_1  \, \norm{y - x}_p \tag{Using the above relation} \\
\implies \norm{\frac{\nabla f(y)}{f(y)} - \frac{\nabla f(x)}{f(x)}}_q & \leq 2 \, H_1 \, \norm{y - x}_p     
\end{align*}
\end{proof}

\begin{lemma}
If~\cref{assn:non-negative,assn:nus} hold with $H_1 > 0$, then, for all $x, y$ s.t. $\norm{y - x}_p \leq \frac{1}{\sqrt{H_1}}$, $\norm{\frac{\nabla f(y)}{f(y) + \frac{H_0}{H_1}} - \frac{\nabla f(x)}{f(x) + \frac{H_0}{H_1}}}_q \leq 2 \, H_1 \, \norm{y - x}_p$.     
\label{lemma:log-smoothness-gen}
\end{lemma}
\begin{proof}
Define $\tilde{f}(\theta):= f(\theta) + \frac{H_0}{H_1}$. We will show that if $f$ satisfies assumption \ref{assn:nus} with $H_0 > 0, H_1 > 0$, then $\tilde{f}$ satisfies assumption \ref{assn:nus} with $H_0 = 0, H_1 > 0$. For all $\theta$, 
\begin{align*}
\norm{\nabla^2 \tilde{f}(\theta)}_{p \to q} &= \norm{\nabla^2 f(\theta)}_{p \to q} \leq H_0 + H_1 \, f(\theta) \tag{Using assumption \ref{assn:nus} for f} \\
& = H_0 + H_1 \, \left( \tilde{f}(\theta) - \frac{H_0}{H_1} \right) = H_0 + H_1 \, \tilde{f}(\theta) - H_0 = H_1 \, \tilde{f}(\theta) \tag{Since $\tilde{f}(\theta) = f(\theta) + \frac{H_0}{H_1}$} \\
\implies \norm{\nabla^2 \tilde{f}(\theta)}_{p \to q} & \leq H_1 \tilde{f}(\theta)
\end{align*}
Hence, $\tilde{f}(\theta)$ satisfies assumption \ref{assn:nus} with $H_0 = 0$. Using~\cref{lemma:log-smoothness} for $\tilde{f}$ and noting that $\nabla \tilde{f}(\theta) = f(\theta)$, then, for all $x, y$,      
\begin{align*}
\norm{\frac{\nabla \tilde{f}(y)}{\tilde{f}(y)} - \frac{\nabla \tilde{f}(x)}{\tilde{f}(x)}}_q \leq 2 \, H_1 \, \norm{y - x}_p  \\
\implies \norm{\frac{\nabla f(y)}{f(y) + \frac{H_0}{H_1}} - \frac{\nabla f(x)}{f(x) + \frac{H_0}{H_1}}}_q \leq 2 \, H_1 \, \norm{y - x}_p
\end{align*}
\end{proof}

\begin{lemma}
If~\cref{assn:non-negative,assn:nus} hold with $H_0 \geq 0$, $H_1 \geq 0$ and $\frac{1}{p} + \frac{1}{q} = 1$, then, for all $y, x$ s.t. $\norm{y - x}_p \leq \frac{1}{\sqrt{H_1}}$ and $c > 0$, 
\begin{align*}
\norm{\nabla f(y)}_q  + c & \leq \left(\norm{\nabla f(x)}_q + c \right) \, \exp\left((H_0 + H_1 \, f(x)) \, e \, \frac{\norm{y - x}_p}{c} \right) 
\end{align*}
\label{lemma:grad-multi-lip-gen}
Furthermore, if $H_1 > 0$ and assumption \ref{assn:pl} holds with $\zeta = 1$ and $\mu$, then, for all $y, x$ and choosing $c = \frac{\mu \, (H_0 + H_1 \, f^*)}{H_1}$,
\begin{align*}
\norm{\nabla f(y)}_q + c & \leq \left(\norm{\nabla f(x)}_q + c \right) \,  \exp \left(\frac{H_1}{\mu} \norm{y - x}_{p} \right) \,,
\end{align*}
\end{lemma}
\begin{proof}
Define the function $h(\theta) := \ln(\norm{\nabla f(\theta)}_q  + c)$. We will first prove that $h(\theta)$ is $\frac{H_1}{\mu}$-Lipschitz w.r.t the $\ell_{p}$ norm. Since $\norm{\cdot}_q$ can be non-smooth, consider a Clarke subgradient computed using Lemma \ref{lem:clarkesubdiffchainrule}
\[
\partial h(\theta) = \frac{\nabla^2 f(\theta)  \partial \|z\|_q}{\|z\|_q+c}, \qquad z =  \nabla f(\theta).
\]
\begin{align*}
\intertext{Then, for $g = \nabla^2 f(\theta) s$ where $s \in \partial(\norm{\cdot}_q)(\nabla f(\theta))$ is a subgradient of $\norm{\cdot}_{q}$ evaluated at $\nabla f(\theta)$.}
\norm{v}_{q} &= \frac{\norm{[\nabla^2 f(\theta)] \, s}_q}{\norm{\nabla f(\theta)}_q + c} \leq \frac{\norm{\nabla^2 f(\theta)}_{p \to q} \, \norm{s}_{p}}{\norm{\nabla f(\theta)}_q + c} \tag{By definition of the matrix norm} \\
\implies \norm{v}_{q} & \leq \frac{\norm{\nabla^2 f(\theta)}_{p \to q}}{\norm{\nabla f(\theta)}_q + c} \tag{If $s \in \partial(\norm{\cdot}_q)$, then, $\norm{s}_p \leq 1$} \\
\end{align*}
Using assumption \ref{assn:nus} to simplify the numerator, and noting that lower-bounding the denominator, 
\begin{align*}
\forall v \in \partial{h}(\theta), \quad \norm{v}_q & \leq \frac{H_0 + H_1 \, f(\theta)}{c}    
\end{align*}
By Lebourg's mean value theorem (\cite{clarke1998nonsmooth}, Thm. 2.4), since $h$ is Lipschitz in an open set containing $y$ and $x$, then there exists a point $u = tx+(1-t)y$, for some $t\in[0,1]$ such that 
\[
h(y)-h(x)\in \langle g,y-x\rangle, \qquad g\in \partial h(u).
\]
Therefore,
\begin{align*}
h(y) - h(x) &\leq \max_{t\in[0,1]}  \langle g,y-x\rangle \quad {where} \quad g \in \partial h(t \, x + (1-t) \, y) \\
& \leq \norm{y - x}_p \, \max_{t\in[0,1]}\,\|g\|_q
 \tag{Using Holder's inequality} \\
& \leq \frac{\norm{y - x}_p}{c} \, \max_{t\in[0,1]} \, (H_0 + H_1 \, f(t \, x + (1-t) \, y)),  \tag{Using the above bound on the subgradient} \\
& \leq \frac{\norm{y - x}_p}{c} \,\max_{t\in[0,1]}  \left[H_0 + H_1 \, \left( \left(f(x) + \frac{H_0}{H_1} \right) \, \exp\left(\sqrt{H_1} \, t \, \norm{y - x}_p \right) - \frac{H_0}{H_1}\right) \right]  \tag{Using~\cref{lemma:func-multi-lip-gen}} \\
& = \frac{\norm{y - x}_p}{c} \,\max_{t\in[0,1]}  (H_0 + H_1 \, f(x)) \,  \exp\left(\sqrt{H_1} \, t \, \norm{y - x}_p \right) \\
& \leq\, \frac{\norm{y - x}_p}{c} \, \max_{t\in[0,1]} \, (H_0 + H_1 \, f(x)) \,  \exp\left(t \right)  \tag{Since $\norm{y  - x}_p \leq \frac{1}{\sqrt{H_1}}$} \\
& \leq\, \frac{\norm{y - x}_p}{c} \,  \, (H_0 + H_1 \, f(x)) \,e  
\end{align*}
\begin{align*}
\implies & \ln(\norm{\nabla f(y)}_q  + c) - \ln(\norm{\nabla f(x)}_q  + c)  \leq  (H_0 + H_1 \, f(x)) \, e \, \frac{\norm{y - x}_p}{c} \\
\implies & \norm{\nabla f(y)}_q  + c \leq \left(\norm{\nabla f(x)}_q  + c \right) \, \exp\left((H_0 + H_1 \, f(x)) \, e \, \frac{\norm{y - x}_p}{c} \right) 
\end{align*}
This proves the first part of the lemma. For the second part, using assumption \ref{assn:nus} to simplify the numerator of the above relation, 
\begin{align*}
\norm{\nabla^2 f(\theta)}_{p \to q} & \leq H_0 + H_1 \, f(\theta) = [H_0 + H_1 \, f^*] + H_1 \, [f(\theta) - f^*] \\
& \leq  [H_0 + H_1 \, f^*] + \frac{H_1}{\mu} \, \norm{\nabla f(\theta)}_{q} \tag{Using assumption \ref{assn:pl} with $\zeta = 1$ and $\mu$} \\
& = \frac{H_1}{\mu} \, \left[ \norm{\nabla f(\theta)}_{q} + \frac{\mu \, [H_0 + H_1 \, f^*]}{H_1} \right] = \frac{H_1}{\mu} \, \left[\norm{\nabla f(\theta)}_{q} + c \right] \tag{By definition of $c$} \\
\implies \forall v \in \partial{h}(\theta), \quad \norm{v}_q & \leq \frac{H_1}{\mu}
\end{align*}
Again, using the mean value theorem
\begin{align*}
h(y) - h(x) &\leq \max_{t\in[0,1]}  \langle g, y - x \rangle  \quad {where} \quad g \in \partial h(t \, x + (1-t) \, y) \\
& \leq \norm{y - x}_p \, \max_{t\in[0,1]} \norm{g}_q  \tag{Using Holder's inequality} \\
& \leq \norm{y - x}_p \,   \frac{H_1}{\mu}  \tag{Using the above bound on the subgradient} \\
\implies h(y) - h(x) & \leq \frac{H_1}{\mu} \, \norm{y - x}_p \\
\implies \norm{\nabla f(y)}_q + c & \leq \left(\norm{\nabla f(x)}_q + c \right) \,  \exp \left(\frac{H_1}{\mu} \norm{y - x}_{p} \right) \,,
\end{align*}
\end{proof}

\begin{lemma}
If $f$ satisfies~\cref{assn:non-negative} and~\cref{assn:nus} with constants $H_0 \geq 0, H_1 \geq 0$ such that $f^* = \min f(\theta) \neq 0$, then, $h(\theta) = f(\theta) - f^*$ satisfies~\cref{assn:non-negative} with $h^* = \min h(\theta) = 0$ and~\cref{assn:nus} with $L'_0 := H_0 + H_1 \, f^*$ and $L'_1 = H_1$.
\label{lemma:shift-nus}
\end{lemma}
\begin{proof}
By definition, since $f(\theta) \geq f^*$, $h$ satisfies~\cref{assn:non-negative} and $h^* = 0$. Since $f$ is $(H_0, H_1)$ $\NS$, 
\begin{align*}
\norm{\nabla^2 f(\theta)} & \leq H_0 + H_1 \, f(\theta) \\
\implies \norm{\nabla^2 h(\theta)} & = \norm{\nabla^2 f(\theta)} \leq H_0 + H_1 \, f(\theta) \\
& = \underbrace{H_0 + H_1 \, f^*}_{:= L'_0} + \underbrace{H_1}_{:= L'_1} \, [f(\theta) - f^*] \\
\implies \norm{\nabla^2 h(\theta)} & \leq L'_0 + L'_1 \, h(\theta)
\end{align*}
\end{proof}


\subsection{Definitions \& Helper Lemmas}

\begin{definition}[Clarke subdifferential  \cite{clarke1998nonsmooth}]
For a nonconvex locally Lipschitz function $f:\R^n\to\R$, the Clarke subdifferential is defined as 
\[
\partial f(x):=\mathrm{cl}\,\mathrm{co}\left\{\lim_{x_k\to x} \nabla f(x_k) \right\}
\]
where $\mathrm{cl}\,\mathrm{co}$ is the closed convex hull of the set, and $x_k$ forms a trajectory of points for which $\nabla f(x_k)$ exists. If additionally $f$ is convex, then $\partial f(x)$ is the standard convex subdifferential.
\end{definition}

\begin{lemma}[Chain rule for Clarke subdifferentials]
\label{lem:clarkesubdiffchainrule}
Suppose $h(x) = f(g(x))$, and both $f$ and $g$ are locally Lipschitz. Then, 
\begin{enumerate}
    \item if $f:\R\to\R$ smooth and $g:\R^n\to \R$ convex but not smooth, then 
    \[
    \partial h(x) = f'(g(x)) \cdot \partial g(x)
    \]
    \item If $f:\R^n\to \R$ nonsmooth but continuous,  and $g:\R^p\to\R^n$ smooth
    then
    \[
    \partial h(x) = J(x)^Ty, \qquad u\in \partial f(u), \; u = g(x)
    \]
   where $J(x) = [\nabla g_1(x),...,\nabla g_n(x)]^T$ the Jacobian of the mapping $g$.
\end{enumerate}  
\end{lemma}
\begin{proof}
The first statement is given exactly in \cite{clarke1998nonsmooth}. Both are the result of the following statement, which is follows from the standard definition of lim sup:

    Consider $A(t) \to a$ and $\limsup_{t\to 0} B(t) = \mathcal B$ where $a$ is a point and $\mathcal B$ is a closed set.
    Then if $a\in \R$ or $\mathcal B\subset \R$, then
    \[
    \limsup_{t\to 0} A(t)\cdot B(t) \subset a\cdot \mathcal B = \{ab : b\in \mathcal B\}.
    \]
    
    If $h(x)$ is smooth at $x$, then 
   \[
   \partial h(x) = \{\nabla h(x)\} = \begin{cases}
   \{f'(g(x))\cdot \nabla g(x)\}, & g:\R^n\to\R\\
       \{J(x)^T \nabla f(g(x))\}, & g:\R^p\to \R^n
   \end{cases}
   \]

   If $h(x)$ is not smooth at $x$, then   in the first case, 
\[
\limsup_{z\to x} \nabla h(z) = \limsup_{z\to x}  f'(g(z))\cdot \nabla g(z) \subseteq  f'(g(x)) \cdot \partial g(x).
\]
Since subdifferentials are closed and convex, then 
\[
\partial h(x)  = \mathrm{co\,cl}\limsup_{z\to x} \nabla h(z) = f'(g(x)) \cdot \partial g(x).  
\]

By similar logic, in the second case, for $w = g(z)$, $u = g(x)$
  \[
\limsup_{z\to x} \nabla h(z) = \limsup_{z\to x}  J(z)^T\nabla f(w) \subset J(x)^T \limsup_{w\to u}   \nabla f(w) = J(x)^T\partial f(u)
\] 
which are also closed and convex sets. So, 
\[
\partial h(x)  = \{J(x)^Ty : y\in\partial f(u),\,u=g(x)\}.
\]
\end{proof}

\section{Normalized Steepest Descent}
\label{app:sd}

\begin{proposition}
If $\nabla_t := \nabla f(\tht)$ with $\nabla_{t,i}$ denoting coordinate $i$ of this vector, and we use $\text{sign}(\nabla_t) \in \{-1,0,1\}^{D}$ to denote the element-wise sign operation with $\text{sign}(0) := 0$ and $e_i$ denotes the \(i\)-th standard basis vector, then the $\NSD$ update in~\cref{eq:sd-update} recovers the following special cases:

\textbullet\ $(p,q) = (\infty,1)$, $\thtt = \tht - \etat \, \text{sign}(\nabla_{t})$ ($\SignGD$)

\textbullet\ $(p,q) = (2,2)$, $\thtt = \tht - \etat \, \frac{\nabla_t}{\norm{\nabla_t}_2}$ ($\NormGD$)

\textbullet\ $(p,q) = (1,\infty)$, $\thtt = \tht - \etat \, \text{sign}(\nabla_{t,i_t}) \, e_{i_t}$, where, $i_t \in \argmax_{i \in [D]}|\nabla_{t,i}|$ ($\SignCD$)

\label{prop:nsd-derivation}
\end{proposition}
\begin{proof}

\textbullet\ For $(p,q) = (\infty,1)$, for $d$ s.t. $\norm{d}_\infty \leq 1$,
\begin{align*}
\langle \nabla_t, d \rangle  \leq \norm{\nabla_t}_1 \, \norm{d}_\infty \leq \norm{\nabla_t}_1 \tag{By Holder's inequality, and since $\norm{d}_\infty \leq 1$}    
\end{align*}
For $d^* = \text{sign}(\nabla_t)$ where $\text{sign}(0) = 0$, $\norm{d^*}_\infty = 1$ and 
\begin{align*}
\langle \nabla_t, d^* \rangle &= \sum_{i \in [D], \nabla_{t,i} \neq 0} \nabla_{t,i} \, d^*_i + \sum_{i \in [D], \nabla_{t,i} = 0} \nabla_{t,i} \, d^*_i = \sum_{i \in [D], \nabla_{t,i} \neq 0} \nabla_{t,i} \, \text{sign}(\nabla_{t,i}) = \sum_{i \in [D], \nabla_{t,i} \neq 0} |\nabla_{t,i}| = \norm{\nabla_t}_1
\end{align*}
Hence, $d^*$ attains the upper-bound, and is therefore optimal. Hence, the $\NSD$ update simplifies to $\thtt = \tht - \etat \, \text{sign}(\nabla_{t})$. 

\textbullet\ For $(p,q) = (2,2)$, for $d$ s.t. $\norm{d}_2 \leq 1$,
\begin{align*}
\langle \nabla_t, d \rangle  \leq \norm{\nabla_t}_2 \, \norm{d}_2 \leq \norm{\nabla_t}_2 \tag{By Cauchy Schwarz and since $\norm{d}_2 \leq 1$}
\end{align*}
For $d^* := \frac{\nabla_t}{\norm{\nabla_t}_2}$, $\norm{d^*}_2 = 1$, and $\langle \nabla_t, d^* \rangle = \norm{\nabla_t}_2$. Hence, $d^*$ attains the upper-bound, and is therefore optimal. Hence, the $\NSD$ update simplifies to $\thtt = \tht - \etat \, \frac{\nabla_t}{\norm{\nabla_t}_2}$

\textbullet\ For $(p,q) = (1,\infty)$, for $d$ s.t. $\norm{d}_1 \leq 1$,
\begin{align*}
\langle \nabla_t, d \rangle  \leq \norm{\nabla_t}_\infty \, \norm{d}_1 \leq \norm{\nabla_t}_\infty \tag{By Holder's inequality, and since $\norm{d}_1 \leq 1$}    
\end{align*}
Pick a coordinate $i_t \in \argmax_{i \in [D]}|\nabla_{t,i}|$ and define $d^* = \text{sign}(\nabla_{t,i_t}) \, e_{i_t}$. For $d^*$, $\norm{d^*}_{1} = 1$ and
\begin{align*}
\langle \nabla_t, d^* \rangle & = \sum_i \nabla_{t,i} \, d^*_i = \nabla_{t,i_t} \, d^*_{i_t} = \nabla_{t,i_t} \, \text{sign}(\nabla_{t,i_t}) = |\nabla_{t, i_t}| = \max_{i \in [D]}  |\nabla_{t, i}| = \norm{\nabla_t}_\infty
\end{align*}
Hence, $d^*$ attains the upper-bound, and is therefore optimal. Hence, the $\NSD$ update simplifies to $\thtt = \tht - \etat \, \text{sign}(\nabla_{t,i_t}) \, e_{i_t}$, where, $i_t \in \argmax_{i \in [D]}|\nabla_{t,i}|$.  
\end{proof}

\nsd*
\begin{proof}
Using~\cref{lemma:sd-descent} for $\eta \leq \frac{1}{\sqrt{H_1}}$, 
\begin{align}
f(\thtt) & \leq f(\tht) - \eta \,  \norm{\nabla f(\tht)}_{q} +\left(H_0 + H_1 \, f(\tht) \right) \, \eta^2  \nonumber \\
& \leq f(\tht) - \eta \,  \mu \, [f(\tht)]^{\tau} +\left(H_0 + H_1 \, f(\tht) \right) \, \eta^2  \tag{Using~\cref{assn:pl} with $f^* = 0$ and $\mu(\theta) = \mu$} \\
\implies f(\thtt) & \leq f(\tht) - \eta \,  \mu \, [f(\tht)]^{\tau} +\left(H_0 + H_1 \, f(\tht) \right) \, \eta^2 \label{eq:sd-2pl-inter-1}
\end{align}
Define $T_0$ to be the first iteration s.t. $f(\theta_{T_0}) < \max\left\{\epsilon, \frac{H_0}{H_1} \right\}$.

\textbf{Phase 1}: We first analyze~\cref{eq:sd-2pl-inter-1} for $t < T_0$ where $f(\theta_t) \geq \max\left\{\epsilon, \frac{H_0}{H_1} \right\}$. Simplifying~\cref{eq:sd-2pl-inter-1} in this case, 
\begin{align*}
f(\thtt) & \leq f(\tht) - \eta \,  \mu \, [f(\tht)]^{\tau} + 2 \, H_1 \, f(\tht) \, \eta^2 
\end{align*}
We will now use an inductive proof and prove that for all $t \leq T_0$, $f(\tht) \leq f(\theta_1)$. 

\textbf{Base Case}: For $t = 1$, this is true by definition. 

\textbf{Inductive Hypothesis}: Assume for iteration $t < T_0$, $f(\tht) \leq f(\theta_1)$.

\textbf{Induction}: We will prove that $f(\thtt) \leq f(\theta_1)$. For this, note that $f(\tht) \leq f(\theta_1)$ by the inductive hypothesis and $1 - \tau \geq 0$. Hence, the above inequality can simplified as:
\begin{align*}
f(\thtt) & \leq f(\tht) - \eta \,  \mu \, \frac{f(\tht)}{[f(\theta_1)]^{1 - \tau}} + 2 H_1 \, f(\tht) \, \eta^2 
\end{align*}
Using~\cref{lemma:sd-p1} with $A = \frac{\mu}{[f(\theta_1)]^{1 - \tau}}$, $B = 2 H_1$, $\bar{\eta}_1 = \frac{1}{\sqrt{H_1}}$ and $\delta = \max\left\{\epsilon, \frac{H_0}{H_1} \right\}$, we get that with $\eta = \min\left\{\bar{\eta}_1, \frac{\mu}{4 \, H_1 \, [f(\theta_1)]^{1 - \tau}} \right\}$, $f(\thtt) \leq f(\tht)$. This completes the induction. Moreover, $f(\theta_{T_0}) \leq \max\left\{\epsilon, \frac{H_0}{H_1} \right\}$ after 
\[
T_0 \geq O \left(\frac{[f(\theta_1)]^{1 - \tau}}{\mu \, \eta} \ln\left(\frac{f(\theta_1)}{\delta}\right) \right) = O \left(\ln\left(\min\left\{\frac{1}{\epsilon}, \frac{H_1}{H_0} \right\} \right) \right)\, \text{iterations.}
\]
Hence, if $\epsilon \geq \frac{H_0}{H_1}$, $\NSD$ with a constant step-size requires $O(\ln(1/\epsilon))$ iterations to guarantee convergence to an $\epsilon$ sub-optimality. 

\textbf{Phase 2}: Consider iteration $t > T_0$ such that $f(\tht) < \frac{H_0}{H_1}$. Simplifying~\cref{eq:sd-2pl-inter-1} in this case, 
\begin{align*}
f(\thtt) & \leq f(\tht) - \eta \,  \mu \, [f(\tht)]^{\tau} + 2 H_0 \, \eta^2 
\end{align*}
Using~\cref{lemma:sd-p2} with $A = \mu$, $B = 2 \, H_0$, $r = \tau$, $\bar{\eta}_1 = \frac{1}{\sqrt{H_1}}$ and $\delta = \epsilon$, we get that with $\eta = \min\left\{\frac{1}{\sqrt{H_1}}, \frac{\mu \, \epsilon^{\tau}}{4 H_0}\right\}$, $f(\thtt) \leq f(\tht) \leq f(\theta_{T_0}) \leq \frac{H_0}{H_1}$ for all $t \geq T_0$. 

Furthermore, if $\tau = 1$, $f(\theta_{T_\epsilon + T_0}) \leq \epsilon$ after 
\[
T_\epsilon \geq \frac{2}{\mu \, \eta} \ln\left(\frac{f(\theta_1)}{\delta}\right)= O \left( \frac{1}{\epsilon} \right) \text{iterations}. 
\] 
Else, if $\tau < 1$, $f(\theta_{T_\epsilon + T_0}) \leq \epsilon$ after 
\[
T_\epsilon \geq \frac{2 \, [f(\theta_{T_0})]^{1 - \tau}}{\mu \, \eta \, (1 - \tau)} = O \left( \frac{1}{\epsilon^{\tau}} \right) \text{iterations}. 
\]
Hence, for $T := T + T_\epsilon$, $f(\theta_{T+1}) \leq \epsilon$ after, 
\[
O \left( \frac{1}{\epsilon^{\tau}} + \ln \left( \frac{H_1}{H_0} \right) \right) \, \text{iterations.}
\]
\end{proof}

\bandits*
\begin{proof}
From~\cref{prop:bandits}, we know that the multi-armed bandit problem  satisfies~\cref{assn:nus} with $H_0 = 0$ and~\cref{assn:pl} with $\tau = 1$ and $f^* = 0$, though with a non-uniform $\mu(\theta)$. In order to use~\cref{thm:nsd-pl}, we need to lower-bound $\mu(\theta) = \pi_\theta(a^*)$ by a constant $\mu$. For this, we use a uniform initialization which guarantees that $\pi_{\theta_1}(a^*) = \frac{1}{K}$. 

We then use~\citet[Proposition 2]{mei2020global} which implies that if $f(\thtt) \leq f(\tht)$, $\pi_{\thtt}(a^*) \geq \pi_{\tht}(a^*)$. For $\NSD$ on a function with a fixed $\mu$,~\cref{thm:nsd-pl} guarantees descent for all $t$ in both phases. Consequently, we can inductively conclude that for all $t$, $\pi_{\thtt}(a^*) \geq \pi_{\tht}(a^*) \geq \pi_{\theta_1}(a^*) \geq \mu := \frac{1}{K}$. 

Hence,~\cref{thm:nsd-pl} immediately recovers a linear convergence rate with $\mu = \frac{1}{K}$. 
\end{proof}

\nsdlog*
\begin{proof}
Since we are considering logistic regression on linearly separable data, $f^* = 0$ and $H_0 = 0$. We will use a similar proof as that for~\cref{thm:nsd-pl}. In particular, using~\cref{lemma:sd-descent} for $\eta \leq \bar{\eta}_0 := \frac{1}{\sqrt{H_1}}$,
\begin{align}
f(\thtt) & \leq f(\tht) - \eta \, \norm{\nabla f(\tht)}_{q} +  H_1 \, f(\tht)  \, \eta^2
\label{eq:sd-logreg-inter-1}    
\end{align}
We will use an inductive argument and prove that for all $t \geq 0$, $f(\tht) \leq f(\theta_1)$. 

\textbf{Base Case}: For $t = 1$, this is true by definition. 

\textbf{Inductive Hypothesis}: Assume for iteration $t > 0$, $f(\tht) \leq f(\theta_1)$.

\textbf{Induction}: We will prove that $f(\thtt) \leq f(\theta_1)$. To complete the induction, we will consider two phases. For this, define $T_0$ to be the first iteration s.t. $f(\theta_{T_0}) < \max\left\{\epsilon, \frac{\ln(2)}{n}\right\}$. 

\textbf{Phase 1:} For all $t < T_0$, $f(\tht) > \frac{\ln(2)}{n}$ and we use part (2) of~\cref{prop:logistic} to conclude that $\norm{\nabla_t}_q \geq \frac{\gamma_p}{3n}$. Simplifying~\cref{eq:sd-logreg-inter-1},
\begin{align*}
f(\thtt) & \leq f(\tht) - \eta \, \norm{\nabla f(\tht)}_{q} +  H_1 \, f(\theta_1)  \, \eta^2 \tag{Using the inductive hypothesis} \\
& \leq f(\tht) - \frac{\eta \, \gamma_p}{4n} + H_1 \, f(\theta_1) \, \, \eta^2 \tag{Using that for $t < T_0$, $\norm{\nabla_t}_q \geq \frac{\gamma_p}{3n} \geq \frac{\gamma_p}{4n} $} \\
& \leq f(\tht) - \frac{\eta \, \gamma_p}{8n} \tag{Setting $\eta \leq \bar{\eta}_1 := \frac{\gamma_p}{8 \, H_1 \, f(\theta_1) \, n}$} \\
\intertext{Since $f(\thtt) \leq f(\tht) \leq f(\theta_1)$, this completes the induction for Phase 1. By recursing for $T_0-1$ iterations,}
\implies f(\theta_{T_0}) & \leq f(\theta_1) -  \frac{\eta \, \gamma_p}{8n} \, (T_0-1) = \ln(2) -  \frac{\eta \, \gamma_p}{8n} \, (T_0-1) \tag{Since $\theta_1 = 0$}
\end{align*}
Hence, by using that $\eta \leq \min\{\bar{\eta}_0, \bar{\eta}_1\}$,
\[
T_0 \geq O \left(\frac{n \, \ln(2)}{\eta \, \gamma_p} \right) = O \left( \left(\frac{n}{\gamma_p}\right)^2 \, H_1 \right) \text{iterations}
\]
guarantee that $f(\theta_{T_0}) \leq \frac{\ln(2)}{n}$. 

\textbf{Phase 2:} After $T_0$ iterations, $f(\theta_{T_0}) \leq \frac{\ln(2)}{n} < f(\theta_1)$. We will now prove that for $t > T_0$, $f(\tht) \leq f(\theta_{T_0}) < f(\theta_1)$, thus completing the induction in Phase 2. We will again do this proof via induction. 

\textbf{Base Case}: For $t = T_0$, this is true by definition. 

\textbf{Inductive Hypothesis}: Assume for iteration $t > T_0$, $f(\tht) \leq f(\theta_{T_0})$.

\textbf{Induction}: We will prove that $f(\thtt) \leq f(\theta_{T_0})$. Since $f(\tht) \leq f(\theta_{T_0}) < f(\theta_1)$ by the inductive hypothesis, combining the above relation with~\cref{eq:sd-logreg-inter-1},
\begin{align*}
f(\thtt) & \leq f(\tht) - \eta \, \norm{\nabla_t}_q + H_1 \, f(\tht) \,\eta^2 \\
& \leq f(\tht) - \frac{\eta \, \gamma_p}{2} f(\tht) + H_1 \, f(\tht) \, \eta^2 \tag{Using Part (1) of~\cref{prop:logistic}} \\
& \leq f(\tht) - \frac{\eta \, \gamma_p}{4} f(\tht) \tag{Setting $\eta \leq \bar{\eta}_2 := \frac{\gamma_p}{4 \, H_1} $} \\
\intertext{Since $f(\thtt) < f(\tht) < f(\theta_{T_0})$, this completes the induction. Furthermore, by recursing for $T$ iterations}
\implies f(\theta_{T_0 + T}) & \leq f(\theta_{T_0}) \, \exp\left(-\frac{\eta \, \gamma_p}{4} \, T \right) \leq \frac{\ln(2)}{n} \, \, \exp\left(-\frac{\eta \, \gamma_p}{4} \, T \right) 
\end{align*}
Hence, by using a step-size $\eta \leq \min\{\bar{\eta}_0, \bar{\eta}_2\}$, 
\[
T \geq O \left( \frac{1}{\eta \, \gamma_p} \, \ln\left(\frac{\ln(2)}{n \, \epsilon} \right) \right) =  O \left( \frac{H_1}{\gamma_p^2} \, \ln\left(\frac{1}{n \epsilon}\right) \right)
\]
iterations suffice to guarantee that $f(\theta_{T + T_0}) \leq \epsilon$. Hence, attaining an $\epsilon$ sub-optimality for logistic regression requires, 
\[
T_{\text{final}} = O \left( \frac{H_1}{\gamma_p^2} \, \ln\left(\frac{1}{n \epsilon}\right) \right) + O \left( \left(\frac{n}{\gamma_p}\right)^2 \, H_1 \right) = O\left(\frac{H_1}{\gamma_p^2} \, \left[n^2 + \ln\left(\frac{1}{n \epsilon}\right)\right]\right) \text{iterations.}
\]
\end{proof}

\begin{proposition}
Consider minimizing a one-dimensional quadratic $f(\theta) = \frac{\theta^2}{2}$ using $\NSD$ with a constant step-size $\eta$. For a fixed $\eta$ and $\epsilon \in \left(0, \frac{1}{8} \right)$, there exists an initialization $\theta_1$ such that the hitting time $T_\epsilon := \inf\{t \geq 1: f(x_t) \leq \epsilon\}$ is bounded as $T_\epsilon = \Omega\left(\frac{1}{\sqrt{\epsilon}}\right)$. 
\label{prop:nsd-lb}
\end{proposition}
\begin{proof}
We first prove for a fixed $\eta > \sqrt{8 \epsilon}$, we can find an initialization $\theta_1$ such that the $\NSD$ iterates oscillate around $\theta^* = 0$ and that $f(\tht) > \epsilon$ for all $t \geq 1$. 

For this, consider the $\NSD$ update at $t = 1$ with $\theta_1 = \frac{\eta}{2}$. In this case, 
\begin{align*}
d_1 &= \argmax_{|d| \leq 1} d \, f'(\theta_1) =  \argmax_{|d| \leq 1} d \, \theta_1 = 1 
\implies \theta_2 = \theta_1 - \eta \, d_1 = \frac{\eta}{2} - \eta = -\frac{\eta}{2}     
\end{align*}
Similarly, for $t = 2$, $d_2 = -1 \implies \theta_3 = \frac{\eta}{2}$. Hence, the iterates oscillate between $\eta/2$ and $-\eta/2$. For all $t \geq 1$, $f(\tht) =  \frac{\tht^2}{2} = \frac{\eta^2}{8} > \epsilon$.  

Hence, ensuring $\epsilon$ convergence from any initialization requires $\eta \leq \sqrt{8 \epsilon}$. In this case, choose the initialization $\theta_1 = 1$ and define $T_\text{sign} := \inf_{t \geq 1} \tht < 0$. Using the $\NSD$ update, for all $t < T_\text{sign}$, $\tht > 0$ and, 
\begin{align*}
\thtt & = \tht - \eta \implies \theta_{T_\text{sign}} = 1 - \eta \, (T_\text{sign} - 1)  \implies T_\text{sign} > 1 + \frac{1}{\eta} \geq 1 + \frac{1}{\sqrt{8 \epsilon}}
\end{align*}

\textbf{Case 1}: If $T_\epsilon \geq T_{\rm sign}$, then
$T_\epsilon \geq T_{\rm sign} > 1+\frac{1}{\eta}
\geq 1+\frac{1}{\sqrt{8\epsilon}}$, which gives the desired lower bound.

\textbf{Case 2}: If $T_\epsilon < T_{\rm sign}$, then for all $t < T_\epsilon < T_\text{sign}$, $f(\tht) = \frac{(1 - \eta \, (t-1))^2}{2}$. By definition of $T_\epsilon$, 
\begin{align*}
\frac{(1 - \eta \, (T_\epsilon-1))^2}{2} \leq \epsilon \implies T_\epsilon \geq 1 + \frac{1 - \sqrt{2 \epsilon}}{\eta} \geq 1 + \frac{1 - \sqrt{2 \epsilon}}{\sqrt{8 \epsilon}}
\end{align*}
Hence, the hitting time $T_\epsilon = \Omega\left(\frac{1}{\sqrt{\epsilon}}\right)$. 
\end{proof}


\subsection{Helper Lemmas}
\label{app:sd-helper-lemmas}

\begin{lemma}
For constants $A, B > 0$, if for all iterations $t \geq 1$, 
\begin{align*}
f(\thtt) & \leq f(\tht) - \eta \, A \, f(\tht) + \eta^2 \, B \, f(\tht)  \,,
\end{align*}
and requires $\eta \leq \bar{\eta}_1$, then, with $\eta = \min \{\bar{\eta}_1, \frac{A}{2 \, B} \}$, after $T = \frac{2 \, \max\left\{\frac{1}{\bar{\eta}_1}, \frac{2 \, B}{A}\right\}}{A} \, \ln \left( \frac{f(\theta_1)}{\delta} \right)$ iterations, 
\begin{align}
f(\theta_{T+1}) \leq \delta \, \text{ and } \forall t \in [T] \,, f(\thtt) \leq f(\tht)
\end{align}
\label{lemma:sd-p1}
\end{lemma}
\begin{proof}
Setting $\eta \leq \frac{A}{2 \, B}$, 
\begin{align*}
f(\thtt) & \leq f(\tht) - \frac{A \, \eta}{2} \, f(\tht) \implies f(\thtt) \leq f(\tht)\left(1-\frac{A\eta}{2}\right)\\
\implies f(\theta_{T+1}) & \leq f(\theta_1) \, \exp\left(-T \, \frac{A \, \eta}{2} \right) \tag{By recursing from $t = 1$ to $T$}
\end{align*}
Hence, in order to guarantee that $f(\theta_{T+1}) \leq \delta$, it is sufficient to set, 
\begin{align*}
T & \geq \frac{2}{A \, \eta} \, \ln \left( \frac{f(\theta_1)}{\delta} \right) = \frac{2 \, \max\left\{\frac{1}{\bar{\eta}_1}, \frac{2 \, B}{A}\right\}}{A} \, \ln \left( \frac{f(\theta_1)}{\delta} \right) = O\left(\ln\left(\frac{1}{\delta} \right)\right)
\end{align*}
\end{proof}

\begin{lemma}
For constants $A, B > 0$, and $r > 0$,  if for all iterations $t \geq 1$, 
\begin{align*}
f(\thtt) & \leq f(\tht) - \eta \, A \, [f(\tht)]^{r} + \eta^2 \, B 
\end{align*}
and requires $\eta \leq \bar{\eta}_1$, then, with $\eta = \min\{\bar{\eta}_1, \frac{A \, \delta^r}{2B} \}$

\textbf{Case 1:} If $r = 1$, after $T \geq \frac{2 \, \max\left\{\frac{1}{\bar{\eta}_1}, \frac{2 \, B}{A \delta}\right\}}{A} \, \ln\left(\frac{f(\theta_1)}{\delta}\right) + 1$ iterations, 

\textbf{Case 2:} If $r < 1$, after $T \geq \frac{2 \, [f(\theta_1)]^{1-r} \, \max\left\{\frac{1}{\bar{\eta}_1}, \frac{2 \, B}{A \delta^r}\right\}}{A (1-r)} + 1$ iterations, 

\textbf{Case 3:} If $r > 1$, after $T \geq \frac{2 \, \max\left\{\frac{1}{\bar{\eta}_1}, \frac{2 \, B}{A \delta^r}\right\}}{A (r-1) \, \delta^{r-1}} + 1$ iterations,

\begin{align}
f(\theta_{T+1}) \leq \delta \, \text{ and } \forall t \in [T] \,, f(\thtt) \leq f(\tht)
\end{align}
\label{lemma:sd-p2}
\end{lemma}
\begin{proof}
For the desired sub-optimality $\delta$, setting $\eta = \min\left\{\bar{\eta}_1, \frac{A \, \delta^{r}}{2B} \right\}$,
\begin{align*}
f(\thtt) & \leq f(\tht) - \eta \, A \, [f(\tht)]^r + \frac{A \, \eta}{2} \, \delta^r    
\end{align*}
Let $T_\delta$ be the first iteration s.t. $f(\tht) \leq \delta$. Hence, for all $t < T_\delta$, $f(\tht) \geq \delta \implies [f(\tht)]^{r} \geq \delta^r$. Hence, for $t < T_\delta$, 
\begin{align}
f(\thtt) & \leq f(\tht) - \frac{A \, \eta}{2} \,  [f(\tht)]^r \implies f(\thtt) \leq f(\tht)\left(1-\frac{A\eta}{2}\right)
\label{eq:lemma-p2-sd-inter-1}
\end{align}
\textbf{Case 1}: If $r = 1$, 
\begin{align*}
f(\thtt) & \leq f(\tht) \left[1 - \frac{A \eta}{2} \right] \\
\implies f(\theta_{T_\delta}) & \leq f(\theta_1) \, \exp\left(-(T_\delta-1) \, \frac{A \eta}{2}\right)
\end{align*}
For $f(\theta_{T_\delta}) \leq \delta$, it is sufficient to set $T_\delta$ s.t., 
\begin{align*}    
T_\delta \geq \frac{2}{A \eta} \, \ln\left(\frac{f(\theta_1)}{\delta}\right) + 1 = \frac{2 \, \max\left\{\frac{1}{\bar{\eta}_1}, \frac{2 \, B}{A \delta}\right\}}{A} \, \ln\left(\frac{f(\theta_1)}{\delta}\right) + 1 = O\left(\frac{1}{\delta} \, \ln\left(\frac{1}{\delta} \right)\right)
\end{align*}

\textbf{Case 2:} If $r < 1$, continuing from~\cref{eq:lemma-p2-sd-inter-1}, note that $g(x) = x^{1-r}$ is concave in $x$ if $r < 1$. Furthermore, $g'(x) = (1-r) \, x^{-r}$. Using concavity, for any $y, x$, 
\begin{align*}
g(y) & \leq g(x) + g'(x) (y - x) \implies y^{1 - r} \leq x^{1 - r} +  (1-r) \, x^{-r} \, (y - x) \\
\implies [f(\thtt)]^{1 - r} & \leq [f(\tht)]^{1-r} + (1-r) \, [f(\tht)]^{-r} \, (f(\thtt) - f(\tht)) \\
& \leq [f(\tht)]^{1-r} - (1-r) \, [f(\tht)]^{-r} \,  \frac{A \, \eta}{2} \,  [f(\tht)]^r \tag{Using~\cref{eq:lemma-p2-sd-inter-1}} \\
& = [f(\tht)]^{1-r} - (1-r) \, \frac{A \, \eta}{2} \\
\implies [f(\theta_{T_\delta})]^{1 - r} & \leq [f(\theta_1)]^{1-r} - (1-r) \, \frac{A \, \eta}{2} \, (T_\delta -1)
\end{align*}
For $f(\theta_{T_\delta}) \leq \delta$, it is sufficient to set $T_\delta$ s.t., 
\begin{align*}
[f(\theta_1)]^{1-r} - (1-r) \, \frac{A \, \eta}{2} \, (T_\delta -1) & \leq \delta^{1-r}  \\
\end{align*}
Hence, it is sufficient to set $T_\delta$ s.t. 
\begin{align*}
T_{\delta} & \geq \frac{2 \, [f(\theta_1)]^{1-r}}{A \, \eta \, (1-r)} + 1 = \frac{2 \, [f(\theta_1)]^{1-r} \, \max\left\{\frac{1}{\bar{\eta}_1}, \frac{2 \, B}{A \delta^r}\right\}}{A (1-r)}  + 1 = O \left( \frac{1}{\delta^r} \right)
\end{align*}

\textbf{Case 3:} If $r > 1$, continuing from~\cref{eq:lemma-p2-sd-inter-1}, note that $g(x) = x^{1-r}$ is convex in $x$ if $r > 1$. Furthermore, $g'(x) = (1-r) \, x^{-r}$. Using convexity, for any $y, x$, 
\begin{align*}
g(y) & \geq g(x) + g'(x) (y - x) \implies y^{1 - r} \geq x^{1 - r} +  (1-r) \, x^{-r} \, (y - x) \\
\implies [f(\thtt)]^{1 - r} & \geq [f(\tht)]^{1-r} + (1-r) \, [f(\tht)]^{-r} \, (f(\thtt) - f(\tht)) \\
& \geq [f(\tht)]^{1-r} - (1-r) \, [f(\tht)]^{-r} \,  \frac{A \, \eta}{2} \,  [f(\tht)]^r \tag{Using~\cref{eq:lemma-p2-sd-inter-1}} \\
& = [f(\tht)]^{1-r} - (1-r) \, \frac{A \, \eta}{2} \\
\implies [f(\theta_{T_\delta})]^{1 - r} & \geq [f(\theta_1)]^{1-r} - (1-r) \, \frac{A \, \eta}{2} \, (T_\delta -1) =  [f(\theta_1)]^{1-r} + (r-1) \, \frac{A \, \eta}{2} \, (T_\delta-1)\\
\implies [f(\theta_{T_\delta})]^{r - 1} & \leq \frac{1}{  [f(\theta_1)]^{1-r} + (r-1) \, \frac{A \, \eta}{2} \, (T_\delta-1)}
\end{align*}
For $f(\theta_{T_\delta}) \leq \delta$, it is sufficient to set $T_\delta$ s.t., 
\begin{align*}
\frac{1}{  [f(\theta_1)]^{1-r} + (r-1) \, \frac{A \, \eta}{2} \, (T_\delta-1)} \leq \delta^{r-1} \implies
[f(\theta_1)]^{1-r} + (r-1) \, \frac{A \, \eta}{2} \, (T_\delta-1)  & \geq  \delta^{1-r}  \\
\end{align*}
Hence, it is sufficient to set $T_\delta$ s.t. 
\begin{align*}
T_{\delta} & \geq \frac{2\, \delta^{1-r}}{A \, \eta \, (r-1) \, } + 1 = \frac{2 \, \max\left\{\frac{1}{\bar{\eta}_1}, \frac{2 \, B}{A \delta^r}\right\}\, \delta^{1-r}}{A (r-1)  } + 1  = O \left( \frac{1}{\delta^{2 \, r - 1}} \right)
\end{align*}

\end{proof}

\begin{lemma}
Under~\cref{assn:non-negative},~\cref{assn:nus} with $H_0 \geq 0$ and $H_1 \geq 0$, for the normalized steepest descent update in~\cref{eq:sd-update}, with $\eta \leq \frac{1}{\sqrt{H_1}}$, 
\begin{align*}
f(\thtt) & \leq f(\tht) - \eta \, \norm{\nabla f(\tht)}_{q} +  \left(H_0 + H_1 \, f(\tht) \right) \, \eta^2
\end{align*}   
\label{lemma:sd-descent}
\end{lemma}
\begin{proof}
If $\norm{\thtt - \tht}_p \leq \frac{1}{\sqrt{H_1}}$, using~\cref{lemma:descent-ineq-gen} and denoting $\nabla_t := \nabla f(\tht)$ for convenience,
\begin{align*}
f(\thtt) & \leq f(\tht) - \eta \,  \langle \nabla f(\tht), d_t \rangle + \left(H_0 + H_1 \, f(\tht) \right) \, \eta^2 \, \normsq{d_t}_{p} \\
& \leq f(\tht) - \eta \,  \langle \nabla f(\tht), d_t \rangle + \left(H_0 + H_1 \, f(\tht) \right) \, \eta^2  \tag{By definition $\norm{d_t}_p \leq 1$} \\
\implies f(\thtt) & \leq f(\tht) - \eta \,  \norm{\nabla f(\tht)}_{q} +\left(H_0 + H_1 \, f(\tht) \right) \, \eta^2  \tag{By definition of $d_t$ and since $\norm{z}_q := \max_{\norm{u}_p \leq 1} \langle z, u \rangle$} 
\end{align*}

\end{proof}
\section{Convergence of RMSProp}
\label{app:rms}
\rms*

\begin{proof}
Using~\cref{lemma:rms-descent} for $\eta \leq \frac{\sqrt{1-\beta}}{\sqrt{H_1}}$, $\norm{d_t}_\infty \leq \frac{1}{\sqrt{1 - \beta}}$ and,  
\begin{align}
f(\thtt) & \leq f(\tht) - \eta \, \langle \nabla_t, d_t \rangle +  \left(H_0 + H_1 \, f(\tht) \right) \, \frac{\eta^2}{1 - \beta} 
\label{eq:rms-2pl-inter-1}    
\end{align}
Using~\cref{lemma:rms-precond-bound} to simplify the second term on the RHS of~\cref{eq:rms-2pl-inter-1}, 
\begin{align}
\langle \nabla_t, d_t \rangle &= \sum_i [\nabla_t]_i \, [d_t]_i = \sum_i \frac{\ti{g^2}}{\ti{\sqrt{v}}} \geq \norm{\nabla_t}_1 \, \frac{\norm{\nabla_t}_1}{\sqrt{1 - \beta} \, \sum_{j = 0}^{t-1} (\sqrt{\beta})^{j} \, \norm{\nabla_{t-j}}_1} \label{eq:rms-2pl-secondterm} 
\intertext{Combining the above inequalities,}
f(\thtt) & \leq f(\tht) - \eta \,\norm{\nabla_t}_1 \, \frac{\norm{\nabla_t}_1}{\sqrt{1 - \beta} \, \sum_{j = 0}^{t-1} (\sqrt{\beta})^{j} \, \norm{\nabla_{t-j}}_1} +  \left(H_0 + H_1 \, f(\tht) \right) \, \frac{\eta^2}{1 - \beta} \label{eq:rms-2pl-inter-2}
\end{align}

\textbf{Phase 1:} Define $T_0$ to be the first iteration s.t. $f(\theta_{T_0}) < \max\left\{\epsilon, \frac{H_0}{H_1} \right\}$, and analyze the above inequality for $t < T_0$ where $f(\theta_t) \geq \max\left\{\epsilon, \frac{H_0}{H_1} \right\}$. Using~\cref{eq:rms-2pl-inter-2} in this case, 
\begin{align}
f(\thtt) & \leq f(\tht) - \eta \, \norm{\nabla_t}_1 \, \underbrace{\frac{\norm{\nabla_t}_1}{\sqrt{1 - \beta} \, \sum_{j = 0}^{t-1} (\sqrt{\beta})^{j} \, \norm{\nabla_{t-j}}_1}}_{(*)} + 2 \, H_1 \, f(\tht)  \, \frac{\eta^2}{1 - \beta}    \label{eq:rms-2pl-p1-1}
\end{align}
We will now use an inductive proof and prove that for all $t \leq T_0$, $f(\tht) \leq f(\theta_1)$. 

\textbf{Base Case}: For $t = 1$, this is true by definition. 

\textbf{Inductive Hypothesis}: Assume for iteration $t < T_0$, $f(\tht) \leq f(\theta_1)$.

\textbf{Induction}: We will prove that $f(\thtt) \leq f(\theta_1)$. 
First, for $\eta \leq \frac{\mu \, \sqrt{1 - \beta}}{ e} \, \left(\frac{H_0}{H_1}\right)^\tau \, \frac{1}{H_0 + H_1 \, f(\theta_1)} \, \ln(\nicefrac{1}{\beta^{1/4}})$ and all $t$, 
\[
\|\theta_{t+1}-\theta_{t}\|_\infty = |\eta_t|\,\underbrace{\max_i |\frac{g_{t,i}}{\sqrt{v_{t,i}}}|}_{\leq \frac{1}{\sqrt{1-\beta}}}\leq 
\underbrace{ \mu  \,\left(\frac{H_0}{H_1}\right)^\tau}_{c}   \, \frac{1}{e(H_0 + H_1 \, f(\theta_1))} \,\underbrace{ \ln(\nicefrac{1}{\beta^{1/4}})}_{a}.
\]
Since $f(\tht) \leq f(\theta_1)$ by the inductive hypothesis, we can use~\cref{lemma:rms-ratio-2pl} and obtain that, 

\begin{align}
(*) = \frac{\norm{\nabla_t}_1}{\sqrt{1 - \beta} \, \sum_{j = 0}^{t-1} (\sqrt{\beta})^{j} \, \norm{\nabla_{t-j}}_1} \geq \underbrace{\frac{1 - \beta^{1/4}}{\sqrt{1 - \beta}}}_{:= C} \, \frac{1}{1 + \left( \frac{H_0}{H_1 \, f(\tht)} \right)^{\tau} }  = \frac{C}{1 + \left( \frac{H_0}{H_1 \, f(\tht)} \right)^{\tau} }  \label{eq:rms-2pl-induction-1}
\end{align}
Combining~\cref{eq:rms-2pl-p1-1} and~\cref{eq:rms-2pl-induction-1} and noting that in Phase (1), $\frac{H_0}{H_1 \, f(\tht)} \leq 1 \implies (*) \geq \frac{C}{2}$,
\begin{align*}
f(\thtt) & \leq f(\tht) - \eta \, \norm{\nabla_t}_1 \, \frac{C}{2} + 2 \, H_1 \, f(\tht)  \, \frac{\eta^2}{1 - \beta} \\
& \leq f(\tht) - \eta \, [f(\tht)]^\tau \, \frac{C \, \mu}{2} + 2 \, H_1 \, f(\tht)  \, \frac{\eta^2}{1 - \beta} 
\tag{Using~\cref{assn:pl} with $f^* = 0$ and $\mu(\theta) = \mu$} \\
& \leq f(\tht) - \eta \, \frac{f(\tht)}{[f(\theta_1)]^{1-\tau}} \, \frac{C \, \mu}{2} + 2 \, H_1 \, f(\tht)  \, \frac{\eta^2}{1 - \beta} \tag{Since $f(\tht) \leq f(\theta_1)$ by the induction hypothesis}
\end{align*}
Using~\cref{lemma:sd-p1} with $A = \frac{C}{2} \, \frac{\mu}{[f(\theta_1)]^{1 - \tau}}$, $B = \frac{2 H_1}{1 - \beta}$, $\bar{\eta}_1 = \min\left\{\frac{\sqrt{1-\beta}}{\sqrt{H_1}}, \frac{\mu \, \sqrt{1 - \beta}}{2 e} \, \left(\frac{H_0}{H_1}\right)^\tau \, \frac{1}{H_0 + H_1 \, f(\theta_1)} \, \ln(\nicefrac{1}{\beta^{1/4})}\right\}$ and $\delta = \max\left\{\epsilon, \frac{H_0}{H_1} \right\}$, we get that with $\eta = \min\left\{\bar{\eta}_1, \frac{C \, \mu \, (1 - \beta)}{8 \, H_1 \, [f(\theta_1)]^{1 - \tau}} \right\}$, $f(\thtt) \leq f(\tht)$. This completes the induction. Moreover, $f(\theta_{T_0}) \leq \max\left\{\epsilon, \frac{H_0}{H_1} \right\}$ after 
\[
T_0 \geq O \left(\frac{[f(\theta_1)]^{1 - \tau}}{C \, \mu \, \eta} \ln\left(\frac{f(\theta_1)}{\delta}\right) \right) = O \left(\ln\left(\min\left\{\frac{1}{\epsilon}, \frac{H_1}{H_0} \right\} \right) \right)\, \text{iterations.}
\]
Hence, if $\epsilon \geq \frac{H_0}{H_1}$, $\RMS$ with a constant step-size requires $O(\ln(1/\epsilon))$ iterations to guarantee convergence to an $\epsilon$ sub-optimality.

\textbf{Phase 2}: Consider iteration $t > T_0$ such that $f(\tht) < \frac{H_0}{H_1}$. Simplifying~\cref{eq:rms-2pl-inter-2} in this case, 
\begin{align}
f(\thtt) & \leq f(\tht) - \eta \, \norm{\nabla_t}_1 \, \underbrace{\frac{\norm{\nabla_t}_1}{\sqrt{1 - \beta} \, \sum_{j = 0}^{t-1} (\sqrt{\beta})^{j} \, \norm{\nabla_{t-j}}_1}}_{(*)} + 2 \, H_0  \, \frac{\eta^2}{1 - \beta}    \label{eq:rms-2pl-p1-2}
\end{align}
We will now use an inductive proof and prove that for all $t \geq T_0$, $f(\tht) \leq f(\theta_{T_0}) \leq \frac{H_0}{H_1}$.

\textbf{Base Case}: For $t = T_0$, this is true by definition. 

\textbf{Induction}: We will prove that $f(\thtt) \leq f(\theta_{T_0})$. Since $f(\tht) \leq f(\theta_{T_0}) \leq \frac{H_0}{H_1} \leq f(\theta_1)$ by the inductive hypothesis, and ensuring that $\eta \leq \frac{\mu \, \sqrt{1 - \beta}}{2 e} \, \left(\frac{H_0}{H_1}\right)^{\tau} \, \frac{1}{H_0 + H_1 \, f(\theta_1)} \, \ln(\nicefrac{1}{\beta^{1/4}})$, we can use~\cref{lemma:rms-ratio-2pl} and obtain that, 
\begin{align}
(*) = \frac{\norm{\nabla_t}_1}{\sqrt{1 - \beta} \, \sum_{j = 0}^{t-1} (\sqrt{\beta})^{j} \, \norm{\nabla_{t-j}}_1} \geq \underbrace{\frac{1 - \beta^{1/4}}{\sqrt{1 - \beta}}}_{:= C} \, \frac{1}{1 + \left( \frac{H_0}{H_1 \, f(\tht)} \right)^{\tau} }  = \frac{C}{1 + \left( \frac{H_0}{H_1 \, f(\tht)} \right)^{\tau} }  \label{eq:rms-2pl-induction-2}
\end{align}
Combining~\cref{eq:rms-2pl-p1-2} and~\cref{eq:rms-2pl-induction-2} and noting that, $\frac{H_0}{H_1 \, f(\tht)} \geq 1 \implies (*) \geq \frac{C}{2 \, \left( \frac{H_0}{H_1 \, f(\tht)} \right)^{\tau}}$, 
\begin{align*}
f(\thtt) & \leq f(\tht) - \eta \, \norm{\nabla_t}_1 \, \frac{C}{2} \, \left(\frac{H_1}{H_0}\right)^\tau \, [f(\theta)]^\tau +  2 \, H_0  \, \frac{\eta^2}{1 - \beta} \\
& \leq f(\tht) - \eta \, \frac{C \, \mu}{2} \, \left( \frac{H_1}{H_0} \right)^\tau \, [f(\tht)]^{2 \, \tau} +  2 \, H_0   \, \frac{\eta^2}{1 - \beta} \tag{Using~\cref{assn:pl} with $f^* = 0$ and $\mu(\theta) = \mu$}
\end{align*}
Using~\cref{lemma:sd-p2} with $A = \frac{C \, \mu}{2} \, \left( \frac{H_1}{H_0} \right)^\tau$, $B = \frac{2 \, H_0}{1- \beta}$, $\bar{\eta}_1 = \min\left\{\frac{\sqrt{1-\beta}}{\sqrt{H_1}}, \frac{\mu \, \sqrt{1 - \beta}}{2 e} \, \left( \frac{H_1}{H_0} \right)^\tau \, \frac{1}{H_0 + H_1 \, f(\theta_1)} \, \ln(\nicefrac{1}{\beta^{1/4}})\right\}$, $r = 2\tau$,  $\delta = \epsilon$, we get that with $\eta = \min\left\{\bar{\eta}_1, \frac{C \, \, \mu \, (1-\beta) \epsilon^{2\tau}}{8  \, H_0} \, \left( \frac{H_1}{H_0} \right)^\tau \right\}$, $f(\thtt) \leq f(\tht) \leq f(\theta_{T_0}) \leq \frac{H_0}{H_1}$ for all $t \geq T_0$. This completes the induction.

Furthermore, for $\tau = \frac{1}{2}$, $f(\theta_{T_\epsilon + T_0}) \leq \epsilon$ after 
\[
T_\epsilon \geq O \left( \frac{1}{C \, \mu \, \eta} \, \sqrt{\frac{H_0}{H_1} }\ln\left(\frac{f(\theta_1)}{\delta}\right) \right) = \tilde{O} \left( \frac{1}{\epsilon} \right) \text{iterations}. 
\] 
Else, for $\tau < \frac{1}{2}$, $f(\theta_{T_\epsilon + T_0}) \leq \epsilon$ after 
\[
T_\epsilon \geq O \left(\frac{[f(\theta_{T_0})]^{1 - 2 \, \tau}}{C \, \mu \, \eta \, (1 - 2 \, \tau)} \, \left( \frac{H_0}{H_1} \right)^\tau \right) = O \left( \frac{1}{\epsilon^{2 \, \tau}} \right) \text{iterations}. 
\]
Else, for $\tau > \frac{1}{2}$, $f(\theta_{T_\epsilon + T_0}) \leq \epsilon$ after 
\[
T_\epsilon \geq O \left(\frac{1}{C \, \mu \, \eta \, (2 \, \tau - 1) \, \epsilon^{2 \tau - 1}} \, \left( \frac{H_0}{H_1} \right)^\tau \right) = O \left( \frac{1}{\epsilon^{4 \tau - 1}} \right) \text{iterations}. 
\]
Hence, for $T := T + T_\epsilon$, $f(\theta_{T+1}) \leq \epsilon$ after, 
\[
O \left( \frac{1}{\epsilon^{2 \tau}} + \ln \left( \frac{H_1}{H_0} \right) \right) \, \text{iterations, if $\tau \leq \frac{1}{2}$ ;}
\]
\[
O \left( \frac{1}{\epsilon^{4 \tau - 1}} + \ln \left( \frac{H_1}{H_0} \right) \right) \, \text{iterations, if $\tau > \frac{1}{2}$}
\]
\end{proof}

\subsection{Helper Lemmas}
\label{app:rms-helper-lemmas}

\begin{lemma}
Under~\cref{assn:non-negative},~\cref{assn:nus} with $H_0 \geq 0$ and $H_1 \geq 0$, for the coordinate-wise RMSProp update in~\cref{eq:rmsprop-update-coordinate}, with $\eta \leq \frac{\sqrt{1-\beta}}{\sqrt{H_1}}$, 
\begin{align*}
f(\thtt) & \leq f(\tht) - \eta \, \langle \nabla_t, d_t \rangle +  \left(H_0 + H_1 \, f(\tht) \right) \, \frac{\eta^2}{1 - \beta} \quad \text{;} \quad \norm{d_t}_\infty \leq \frac{1}{\sqrt{1 - \beta}}
\end{align*}   
\label{lemma:rms-descent}
\end{lemma}
\begin{proof}
If $\norm{\thtt - \tht}_\infty \leq \frac{1}{\sqrt{H_1}}$, using~\cref{lemma:descent-ineq-gen} with $p = \infty$ and $q = 1$, and denoting $\nabla_t := \nabla f(\tht)$ for convenience,
\begin{align*}
f(\thtt) & \leq f(\tht) - \eta \, \langle \nabla_t, d_t \rangle + \left(H_0 + H_1 \, f(\tht) \right) \, \eta^2 \,  \normsq{d_t}_\infty
\end{align*}
Simplifying the third term on the RHS, 
\begin{align*}
\norm{d_t}_\infty = \max_{i} \frac{\ti{g}}{\ti{\sqrt{v}}} \leq \frac{1}{\sqrt{1 - \beta}} \tag{Since $\ti{v} \geq (1-\beta) \, \ti{g^2}$} 
\end{align*}
Ensuring that $\eta \leq \frac{\sqrt{1-\beta}}{\sqrt{H_1}}$ guarantees that $\norm{\thtt - \tht}_\infty \leq \frac{1}{\sqrt{H_1}}$. Combining the above inequalities, 
\begin{align}
f(\thtt) & \leq f(\tht) - \eta \, \langle \nabla_t, d_t \rangle +  \left(H_0 + H_1 \, f(\tht) \right) \, \frac{\eta^2}{1 - \beta}
\end{align}    
\end{proof}

\begin{lemma}
For $\beta \in (0,1)$, if $\ti{g} = [\nabla f(\tht)]_{i}$ and $\ti{v} = (1-\beta) \, \sum_{s = 1}^t \beta^{t -s} \, \si{g^2}  = \beta \, \tpi{v} + (1-\beta) \, \ti{g^2}$ with $[v_0]_i = 0$, then, 
\begin{align*}
\sum_i \frac{\ti{g^2}}{\ti{\sqrt{v}}} & \geq \norm{\nabla_t}_1 \, \frac{\norm{\nabla_t}_1}{\sqrt{1 - \beta} \, \sum_{j = 0}^{t-1} (\sqrt{\beta})^{j} \, \norm{\nabla_{t-j}}_1}    
\end{align*}
\label{lemma:rms-precond-bound}    
\end{lemma}
\begin{proof}
\begin{align}
\norm{\nabla_t}_1 &= \sum_i |\ti{g}| = \sum_i \frac{|\ti{g}|}{\ti{v^{1/4}}} \, \ti{v^{1/4}} \leq \sqrt{\sum_i \frac{\ti{g^2}}{\ti{\sqrt{v}}} } \, \sqrt{\sum_i \ti{\sqrt{v}}} \tag{By Cauchy Schwarz} \nonumber \\
\implies \normsq{\nabla_t}_1 & \leq \left(\sum_i \frac{\ti{g^2}}{\ti{\sqrt{v}}}\right) \, \left(\sum_i \ti{\sqrt{v}}\right) \implies \left(\sum_i \frac{\ti{g^2}}{\ti{\sqrt{v}}}\right) \geq \frac{\normsq{\nabla_t}_1}{\underbrace{\sum_i \ti{\sqrt{v}} }_{:= u_t} } \label{eq:rms-1pl-secondterm-2} 
\end{align}
Simplifying $u_t$,
\begin{align}
u_t = \sum_i \ti{\sqrt{v}}  &  = \sum_i \sqrt{\beta \, \tpi{v} + (1-\beta) \, \ti{g^2}} \leq \sum_i \sqrt{\beta} \, \tpi{\sqrt{v}} + \sqrt{1-\beta} \, |\ti{g}| \tag{Since $\sqrt{a + b} \leq \sqrt{a} + \sqrt{b}$} \nonumber \\
& = \sqrt{\beta} \, \sum_{i}  \tpi{\sqrt{v}} + \sqrt{1-\beta} \, \norm{\nabla_t}_1 = \sqrt{\beta} \, u_{t-1} + \sqrt{1 - \beta} \, \norm{\nabla_t}_1 \\
\intertext{Recursing from $j = 0$ to $t-1$, and using that $v_{0,i} = 0$,}
\implies u_t &\leq \sqrt{1 - \beta} \, \sum_{j = 0}^{t-1} (\sqrt{\beta})^{j} \, \norm{\nabla_{t-j}}_1 \label{eq:rms-1pl-urecursion} \\
\implies \left(\sum_i \frac{\ti{g^2}}{\ti{\sqrt{v}}}\right) & \geq \norm{\nabla_t}_1 \, \frac{\norm{\nabla_t}_1}{\sqrt{1 - \beta} \, \sum_{j = 0}^{t-1} (\sqrt{\beta})^{j} \, \norm{\nabla_{t-j}}_1}\tag{Combining with~\cref{eq:rms-1pl-secondterm-2}}
\end{align}
\end{proof}
    
\begin{lemma}
Under~\cref{assn:non-negative},~\cref{assn:nus} with $H_0 \geq 0, H_1 > 0$, fix constants $c > 0$ and $a > 0$ such that $\sqrt{\beta} \exp(a) < 1$. 

If for all iterations $t \geq 1$, (i) $\norm{\tht - \thtt}_p \leq \frac{c}{e} \, \frac{a}{H_0 + H_1 \, f(\theta_1)}$ such that and (ii) $f(\tht) \leq f(\theta_1)$, then, 
\begin{align*}
\frac{\norm{\nabla_t}_1}{\sqrt{1 - \beta} \, \sum_{j = 0}^{t-1} (\sqrt{\beta})^{j} \, \norm{\nabla_{t-j}}_1} & \geq \frac{\norm{\nabla_t}_1}{\norm{\nabla_t}_1 + c} \, \frac{1 - \sqrt{\beta} \, \exp(a)}{\sqrt{1 - \beta}}     
\end{align*}

Furthermore, if $f$ satisfies~\cref{assn:pl} with $f^* = 0$ and $\mu(\theta) = \mu$ for all $\theta$, then, by choosing $c = \mu \, \left(\frac{H_0}{H_1}\right)^{\tau} > 0$,  
\begin{align*}
\frac{\norm{\nabla_t}_1}{\sqrt{1 - \beta} \, \sum_{j = 0}^{t-1} (\sqrt{\beta})^{j} \, \norm{\nabla_{t-j}}_1} \geq \frac{1 - \sqrt{\beta} \, \exp(a)}{\sqrt{1 - \beta}} \, \frac{1}{1 + \left(\frac{H_0}{H_1} \, \frac{1}{f(\tht)}\right)^{\tau}}
\end{align*}
\label{lemma:rms-ratio-2pl}
\end{lemma}
\begin{proof}
Using part (1) of~\cref{lemma:grad-multi-lip-gen} with $p = \infty$, $q = 1$ and $y = \theta_{t-1}$, $x = \tht$, $c > 0$, 
\begin{align*}
\norm{\nabla_{t-1}}_1 + c & \leq (\norm{\nabla_{t}}_1 + c) \, \exp\left((H_0 + H_1 \, f(\tht)) \, e \,  \frac{\norm{\theta_{t-1} - \tht}_\infty}{c} \right) \\
& \leq (\norm{\nabla_{t}}_1 + c) \, \exp\left((H_0 + H_1 \, f(\theta_1)) \, e \,  \frac{\norm{\theta_{t-1} - \tht}_\infty}{c} \right) \tag{Using assumption (ii)} \\
& \leq (\norm{\nabla_{t}}_1 + c) \, \exp(a) \tag{Using assumption (i)} \\
\implies \norm{\nabla_{t-j}}_1 & \leq (\norm{\nabla_{t}}_1 + c) \, \exp(a \, j) \tag{Recursing and since $c > 0$}
\end{align*}
Using the above relation to lower-bound $\sqrt{1 - \beta} \, \sum_{j = 0}^{t-1} (\sqrt{\beta})^{j} \, \norm{\nabla_{t-j}}_1$, first note that, 
\begin{align*}
\sqrt{1 - \beta} \, \sum_{j = 0}^{t-1} (\sqrt{\beta})^{j} \, \norm{\nabla_{t-j}}_1 & \leq \sqrt{1 - \beta} \, \sum_{j = 0}^{t-1} (\sqrt{\beta})^j \, (\norm{\nabla_{t}}_1 + c) \, \exp(a \, j) \\
& \leq \frac{\sqrt{1 - \beta} \; (\norm{\nabla_{t}}_1 + c)}{1 - \sqrt{\beta} \, \exp(a)} \tag{Since $\sqrt{\beta} \, \exp(a) < 1$} \\
\implies \frac{\norm{\nabla_t}_1}{\sqrt{1 - \beta} \, \sum_{j = 0}^{t-1} (\sqrt{\beta})^{j} \, \norm{\nabla_{t-j}}_1} & \geq \frac{\norm{\nabla_t}_1}{\norm{\nabla_t}_1 + c} \, \frac{1 - \sqrt{\beta} \, \exp(a)}{\sqrt{1 - \beta}} 
\end{align*}
Furthermore, if $f$ satisfies~\cref{assn:pl} with $f^* = 0$ and $\mu(\theta) = \mu$, then, 
\begin{align*}
\frac{\norm{\nabla_t}_1}{\sqrt{1 - \beta} \, \sum_{j = 0}^{t-1} (\sqrt{\beta})^{j} \, \norm{\nabla_{t-j}}_1}  & \geq \frac{1 - \sqrt{\beta} \, \exp(a)}{\sqrt{1 - \beta}} \, \frac{1}{1 + \frac{c}{\mu \, [f(\tht)]^{\tau}}} \tag{Using~\cref{assn:pl} with $f^* = 0$ and $\mu(\theta) = \mu$} \\
\implies \frac{\norm{\nabla_t}_1}{\sqrt{1 - \beta} \, \sum_{j = 0}^{t-1} (\sqrt{\beta})^{j} \, \norm{\nabla_{t-j}}_1} & \geq \frac{1 - \sqrt{\beta} \, \exp(a)}{\sqrt{1 - \beta}} \, \frac{1}{1 + \left(\frac{H_0}{H_1} \, \frac{1}{f(\tht)}\right)^{\tau}} \tag{Setting $c = \mu \, \left(\frac{H_0}{H_1}\right)^{\tau}$}
\end{align*}
\end{proof}
\section{Convergence of Adam}
\label{app:adam}

\adam*
\begin{proof}
Using~\cref{lemma:adam-descent} for $\eta \leq \frac{1}{C_3 \, \sqrt{H_1}}$ and $\beta_1 \leq \beta_2$ where $C_3 := \sqrt{\frac{1 - \beta_1}{1 - \beta_2}}$, and, 
\begin{align}
f(\thtt) & \leq f(\tht) - \eta \, \langle \nabla_t, d_t \rangle +  \left(H_0 + H_1 \, f(\tht) \right) \, \eta^2 \, C^2_3 \label{eq:adam-2pl-inter-1}    
\end{align}
Simplifying the second term on the RHS of~\cref{eq:adam-2pl-inter-1}, first note that, 
\begin{align}
\langle \nabla_t, d_t \rangle &= \sum_i [\nabla_t]_i \, [d_t]_i = \sum_i \frac{\ti{g} \, \ti{m}}{\ti{\sqrt{v}}} = \sum_i \frac{\ti{g} \, (\ti{m} - \ti{g})}{\ti{\sqrt{v}}} + \frac{\ti{g^2}}{\ti{\sqrt{v}}} \nonumber \\
& \geq -\sum_i \frac{|\ti{g}| \, |\ti{m} - \ti{g}|}{\ti{\sqrt{v}}} + \sum_i \frac{\ti{g^2}}{\ti{\sqrt{v}}} \nonumber \\
\implies - \langle \nabla_t, d_t \rangle & \leq \underbrace{\sum_i \frac{|\ti{g}| \, |\ti{m} - \ti{g}|}{\ti{\sqrt{v}}}}_{:= \text{Term (i)}} - \underbrace{\sum_i \frac{\ti{g^2}}{\ti{\sqrt{v}}}}_{:= \text{Term (ii)}} \label{eq:adam-2pl-secondterm-1}
\end{align}
Bounding Term (i) using~\cref{lemma:adam-mom-bound} and using a constant step-size $\eta \leq \frac{1}{C_3 \, \sqrt{H_1}} \, \ln(1/\sqrt{\beta_1})$, we get that, 
\begin{align}
\text{Term (i)} & \leq  \frac{\beta_1\,e \, C_3 \, \sum_{s = 1}^{t-1} (\sqrt{\beta_1})^{t-1-s} \, \eta_{s}}{\sqrt{1 - \beta_2}} \, (H_0 + H_1 \, f(\tht)) \\
& \leq \eta \, \underbrace{\frac{\beta_1\,e \, C_3 \, }{\sqrt{1 - \beta_2} \, (1 - \sqrt{\beta_1})}}_{:= C_1} (H_0 + H_1 \, f(\tht)) \tag{Since $\eta_s = \eta$ for all $s$} \\
\implies \text{Term (i)} &  \leq \, \eta  \, C_1 \,  (H_0 + H_1 \, f(\tht)) \label{eq:adam-2pl-secondterm-p1} 
\end{align}
Similar to the proof of~\cref{thm:rms-pl}, we use~\cref{lemma:rms-precond-bound} to bound Term (ii), 
\begin{align}
\text{Term (ii)} & \geq \norm{\nabla_t}_1 \, \underbrace{\frac{\norm{\nabla_t}_1}{\sqrt{1 - \beta_2} \, \sum_{j = 0}^{t-1} (\sqrt{\beta_2})^{j} \, \norm{\nabla_{t-j}}_1}}_{(*)} \label{eq:adam-1pl-secondterm} 
\end{align}

\textbf{Phase 1:} Define $T_0$ to be the first iteration s.t. $f(\theta_{T_0}) < \max\left\{\epsilon, \frac{H_0}{H_1} \right\}$, and analyze~\cref{eq:adam-2pl-inter-1} for $t < T_0$ where $f(\tht) \geq \max\left\{\epsilon, \frac{H_0}{H_1} \right\}$. Using~\cref{eq:adam-2pl-secondterm-p1,eq:adam-1pl-secondterm} and simplifying~\cref{eq:adam-2pl-inter-1} in this case, 
\begin{align}
f(\thtt) & \leq f(\tht) - \eta \, \norm{\nabla_t}_1 \, \underbrace{\frac{\norm{\nabla_t}_1}{\sqrt{1 - \beta_2} \, \sum_{j = 0}^{t-1} (\sqrt{\beta_2})^{j} \, \norm{\nabla_{t-j}}_1}}_{(*)} + 2 \, H_1 \, f(\tht)  \, \eta^2 \, (C_1+ C^2_3)     \label{eq:adam-2pl-p1-1}
\end{align}
We will now use an inductive proof and prove that for all $t \leq T_0$, $f(\tht) \leq f(\theta_1)$. 

\textbf{Base Case}: For $t = 1$, this is true by definition. 

\textbf{Inductive Hypothesis}: Assume for iteration $t < T_0$, $f(\tht) \leq f(\theta_1)$.

\textbf{Induction}:  We will prove that $f(\thtt) \leq f(\theta_1)$. Since $f(\tht) \leq f(\theta_1)$ by the inductive hypothesis, and $\eta \leq \frac{\mu   }{2 e C_3} \, \left(\frac{H_0}{H_1}\right)^\tau \, \frac{1}{H_0 + H_1 \, f(\theta_1)} \, \ln(\nicefrac{1}{\beta_2^{1/4}})$
so
\[
\|\theta_{t+1}-\theta_{t}\|_\infty \leq  \eta \, C_3 \leq \underbrace{\mu  \, \left(\frac{H_0}{H_1}\right)^\tau}_{c} \,\frac{1}{e}\, \frac{1}{H_0 + H_1 \, f(\theta_1)} \, \underbrace{\ln(\nicefrac{1}{\beta_2^{1/4}})}_{a},  
\]
 we can use~\cref{lemma:rms-ratio-2pl} and obtain that, 

\begin{align}
(*) = \frac{\norm{\nabla_t}_1}{\sqrt{1 - \beta_2} \, \sum_{j = 0}^{t-1} (\sqrt{\beta_2})^{j} \, \norm{\nabla_{t-j}}_1} \geq \underbrace{\frac{1 - \beta_2^{1/4}}{\sqrt{1 - \beta_2}}}_{:= C_2} \, \frac{1}{1 + \left( \frac{H_0}{H_1 \, f(\tht)} \right)^{\tau} }  = \frac{C_2}{1 + \left( \frac{H_0}{H_1 \, f(\tht)} \right)^{\tau} }  \label{eq:adam-2pl-induction-1}
\end{align}
Combining~\cref{eq:adam-2pl-p1-1} and~\cref{eq:adam-2pl-induction-1} and noting that in Phase (1), $\frac{H_0}{H_1 \, f(\tht)} \leq 1 \implies (*) \geq \frac{C_2}{2}$,
\begin{align*}
f(\thtt) & \leq f(\tht) - \eta \, \norm{\nabla_t}_1 \, \frac{C_2}{2} + 2 \, H_1 \, f(\tht)  \, \eta^2 \, (C_1+ C^2_3) \\
& \leq f(\tht) - \eta \, [f(\tht)]^\tau \, \frac{C_2 \, \mu}{2} + 2 \, H_1 \, f(\tht)  \, \eta^2 \, (C_1+ C^2_3) 
\tag{Using~\cref{assn:pl} with $f^* = 0$ and $\mu(\theta) = \mu$} \\
& \leq f(\tht) - \eta \, \frac{f(\tht)}{[f(\theta_1)]^{1-\tau}} \, \frac{C_2 \, \mu}{2}+ 2 \, H_1 \, f(\tht)  \, \eta^2 \, (C_1+ C^2_3) \tag{Since $f(\tht) \leq f(\theta_1)$ by the induction hypothesis}
\end{align*}

Using~\cref{lemma:sd-p1} with $A = \frac{C_2}{2} \, \frac{\mu}{[f(\theta_1)]^{1 - \tau}}$, $B = 2 \, H_1(C_1+ C^2_3)$, $\bar{\eta}_1 = \min\left\{\frac{1}{C_3 \, \sqrt{H_1}}, \frac{1}{C_3 \, \sqrt{H_1}} \, \ln(1/\sqrt{\beta_1}), \frac{\mu }{2 e C_3} \, \left(\frac{H_0}{H_1}\right)^\tau \, \frac{1}{H_0 + H_1 \, f(\theta_1)} \, \ln(\nicefrac{1}{\beta_2^{1/4})} \right\}$, $\delta = \max\left\{\epsilon, \frac{H_0}{H_1} \right\}$, we get that with $\eta = \min\left\{\bar{\eta}_1, \frac{C_2 \, \mu}{8 \, (C_1+ C^2_3) \, H_1 \, [f(\theta_1)]^{1 - \tau}} \right\}$, $f(\thtt) \leq f(\tht)$. This completes the induction. Moreover, $f(\theta_{T_0}) \leq \max\left\{\epsilon, \frac{H_0}{H_1} \right\}$ after 
\[
T_0 \geq O \left(\frac{ [f(\theta_1)]^{1 - \tau}}{C_2 \, \mu \, \eta} \right) = O \left(\ln\left(\min\left\{\frac{1}{\epsilon}, \frac{H_1}{H_0} \right\} \right) \right)\, \text{iterations.}
\]
Hence, if $\epsilon \geq \frac{H_0}{H_1}$, $\Adam$ with a constant step-size requires $O(\ln(1/\epsilon))$ iterations to guarantee convergence to an $\epsilon$ sub-optimality. 

\textbf{Phase 2}: Consider iteration $t > T_0$ such that $f(\tht) < \frac{H_0}{H_1}$. Simplifying~\cref{eq:adam-2pl-inter-1} in this case, 
\begin{align}
f(\thtt) & \leq f(\tht) - \eta \, \norm{\nabla_t}_1 \, \underbrace{\frac{\norm{\nabla_t}_1}{\sqrt{1 - \beta_2} \, \sum_{j = 0}^{t-1} (\sqrt{\beta_2})^{j} \, \norm{\nabla_{t-j}}_1}}_{(*)} + 2 \, H_0  \, \eta^2 (C_1+ C^2_3)     \label{eq:adam-2pl-p1-2}
\end{align}
We will now use an inductive proof and prove that for all $t \geq T_0$, $f(\tht) \leq f(\theta_{T_0}) \leq \frac{H_0}{H_1} \leq f(\theta_1)$. 

\textbf{Base Case}: For $t = T_0$, this is true by definition. 

\textbf{Induction}: We will prove that $f(\thtt) \leq f(\theta_{T_0})$. Since $f(\tht) \leq f(\theta_{T_0}) \leq \frac{H_0}{H_1} \leq f(\theta_1)$ by the inductive hypothesis, and ensuring that $\eta \leq \frac{\mu  }{2C_3 e} \, \left(\frac{H_0}{H_1}\right)^{\tau} \, \frac{1}{H_0 + H_1 \, f(\theta_1)} \, \ln(\nicefrac{1}{\beta^{1/4}})$, we can use~\cref{lemma:rms-ratio-2pl} and obtain that, 

\begin{align}
(*) = \frac{\norm{\nabla_t}_1}{\sqrt{1 - \beta} \, \sum_{j = 0}^{t-1} (\sqrt{\beta})^{j} \, \norm{\nabla_{t-j}}_1} \geq \underbrace{\frac{1 - \beta^{1/4}}{\sqrt{1 - \beta}}}_{:= C_2} \, \frac{1}{1 + \left( \frac{H_0}{H_1 \, f(\tht)} \right)^{\tau} }  = \frac{C_2}{1 + \left( \frac{H_0}{H_1 \, f(\tht)} \right)^{\tau} }  \label{eq:adam-2pl-induction-2}
\end{align}
Combining~\cref{eq:adam-2pl-p1-2} and~\cref{eq:adam-2pl-induction-2} and noting that, $\frac{H_0}{H_1 \, f(\tht)} \geq 1 \implies (*) \geq \frac{C_2}{2 \, \left( \frac{H_0}{H_1 \, f(\tht)} \right)^{\tau}}$, 
\begin{align*}
f(\thtt) & \leq f(\tht) - \eta \, \norm{\nabla_t}_1 \, \frac{C_2}{2} \, \left(\frac{H_1}{H_0}\right)^\tau \, [f(\theta)]^\tau +  2 \, H_0  \, \eta^2(C_1+ C^2_3)  \\
& \leq f(\tht) - \eta \, \frac{C_2 \, \mu}{2} \, \left( \frac{H_1}{H_0} \right)^\tau \, [f(\tht)]^{2 \, \tau}+  2 \, H_0   \, \eta^2 \, (C_1+ C^2_3)\tag{Using~\cref{assn:pl} with $f^* = 0$ and $\mu(\theta) = \mu$}
\end{align*}

Using~\cref{lemma:sd-p2} with $A = \frac{C_2 \, \mu}{2} \, \left( \frac{H_1}{H_0} \right)^\tau$, $B = 2 \, H_0\, (C_1+ C^2_3) $, \\ $\bar{\eta}_1 = \min\left\{\frac{1}{C_3 \, \sqrt{H_1}}, \frac{1}{C_3 \, \sqrt{H_1}} \, \ln(1/\sqrt{\beta_1}), \frac{\mu  }{2 e C_3} \, \left(\frac{H_0}{H_1}\right)^\tau \, \frac{1}{H_0 + H_1 \, f(\theta_1)} \, \ln(\nicefrac{1}{\beta_2^{1/4}}) \right\}$, $r = 2\tau$,  $\delta = \epsilon$, we get that with \\ $\eta = \min\left\{\bar{\eta}_1, \frac{C_2 \, \mu \, \epsilon^{2\tau}}{8  \, H_0 \, (C_1+ C^2_3)} \, \left( \frac{H_1}{H_0} \right)^\tau \right\}$, $f(\thtt) \leq f(\tht) \leq f(\theta_{T_0}) \leq \frac{H_0}{H_1}$ for all $t \geq T_0$. This completes the induction.

Furthermore, for $\tau = \frac{1}{2}$, $f(\theta_{T_\epsilon + T_0}) \leq \epsilon$ after 
\[
T_\epsilon \geq O\left(\frac{1}{C_2 \, \mu \, \eta} \, \sqrt{\frac{H_0}{H_1} } \ln\left(\frac{f_{T_0}}{2}\right)\right) = \tilde O \left( \frac{1}{\epsilon} \right) \text{iterations}. 
\] 
Else, for $\tau < \frac{1}{2}$, $f(\theta_{T_\epsilon + T_0}) \leq \epsilon$ after 
\[
T_\epsilon \geq O \left( \frac{[f(\theta_{T_0})]^{1 - 2 \, \tau}}{C_2 \, \mu \, \eta \, (1 - 2 \, \tau)} \, \left( \frac{H_0}{H_1} \right)^\tau \right) = O \left( \frac{1}{\epsilon^{2 \, \tau}} \right) \text{iterations}. 
\]
Else, for $\tau > \frac{1}{2}$, $f(\theta_{T_\epsilon + T_0}) \leq \epsilon$ after 
\[
T_\epsilon \geq O \left(\frac{1}{C_2 \, \mu \, \eta \, (2 \, \tau - 1) \, \epsilon^{ 2 \tau - 1}} \, \left( \frac{H_0}{H_1} \right)^\tau \right) = O \left( \frac{1}{\epsilon^{4 \tau - 1}} \right) \text{iterations}. 
\]
Hence, for $T := T + T_\epsilon$, $f(\theta_{T+1}) \leq \epsilon$ after, 
\[
O \left( \frac{1}{\epsilon^{2 \tau}} + \ln \left( \frac{H_1}{H_0} \right) \right) \, \text{iterations, if $\tau \leq \frac{1}{2}$ ;}
\]
\[
O \left( \frac{1}{\epsilon^{4 \tau - 1}} + \ln \left( \frac{H_1}{H_0} \right) \right) \, \text{iterations, if $\tau > \frac{1}{2}$}
\]

\end{proof}

\begin{theorem}
For the logistic regression defined in~\cref{prop:logistic} on linearly separable data with the normalized margin equal to $\gamma_p := \min_{i\in [n]} \frac{y_i \, \langle x_i, \theta^* \rangle}{\norm{\theta^*}_p} > 0$, Adam with the update in~\cref{eq:adam-update-coordinate} with $\theta_1 = 0$ has the following convergence rate:   
\begin{itemize}
    \item Using constant step size $\eta  = O(1)$ guarantees that after 
    \[
T = O\left( \frac{1}{\gamma_1^2} \, \left[n^2 + \ln\left(\frac{1}{n \epsilon}\right)\right]\right)
\]
iterations, 
    $f(\theta_{T+1}) \leq \epsilon$. 
\end{itemize}
\label{thm:adam-logreg}
\end{theorem}

\begin{proof}
Since we are considering logistic regression on linearly separable data, $f^* = 0$ and $H_0 = 0$. We will use a similar proof as that for~\cref{thm:adam-pl}. In particular, using~\cref{lemma:adam-descent} for $\eta \leq \bar{\eta}_0 := \frac{1}{C_3 \, \sqrt{H_1}}$ and $\beta_1 \leq \beta_2$ where $C_3 := \sqrt{\frac{1 - \beta_1}{1 - \beta_2}}$, and, 
\begin{align}
f(\thtt) & \leq f(\tht) - \eta \, \langle \nabla_t, d_t \rangle +  \left(H_1 \, f(\tht) \right) \, \eta^2 \, C_3 \label{eq:adam-logreg-inter-1}    
\end{align}
Simplifying the second term on the RHS of~\cref{eq:adam-logreg-inter-1}, first note that, 
\begin{align}
\langle \nabla_t, d_t \rangle &= \sum_i [\nabla_t]_i \, [d_t]_i = \sum_i \frac{\ti{g} \, \ti{m}}{\ti{\sqrt{v}}} = \sum_i \frac{\ti{g} \, (\ti{m} - \ti{g})}{\ti{\sqrt{v}}} + \frac{\ti{g^2}}{\ti{\sqrt{v}}} \nonumber \\
& \geq -\sum_i \frac{|\ti{g}| \, |\ti{m} - \ti{g}|}{\ti{\sqrt{v}}} + \sum_i \frac{\ti{g^2}}{\ti{\sqrt{v}}} \nonumber \\
\implies - \langle \nabla_t, d_t \rangle & \leq \underbrace{\sum_i \frac{|\ti{g}| \, |\ti{m} - \ti{g}|}{\ti{\sqrt{v}}}}_{:= \text{Term (i)}} - \underbrace{\sum_i \frac{\ti{g^2}}{\ti{\sqrt{v}}}}_{:= \text{Term (ii)}} \label{eq:adam-logreg-secondterm-1}
\end{align}
Bounding Term (i) using~\cref{lemma:adam-mom-bound} and using a constant step-size equal to $\eta$, we get that, 
\begin{align}
\text{Term (i)} & \leq  \frac{\beta_1\,e \, C_3 \, \sum_{s = 1}^{t-1} (\sqrt{\beta_1})^{t-1-s} \, \eta_{s}}{\sqrt{1 - \beta_2}} \, ( H_1 \, f(\tht)) \\
& \leq \eta \, \underbrace{\frac{\beta_1\,e \, C_3 \, }{\sqrt{1 - \beta_2} \, (1 - \sqrt{\beta_1})}}_{:= C_1} ( H_1 \, f(\tht)) \tag{Since $\eta_s = \eta$} \\
\implies \text{Term (i)} &  \leq \, \eta C_1 \,  (H_1 \, f(\tht)) \label{eq:adam-logreg-secondterm-2} 
\end{align}
Similar to the proof of~\cref{thm:rms-pl}, we use~\cref{lemma:rms-precond-bound} to bound Term (ii), 
\begin{align}
\text{Term (ii)} & \geq \norm{\nabla_t}_1 \, \underbrace{\frac{\norm{\nabla_t}_1}{\sqrt{1 - \beta_2} \, \sum_{j = 0}^{t-1} (\sqrt{\beta_2})^{j} \, \norm{\nabla_{t-j}}_1}}_{(*)} \label{eq:adam-logreg-secondterm-3} 
\end{align}

We will use~\cref{lemma:rms-ratio-2pl} to bound (*). For this, we need to ensure that (i) the value of $\eta$ ensures that $\norm{\thtt - \tht}_\infty = \eta \, \norm{d_t}_{\infty} = \eta \, C_3 < \underbrace{\frac{\gamma_1}{3 n \, e}}_{c/e} \, \frac{1}{H_1 \, f(\theta_1)} \, \underbrace{\ln(1/\beta_2^{1/4})}_{a}$. Setting 
\[
\eta \leq \bar{\eta}_1 := \frac{\gamma_1 }{4 \, n \, e \, C_3} \, \frac{1}{H_1 \, f(\theta_1)} \, \ln(1/ \beta_2^{1/4})
\]
ensures that this condition is satisfied. Furthermore, we need to ensure that (ii) $f(\tht) \leq f(\theta_1)$. For this, we will use an inductive argument and prove that for all $t \geq 1$, $f(\tht) \leq f(\theta_1)$. 

\textbf{Base Case}: For $t = 1$, this is true by definition. 

\textbf{Inductive Hypothesis}: Assume for iteration $t > 1$, $f(\tht) \leq f(\theta_1)$.

\textbf{Induction}: We will prove that $f(\thtt) \leq f(\theta_1)$. Since $f(\tht) \leq f(\theta_1)$ by the inductive hypothesis, and $\eta \leq \bar{\eta}_1$, we can use~\cref{lemma:rms-ratio-2pl} and obtain that, 
\begin{align}
(*) = \frac{\norm{\nabla_t}_1}{\sqrt{1 - \beta} \, \sum_{j = 0}^{t-1} (\sqrt{\beta_2})^{j} \, \norm{\nabla_{t-j}}_1} & \geq \frac{\norm{\nabla_t}_1}{\norm{\nabla_t}_1 + c} \, \frac{1 - \sqrt{\beta_2} \, \exp(a)}{\sqrt{1 - \beta_2}}    \label{eq:adam-logreg-precond}
\end{align}
To complete the induction, we will consider two phases. For this, define $T_0$ to be the first iteration s.t. $f(\theta_{T_0}) < \max\left\{\epsilon, \frac{\ln(2)}{n}\right\}$.

\textbf{Phase 1:} For all $t < T_0$, $f(\tht) > \frac{\ln(2)}{n}$ and we use part (2) of~\cref{prop:logistic} to conclude that $\norm{\nabla_t}_1 \geq \frac{\gamma_1}{3n}$. Using this relation with~\cref{eq:adam-logreg-precond} and using that $c = \frac{\gamma_1}{3n}$, $\alpha = \ln(\frac{1}{\beta_2^{1/4}})$
\begin{align*}
(*) \geq \frac{1}{2} \, \underbrace{\frac{1 - \beta_2^{1/4}}{\sqrt{1- \beta_2}}}_{C_4}  \tag{Using that $\eta \leq \bar{\eta}_1$ and since $\beta_2 < 1$}    
\end{align*}
Combining the above relation with~\cref{eq:adam-logreg-inter-1,eq:adam-logreg-secondterm-2},
\begin{align*}
f(\thtt) & \leq f(\tht) - \frac{\eta \,C_4\, \norm{\nabla_t}_1}{2} + H_1 \, f(\tht) \, \eta^2(C_1+C^2_3) \\
& \leq  f(\tht) - \frac{\eta \,C_4\, \norm{\nabla_t}_1}{2} + H_1 \, f(\theta_1) \, \eta^2(C_1+C^2_3) \tag{Using the inductive hypothesis} \\
& \leq f(\tht) - \frac{\eta\,C_4 \, \gamma_1}{6n} + H_1 \, f(\theta_1) \, \eta^2(C_1+C^2_3) \tag{Using that for $t < T_0$, $\norm{\nabla_t}_1 \geq \frac{\gamma_1}{3n}$} \\
& \leq f(\tht) - \frac{\eta \,C_4\, \gamma_1}{12n} \tag{Setting $\eta \leq \bar{\eta}_2 := \frac{ \gamma_1\,C_4}{12 \, H_1 \, f(\theta_1) \, n \, (C_1+C^2_3)}$} \\
\intertext{Since $f(\thtt) \leq f(\tht) \leq f(\theta_1)$, this completes the induction for Phase 1. By recursing for $T_0-1$ iterations,}
\implies f(\theta_{T_0}) & \leq f(\theta_1) -  \frac{\eta\,C_4 \, \gamma_1}{12n} \, (T_0-1) = \ln(2) -  \frac{\eta\,C_4 \, \gamma_1}{12n} \, (T_0-1)
\end{align*}
Hence, by using that $\eta \leq \min\{\bar{\eta}_0, \bar{\eta}_1 , \bar{\eta}_2\}$,
\begin{align*}
T_0 & \geq O \left(\frac{n}{\eta \,C_4\, \gamma_1} \right) = O \left( \left(\frac{n}{\gamma_1}\right)^2   \right)
\end{align*}
iterations guarantee that $f(\theta_{T_0}) \leq \frac{\ln(2)}{n}$.

\textbf{Phase 2:} After $T_0$ iterations, $f(\theta_{T_0}) \leq \frac{\ln(2)}{n} < f(\theta_1)$. We will now prove that for $t > T_0$, $f(\tht) \leq f(\theta_{T_0}) < f(\theta_1)$, thus completing the induction in Phase 2. We will again do this proof via induction. 

\textbf{Base Case}: For $t = T_0$, this is true by definition. 

\textbf{Inductive Hypothesis}: Assume for iteration $t > T_0$, $f(\tht) \leq f(\theta_{T_0})$.

\textbf{Induction}: We will prove that $f(\thtt) \leq f(\theta_{T_0)}$. Since $f(\tht) \leq f(\theta_{T_0}) < f(\theta_1)$ by the inductive hypothesis, we can again use an appropriate step-size 
\[
\eta \leq \bar{\eta}_3 :=  \frac{\gamma_1}{2e \, C_3 \, H_1}\ln\left(\frac{1}{\beta_2^{1/4}}\right)\,,
\]
which then guarantees
\[
\|\theta_{t+1} - \tht\|_p \leq \eta\, C_3 \leq \frac{\gamma_1 }{2e H_1 }\ln\left(\frac{1}{\beta_2^{1/4}}\right) = \frac{ca}{eH_1f(\theta)}, \qquad c  = \frac{\gamma_1 \, f(\theta_1)}{2}, \, a = \ln\left(\frac{1}{\beta_2^{1/4}}\right)
\]
and~\cref{eq:adam-logreg-precond} to bound (*) as follows:
\begin{align*}
(*) = \frac{\norm{\nabla_t}_1}{\sqrt{1 - \beta_2} \, \sum_{j = 0}^{t-1} (\sqrt{\beta_2})^{j} \, \norm{\nabla_{t-j}}_1} & \geq \frac{\norm{\nabla_t}_1}{\norm{\nabla_t}_1 + c} \, \frac{1 - \sqrt{\beta_2} \, \exp(a)}{\sqrt{1 - \beta_2}} \\
\intertext{Using part (1) of~\cref{prop:logistic} to conclude that $\norm{\nabla_t}_1 \geq \frac{\gamma_p}{2} \, f(\tht)$, using the induction hypothesis to further bound $f(\tht)$ by $f(\theta_1)$ and using that $c = \frac{\gamma_1 \, f(\theta_1)}{2}$,}
(*) & \geq \frac{1}{2} \, \underbrace{\frac{1 - \beta_2^{1/4}}{\sqrt{1 - \beta_2}}}_{=:C_4}   \tag{Using that $\eta \leq \bar{\eta}_3$}
\end{align*}
Combining the above relation with~\cref{eq:adam-logreg-inter-1,eq:adam-logreg-secondterm-2},
\begin{align*}
f(\thtt) & \leq f(\tht) - \frac{\eta \,C_4\, \norm{\nabla_t}_1}{2} + H_1 \, f(\tht) \, \eta^2\, (C_1+ C^2_3) \\
& \leq f(\tht) - \frac{\eta \, C_4\,\gamma_1}{4} f(\tht) + H_1 \, f(\tht) \, \eta^2\, (C_1+ C^2_3)\tag{Using Part (1) of~\cref{prop:logistic}} \\
& \leq f(\tht) - \frac{\eta \, C_4\,\gamma_1}{8} f(\tht) \tag{Setting $\eta \leq \bar{\eta}_4 := \frac{\gamma_1 \,}{8 \, H_1  \, (C_1+C^2_3)} $} \\
\intertext{Since $f(\thtt) < f(\tht) < f(\theta_{T_0})$, this completes the induction. Furthermore, by recursing for $T$ iterations}
\implies f(\theta_{T_0 + T}) & \leq f(\theta_{T_0}) \, \exp\left(-\frac{\eta \, C_4\,\gamma_1}{8} \, T \right) \leq \frac{\ln(2)}{n} \, \, \exp\left(-\frac{\eta \,C_4\, \gamma_1}{8} \, T \right) 
\end{align*}
Hence, by using a step-size $\eta \leq \min\{\bar{\eta}_0, \bar{\eta}_3, \bar{\eta}_4\}$, 
\[
T \geq O \left( \frac{1}{\eta \,C_4\, \gamma_1} \, \ln\left(\frac{\ln(2)}{n \, \epsilon} \right) \right) =  O \left( \frac{1}{\gamma_1^2} \, \ln\left(\frac{1}{n \epsilon}\right) \right)
\]
iterations suffice to guarantee that $f(\theta_{T + T_0}) \leq \epsilon$. Hence, attaining an $\epsilon$ sub-optimality for logistic regression requires, 
\begin{align*}
T_{\text{final}} &= O \left( \frac{1}{\gamma_1^2} \, \ln\left(\frac{1}{n \epsilon}\right) \right) + O \left( \left(\frac{n}{\gamma_1}\right)^2  \right) \\
&= O\left(  \frac{1}{\gamma_1^2} \, \left[n^2 + \ln\left(\frac{1}{n \epsilon}\right)\right]\right) \text{iterations}. 
\end{align*}     
\end{proof}

\subsection{Helper Lemmas}
\label{app:adam-helper-lemmas}

\begin{lemma}
Under~\cref{assn:non-negative},~\cref{assn:nus} with $H_0 \geq 0$ and $H_1 \geq 0$, for the coordinate-wise Adam update in~\cref{eq:adam-update-coordinate}, with $\eta \leq \frac{1}{B \, \sqrt{H_1}}$ and $\beta_1 \leq \beta_2$, where $B := \sqrt{\frac{1 - \beta_1}{1 - \beta_2}}$, 
\begin{align*}
f(\thtt) & \leq f(\tht) - \eta \, \langle \nabla_t, d_t \rangle +  \left(H_0 + H_1 \, f(\tht) \right) \, \eta^2 \, B^2 \quad \text{;} \quad \norm{d_t}_\infty \leq B
\end{align*}   
\label{lemma:adam-descent}
\end{lemma}
\begin{proof}
If $\norm{\thtt - \tht}_\infty \leq \frac{1}{\sqrt{H_1}}$, using~\cref{lemma:descent-ineq-gen} with $p = \infty$ and $q = 1$, and denoting $\nabla_t := \nabla f(\tht)$ for convenience,
\begin{align*}
f(\thtt) & \leq f(\tht) - \eta \, \langle \nabla_t, d_t \rangle + \left(H_0 + H_1 \, f(\tht) \right) \, \eta^2  \, \normsq{d_t}_\infty
\end{align*}
Simplifying the third term on the RHS, $\norm{d_t}_\infty = \max_{i} \frac{\ti{m}}{\ti{\sqrt{v}}}$, and
\begin{align*}
\ti{m} & = (1-\beta_1) \, \sum_{s = 1}^{t} \beta_1^{t-s} \si{g} \\
\implies \ti{m^2} & \leq \left(1-\beta_1\right)^2 \, \left(\sum_{s = 1}^{t} \beta_1^{t-s}\right) \left(\sum_{s = 1}^{t} \beta_1^{t-s} \si{g^2} \right) \tag{By Cauchy--Schwarz} \\
& \leq \left(1-\beta_1\right) \, \sum_{s = 1}^{t} \beta_1^{t-s} \si{g^2} \tag{By geometric series} \\
& \leq \frac{1 - \beta_1}{1 - \beta_2} \, \left[\left(1-\beta_2\right) \, \sum_{s = 1}^{t} \beta_2^{t-s} \si{g^2} \right] \tag{Since $\beta_1 \leq \beta_2$} \\
&= \frac{1 - \beta_1}{1 - \beta_2} \, \ti{v} \tag{By definition of $\ti{v}$} \\
\implies \frac{\ti{m^2}}{\ti{v}} & \leq \frac{1 - \beta_1}{1 - \beta_2}  \implies  \frac{\ti{m}}{\ti{\sqrt{v}}} \leq B := \sqrt{\frac{1 - \beta_1}{1 - \beta_2} } \\
\implies \norm{\thtt - \tht}_\infty & \leq \eta \, \norm{d_t}_\infty \leq \eta \, B \label{eq:adam-1pl-thirdterm}
\end{align*} 
Ensuring that $\eta \leq \frac{1}{B \, \sqrt{H_1}}$ guarantees that $\norm{\thtt - \tht}_\infty \leq \frac{1}{\sqrt{H_1}}$. Combining the above inequalities, 
\begin{align}
f(\thtt) & \leq f(\tht) - \eta \, \langle \nabla_t, d_t \rangle +  \left(H_0 + H_1 \, f(\tht) \right) \, \eta^2 \, B^2 
\end{align}    
\end{proof}

\begin{lemma}
Under~\cref{assn:non-negative},~\cref{assn:nus}, if $B = \sqrt{\frac{1 - \beta_1}{1 - \beta_2}}$, then, for the Adam update in~\cref{eq:adam-update-coordinate} with $\eta \leq \frac{1}{B \, \sqrt{H_1}} \, \min\left\{1 , \ln(1/\sqrt{\beta_1}) \right\}$ for all $t$,
\begin{align*}
\sum_i \frac{|\ti{g}| \, |\ti{m} - \ti{g}|}{\ti{\sqrt{v}}} &\leq C_1 \, (H_0 + H_1 \, f(\tht)) \quad \text{where,} \quad C_1 = \frac{\beta_1\,e \, B \, \sum_{s = 1}^{t-1} (\sqrt{\beta_1})^{t-1-s} \, \eta_{s}}{\sqrt{1 - \beta_2}}. 
\end{align*}
\label{lemma:adam-mom-bound}
\end{lemma}
\begin{proof}
\begin{align}
\sum_i \frac{|\ti{g}| \, |\ti{m} - \ti{g}|}{\ti{\sqrt{v}}} &\leq \frac{1}{\sqrt{1 - \beta_2}} \, \sum_i |\ti{m} - \ti{g}| \tag{Since $\ti{v} \geq (1 - \beta_2) \, \ti{g^2}$} \\
& \leq \frac{\norm{m_t - g_t}_1}{\sqrt{1 - \beta_2}} \label{eq:adam-1pl-mom-recursion-1}
\end{align}
In order to bound $\norm{m_t - g_t}_1$, note that for all $i \in [D]$,
\begin{align}
|\ti{m} - \ti{g}| &= |\beta_1 \, \tpi{m} + (1-\beta_1) \, \ti{g} - \ti{g}| = \beta_1 \, |\tpi{m} - \ti{g}| \\
& \leq \beta_1 \, |\tpi{m} - \tpi{g}| + \beta_1 \, |\tpi{g} - \ti{g}| \tag{By triangle inequality} \\
\implies \sum_i |\ti{m} - \ti{g}| & \leq \beta_1 \, \sum_i |\tpi{m} - \tpi{g}| + \beta_1 \, \sum_i |\tpi{g} - \ti{g}| \nonumber \\
\implies \norm{m_t - g_t}_1 & \leq \beta_1 \norm{m_{t-1}  - \nabla_{t-1}}_1 + \beta_1 \norm{\nabla_t - \nabla_{t-1}}_1 \nonumber \\
\intertext{Simplifying the second term on the RHS using~\cref{lemma:grad-lip-gen} with $p = \infty$, $q = 1$, $y = \theta_{t-1}$ and $x = \theta_t$. Ensuring that $\eta_{t-1} \leq \frac{1}{B \, \sqrt{H_1}}$ guarantees that $\norm{\tht - \thtt}_\infty \leq \frac{1}{\sqrt{H_1}}$. Hence,} 
\implies\norm{m_t - g_t}_1& \leq \beta_1 \, \norm{m_{t-1}  - \nabla_{t-1}}_1 + \beta_1 \, (H_0 + H_1 \, f(\tht)) \, e \, \norm{\tht - \theta_{t-1}}_{\infty} \nonumber \\
& \leq \beta_1 \, \norm{m_{t-1}  - \nabla_{t-1}}_{1} + \beta_1 \, \underbrace{(H_0 + H_1 \, f(\tht))}_{:= \bar{L}_t} \, e \, \eta_{t-1} \, B \tag{Using~\cref{lemma:adam-descent}} \\
\implies \frac{\norm{m_t - g_t}_1}{\bar{L}_t} & \leq \beta_1 \, \frac{\norm{m_{t-1}  - \nabla_{t-1}}_{1}}{\bar{L}_t} + \beta_1 \,e \, \eta_{t-1} \, B \label{eq:adam-1pl-mom-recursion-2}
\end{align}
We now write $\bar{L}_t$ in terms of $\bar{L}_{t-1}$. In particular, using~\cref{lemma:func-multi-lip-gen} with $p = \infty$, $q = 1$ and $y = \theta_{t-1}$, $x = \tht$, 
\begin{align}
\bar{L}_t = H_0 + H_1 \, f(\tht) & \geq H_0 + H_1 \left[ \left(f(\theta_{t-1}) + \frac{H_0}{H_1} \right) \exp\left(-\sqrt{H_1} \, \norm{\tht - \theta_{t-1}}_{\infty} \right) - \frac{H_0}{H_1} \right] \nonumber\\
& = \left(H_0 + H_1 \, f(\theta_{t-1}) \right) \, \exp\left(-\sqrt{H_1} \, \norm{\tht - \theta_{t-1}}_{\infty} \right) \nonumber \\
&= \bar{L}_{t-1} \, \exp\left(-\sqrt{H_1} \, \norm{\tht - \theta_{t-1}}_{\infty} \right) \nonumber\\
\implies \frac{1}{\bar{L}_t} & \leq \frac{1}{\bar{L}_{t-1}} \, \exp\left(\sqrt{H_1} \, \norm{\tht - \theta_{t-1}}_{\infty} \right) \leq \frac{1}{\bar{L}_{t-1}} \, \frac{1}{\sqrt{\beta_1}} \tag{Using~\cref{eq:adam-1pl-thirdterm} and since $\eta_{t-1} \leq \frac{1}{B \, \sqrt{H_1}} \, \ln(1/\sqrt{\beta_1}$}
\end{align}
Combining the above inequality with~\cref{eq:adam-1pl-mom-recursion-2}, 
\begin{align}
\frac{\norm{m_t - g_t}_1}{\bar{L}_t} & \leq \sqrt{\beta_1} \, \frac{\norm{m_{t-1}  - \nabla_{t-1}}_{1}}{\bar{L}_{t-1}} + \beta_1\,e \, \eta_{t-1} \, B \\
& \leq \beta_1\,e \, B \, \sum_{s = 1}^{t-1} (\sqrt{\beta_1})^{t-1-s} \, \eta_{s} \tag{Recursing and using that $m_{0,i} = g_{1,i}$} \\
\implies \norm{m_t - g_t}_1 & \leq (H_0 + H_1 \, f(\tht)) \, \beta_1\,e \, B \, \sum_{s = 1}^{t-1} (\sqrt{\beta_1})^{t-1-s} \, \eta_{s} \label{eq:adam-1pl-mom-recursion-final}
\end{align}
Combining the above inequality with~\cref{eq:adam-1pl-mom-recursion-1}, 
\begin{align*}
\sum_i \frac{|\ti{g}| \, |\ti{m} - \ti{g}|}{\ti{\sqrt{v}}}  & \leq (H_0 + H_1 \, f(\tht)) \, \underbrace{\frac{\beta_1\,e \, B \, \sum_{s = 1}^{t-1} (\sqrt{\beta_1})^{t-1-s} \, \eta_{s}}{\sqrt{1 - \beta_2}}}_{:= C_1} = C_1 \, (H_0 + H_1 \, f(\tht)) 
\end{align*}
    
\end{proof}

\section{Lower-Bound}
\label{app:lb}

\lb*
\begin{proof}
Note that $\nabla f(\theta) = \frac{-1}{1 + \exp(\theta)} < 0$ for all finite $\theta$. We define $\nabla_t := \nabla f(\tht)$. Since $\beta \in [0,1)$, for all $t$,   
\begin{align}
m_t & = (1-\beta) \, \sum_{s = 1}^{t} \beta^{t-s} \, \nabla_s \implies m_t < 0 \label{eq:mt-sign} \\
\implies \thtt & = \tht - \etat \, m_t = \tht + \etat | m_t | \leq \tht + \eta_1 \, |m_t| \tag{Since $\etat$ is non-increasing} \\
\implies \thtt & = \tht + \eta_1 \, |m_t| \label{eq:effective-update} 
\end{align}
Hence, $\thtt > \tht$ for all $t$. Since $\theta_1 = 0$, $\theta_t \geq 0$ for all $t \geq 1$ and consequently, $|\nabla_t| \leq 1$. Using this relation to uniformly bound the magnitude of $m_t$, 
\begin{align}
|m_t| & = |(1-\beta) \, \sum_{s = 1}^{t} \beta^{t-s} \, \nabla f(\theta_s)| \leq (1-\beta) \, \sum_{s = 1}^{t} \beta^{t-s} \, |\nabla_s| \tag{By triangle inequality} \\
\implies |m_t| & \leq 1 \label{eq:mt-magnitude}
\end{align}
We will now relate $|m_t|$ to $|\nabla_t|$. For this note that the logistic loss satisfies~\cref{assn:non-negative},~\cref{assn:nus} with $H_0 = 0$ and $H_1 = 1$, and~\cref{assn:pl} with $\zeta = 1$, $f^* = 0$ and $\mu = 1$. Using Part 2 of~\cref{lemma:grad-multi-lip-gen} and noting that $c :=\frac{\mu(H_0+H_1 f^*)}{H_1} = 0$, we get that, 
\begin{align*}
|\nabla_{t-1}| &\leq |\nabla_t| \, \exp \left( | \theta_{t} - \theta_{t-1}| \right) =  |\nabla_t| \, \exp \left( \eta_{t-1} \, |m_{t-1}| \right) \leq |\nabla_t| \, \exp \left( \eta_1 \, |m_{t-1}| \right) \tag{Since $\etat$ is non-increasing} \\
\implies |\nabla_{t-1}| &\leq |\nabla_t| \, \exp(\eta_1) \implies |\nabla_{t-j}| \leq |\nabla_t| \, \exp(\eta_1 \, j) \tag{Using~\cref{eq:mt-magnitude}}
\end{align*}
Using the above inequality to bound $|m_t|$ in terms of $|\nabla_t|$, 
\begin{align}
|m_t| & \leq (1-\beta) \, \sum_{j = 0}^{t} \beta^j \, |\nabla_{t-j}| \tag{Using the definition of $m_t$ and triangle inequality} \\
& \leq (1-\beta) \, |\nabla_{t}|  \sum_{j = 0}^{t} \beta^j \, \exp(\eta_1 \, j) \nonumber \\
& \leq \frac{1-\beta}{1 - \beta \, \exp(\eta_1)} \, |\nabla_{t}|  \tag{Using the assumption that $\eta_1 < \ln\left(\nicefrac{1}{\beta}\right)$} \\
\implies |m_t| & \leq \frac{1-\beta}{1 - \beta \, \exp(\eta_1)} \, |\nabla_t| \label{eq:mt-nablat}
\end{align}
Using~\cref{eq:mt-nablat} with~\cref{eq:effective-update}, 
\begin{align*}
\thtt &= \tht + \eta_1 \, |m_t| \leq \tht + \underbrace{\frac{\eta_1 \,  (1-\beta)}{1 - \beta \, \exp(\eta_1)}}_{:= C} \, |\nabla_t|
= \tht + \frac{C}{1 + \exp(\tht)} \tag{Using the gradient expression}
\end{align*}
In order to complete the proof we will use the following inequalities to bound $\frac{1}{1 + \exp(\theta)}$. 
\begin{align}
\forall \theta \text{,} \quad & \frac{1}{1 + \exp(\theta)}  < \frac{1}{\exp(\theta)} = \exp(-\theta) \quad \text{;} \quad \forall \theta \geq 0 \,, \, \frac{1}{1 + \exp(\theta)} = \frac{\exp(-\theta)}{1 + \exp(-\theta)} > \frac{\exp(-\theta)}{2} \label{eq:grad-bounds}
\end{align}
Using the above bounds, since $C> 0$,
\begin{align}
\thtt - \tht & \leq C\, \exp(-\tht)  \nonumber \\
\implies \exp(\thtt) - \exp(\tht) &= \exp(\tht) [ \exp(\thtt - \tht) - 1] \leq \exp(\tht) \,  [ \exp(C\,  \exp(-\tht)) - 1]  \nonumber\\
\intertext{For $x \in (0, 1]$ and $c > 0$, $\exp(cx) - 1 \leq (\exp(c)) \, c \, x$. Since $\tht \geq 0$, using this inequality with $x = \exp(-\tht) \leq 1$ and $c = C > 0$,}
\exp(\thtt) - \exp(\tht) & \leq \exp(\tht) \, \left[ C\, (\exp(C))  \, \exp(-\tht) \right] = C \, \exp(C) \nonumber \\
\intertext{Summing up from $t = 1$ to $T$ and telescoping, }
\exp(\theta_{T+1}) - \exp(\theta_1) & \leq C \, \exp(C) \, T \implies \exp(\theta_{T+1}) \leq 1 + C \, \exp(C) \, T \tag{Since $\theta_1 = 0$} \\
\implies \exp(-\theta_{T+1}) & \geq \frac{1}{1 + C\, \exp(C) \, T} \label{eq:adagrd-inter}
\end{align}
For all $x \in [0,1]$, $\ln(1+x) > \frac{x}{2}$. Since $\theta_{T+1} \geq 0$, $\exp(-\theta_{T+1}) \leq 1$, using this inequality with $x = \exp(-\theta_{T+1})$,
\begin{align}
f(\theta_{T+1}) &= \ln(1 + \exp(-\theta_{T+1})) \geq \frac{\exp(-\theta_{T+1})}{2} \geq \frac{1}{2 \, (1 + C \, \exp(C) \, T)} 
\end{align}
Since $C$ is a constant independent of $T$, the convergence rate is lower-bounded by $\Omega(1/T)$. 
\end{proof}

\subsection{Methods satisfying the update restrictions}
\label{app:lb-examples}

\paragraph{Gradient Descent:} For a bounded, constant step-size $\eta > 0$ independent of $T$, 
\[
\thtt = \tht - \eta \, \nabla f(\tht)
\]
Hence, (i) $\eta_t = \eta$ is non-increasing, (ii) $\beta = 0 \in [0,1)$ and (iii) since $\beta = 0$, there is no restriction on $\eta$.

\paragraph{Heavy-Ball Momentum:} For a constant momentum parameter $\beta_{\text{HB}} \in (0,1)$ and a bounded, constant step-size $\eta > 0$ independent of $T$, such that $\eta \leq \ln\left(\nicefrac{1}{\beta_{\text{HB}}} \right)$, 
\[
\thtt = \tht - \eta \, m_t \quad \text{;} \quad m_t = (1 - \beta_{\text{HB}} ) \, \sum_{s = 1}^{t} \beta_{\text{HB}}^{t-s} \, \nabla f(\theta_s)
\]
Hence, (i) $\eta_t = \eta$ is non-increasing, (ii) $\beta =  \beta_{\text{HB}} \in [0,1)$ and (iii) for $\eta \leq \ln\left(\nicefrac{1}{\beta_{\text{HB}}} \right)$. 

\paragraph{AdaGrad:} For a bounded, constant step-size $\eta > 0$ independent of $T$, 
\[
\thtt = \tht - \frac{\eta}{\sqrt{v_t}} \, \nabla f(\tht) \quad \text{;} \quad v_t = \sum_{s=1}^t \normsq{\nabla f(\theta_s)}
\]
Hence, (i) $\eta_t = \frac{\eta}{\sqrt{v_t}}$ is non-increasing, (ii) $\beta = 0 \in [0,1)$ and (iii) since $\beta = 0$, there is no restriction on $\eta$. 

\paragraph{AMSGrad:} For a constant $\beta_1 \in (0,1)$, $\beta_2 \in (0,1)$ and a bounded, constant step-size $\eta > 0$ independent of $T$, such that $\eta \leq \frac{1}{2} \ln\left(\nicefrac{1}{\beta_1}\right)$, $m_0 = g_{1}$, $v_{0} =0$, and $\hat{v}_{0} = 0$,
\begin{align*}
\thtt & = \tht - \frac{\eta \, m_t}{\sqrt{\hat{v}_t}} \quad \text{;} \quad \hat{v}_t = \max\{\hat{v}_{t-1}, v_{t}\} \\
v_t & = (1-\beta_2) \, \sum_{s=1}^t \beta_2^{t-s} \normsq{\nabla f(\theta_s)} \quad \text{;} \quad m_t = (1 - \beta_1) \, \sum_{s = 1}^{t} \beta_1^{t-s} \, \nabla f(\theta_s) \\
\end{align*}
Hence, (i) $\eta_t = \frac{\eta}{\sqrt{ \max\{v_t, \hat{v}_{t-1} \} }}$ is non-increasing, (ii) $\beta_1 \in (0,1)$ and (iii) the effective step-size 
\begin{align*}
\eta_1 = \frac{\eta}{\sqrt{ \max\{v_{1}, \hat{v}_{0} \} }} = \frac{\eta}{|\nabla_1|} = 2 \, \eta \leq \ln\left(\nicefrac{1}{\beta_1} \right) \tag{Since $\hat{v}_0 = 0$ and $\theta_1 = 0$}  
\end{align*}
\section{Adam with $(L_0, L_1)$ Non-Uniform Smoothness}
\label{app:gen-nus-adam}
\begin{assumption}
$f$ is $(L_0, L_1)$ non-uniform smooth i.e. for constants $L_0 \geq 0, L_1 > 0$, and $p, q \geq 1$ s.t. $\frac{1}{p} + \frac{1}{q} = 1$, for all $\theta$, $\norm{\nabla^2 f(\theta)}_{p \to q} \leq L_0 + L_1 \, \norm{\nabla f(\theta)}_{q}$.  
\label{assn:gnus}
\end{assumption}
In the following proofs, for simplicity, we use $\nabla_t :=\nabla f(\tht)$, and $g_{t,i}$ is the $i$-th component of this vector. The following theorem establishes an $O\left(\frac{1}{\epsilon}\right)$ rate for $\Adam$ when $\epsilon = O\left(\frac{L_0}{L_1}\right)$. 

\subsection{Adam convergence under General NUS}
\label{app:Adam-convergence-gnus}
\adamgnus*
\begin{proof}
Using~\cref{lemma:adam-descent-gnus} for $\eta \leq \bar{\eta}_0 := \frac{1}{C_3 \, L_1}$ and $\beta_1 \leq \beta_2$ where $C_3 := \sqrt{\frac{1 - \beta_1}{1 - \beta_2}}$, and, 
\begin{align}
f(\thtt) & \leq f(\tht) - \eta \, \langle \nabla_t, d_t \rangle +  \left(L_0 + L_1 \, \norm{\nabla f(\tht)}_1 \right) \, \eta^2 \, C^2_3 \label{eq:adam-2pl-inter-1-gnus}    
\end{align}
Simplifying the second term on the RHS of~\cref{eq:adam-2pl-inter-1-gnus}, first note that, 
\begin{align}
\langle \nabla_t, d_t \rangle &= \sum_i [\nabla_t]_i \, [d_t]_i = \sum_i \frac{\ti{g} \, \ti{m}}{\ti{\sqrt{v}}} = \sum_i \frac{\ti{g} \, (\ti{m} - \ti{g})}{\ti{\sqrt{v}}} + \frac{\ti{g^2}}{\ti{\sqrt{v}}} \nonumber \\
& \geq -\sum_i \frac{|\ti{g}| \, |\ti{m} - \ti{g}|}{\ti{\sqrt{v}}} + \sum_i \frac{\ti{g^2}}{\ti{\sqrt{v}}} \nonumber \\
\implies - \langle \nabla_t, d_t \rangle & \leq \underbrace{\sum_i \frac{|\ti{g}| \, |\ti{m} - \ti{g}|}{\ti{\sqrt{v}}}}_{:= \text{Term (i)}} - \underbrace{\sum_i \frac{\ti{g^2}}{\ti{\sqrt{v}}}}_{:= \text{Term (ii)}} \label{eq:adam-2pl-secondterm-1-gnus}
\end{align}
Bounding Term (i) using~\cref{lemma:adam-mom-bound-gnus} and using a constant step-size equal to $\eta$, we get that, 
\begin{align}
\text{Term (i)} & \leq  \frac{\beta_1\,e \, C_3 \, \sum_{s = 1}^{t-1} (\sqrt{\beta_1})^{t-1-s} \, \eta_{s}}{\sqrt{1 - \beta_2}} \, (L_0 + L_1 \, \norm{\nabla f(\tht)}_1 ) \\
& \leq \eta \, \underbrace{\frac{\beta_1\,e \, C_3 \, }{\sqrt{1 - \beta_2} \, (1 - \sqrt{\beta_1})}}_{:= C_1} (L_0 + L_1 \, \norm{\nabla f(\tht)}_1 ) \tag{Since $\eta_s = \eta$} \\
\implies \text{Term (i)} &  \leq \, \eta C_1 \,  (L_0 + L_1 \, \norm{\nabla f(\tht)}_1 ) \label{eq:adam-2pl-secondterm-p1-gnus} 
\end{align}
We use~\cref{lemma:rms-precond-bound} to bound $\text{Term (ii)} = \sum_i \frac{\ti{g^2}}{\ti{\sqrt{v}}}$, with $u_t := \sqrt{1 - \beta_2} \, \sum_{j = 0}^{t-1} (\sqrt{\beta_2})^{j} \, \norm{\nabla_{t-j}}_1$,  
\begin{align}
\sum_i \frac{\ti{g^2}}{\ti{\sqrt{v}}} & \geq \norm{\nabla_t}_1 \, \frac{\norm{\nabla_t}_1}{\sqrt{1 - \beta_2} \, \sum_{j = 0}^{t-1} (\sqrt{\beta_2})^{j} \, \norm{\nabla_{t-j}}_1} = \frac{\normsq{\nabla_t}_1}{u_t} \label{eq:adam-1pl-secondterm-gnus} 
\end{align}
Combining~\cref{eq:adam-2pl-inter-1-gnus} with the bounds on Term (i) and Term (ii), 
\begin{align}
f(\thtt) & \leq f(\tht) - \eta \,  \frac{\normsq{\nabla_t}_1}{u_t} +  \left(L_0 + L_1 \, \norm{\nabla f(\tht)}_1 \right) \, \eta^2 \, (C^2_3 + C_1)
\label{eq:adam-2pl-inter-2-gnus}
\end{align}

\textbf{Case (1): For $\epsilon \geq \frac{L_0}{L_1}$:} 
To bound $u_t$, we use~\cref{lemma:rms-ratio-2pl-gnus}, with $a = \ln(1/ \beta_2^{1/4})$, while $\sqrt{\beta_2} \exp( \ln(1/\beta_2^{1/4})) =\beta_2^{1/4}\leq 1$ and setting $\eta \leq \bar{\eta}_1 := \frac{1}{2 \, C_3 \, L_1}\ln(1/\beta_2^{1/4})$ 
\begin{align}
u_t & \leq \frac{\sqrt{1 - \beta_2} \; (\norm{\nabla_{t}}_1 + \frac{L_0}{L_1})}{1 - \sqrt{\beta_2} \, \exp(a)}
\end{align}
Combining the above inequality with~\cref{eq:adam-2pl-inter-2-gnus} and $C_2 := \frac{1-\beta_2^{1/4}}{ \sqrt{1 - \beta_2}}$, we have 
\begin{align}
    f(\thtt) & \leq f(\tht) - \eta \, C_2 \, \frac{\norm{\nabla_t}^2_1}{\norm{\nabla_t}_1 + \frac{L_0}{L_1}} +  \left(L_0 + L_1 \, \norm{\nabla_t}_1 \right) \, \eta^2 \,(C_1+ C^2_3) \label{eq:adam-2pl-inter-2half-gnus} 
\end{align}
Define $T$ to be the first iteration s.t. $\norm{\nabla f(\theta_{T})}_1 < \epsilon, \frac{L_0}{L_1}$, and analyze the above inequality for $t < T$ where $\norm{\nabla f(\theta_t)}_1 \geq \epsilon \geq \frac{L_0}{L_1}$. Setting $\eta \leq \bar{\eta}_2 := \frac{C_2}{8\, L_1\,(C_1+C_3^2)}$ and simplifying~\cref{eq:adam-2pl-inter-2-gnus} in this case, 
\begin{align}
f(\thtt) & \leq f(\tht) - \frac{\eta \, C_2}{2} \, \norm{\nabla_t}_1 +  \left(2 L_1 \, \norm{\nabla_t}_1 \right) \, \eta^2 \,(C_1+ C^2_3) \nonumber \\
& \leq f(\tht) - \frac{\eta \, C_2}{4} \, \norm{\nabla_t}_1 \tag{Since $\eta \leq \bar{\eta}_2$} \nonumber \\
\implies & \frac{\eta \, C_2}{4} \, \norm{\nabla_t}_1 \leq f(\tht) - f(\thtt) 
\label{eq:adam-2pl-inter-3-gnus}
\end{align}
Note that from the above we have $f(\thtt) \leq f(\tht)$ for all $t \leq T $ which indicates $f(\theta_{T}) \leq f(\theta_1)$. Recursing for $T$ iterations, we have 
\begin{align*}
    \frac{\eta \, C_2}{4}\, T \, \norm{\nabla_{T}} &\leq \frac{\eta \, C_2}{4} \, \sum_{t=1}^{T} \norm{\nabla_t}_1 \tag{Since $\norm{\nabla_{T}} \leq \frac{L_0}{L_1}$ and $\norm{\nabla_{t}} \geq \frac{L_0}{L_1} \, \forall t < T$}\\ 
    &\leq f(\theta_1)- f(\theta_{T + 1}) \\ 
    & \leq f(\theta_1) \tag{Since $f(\theta) \geq 0$}\\
    \implies \norm{\nabla_{T}} & \leq \frac{4 \, f(\theta_1)}{\eta \, C_2 \, T} 
\end{align*}
Hence when $T \geq \frac{4 \, f(\theta_1) }{\eta \, C_2\, \max \{\epsilon, \frac{L_0}{L_1}\}} =  O(\frac{1}{\epsilon})$, then $\norm{\nabla_{T}}_1 \leq \epsilon$, where $\eta = \min \left\{\bar{\eta}_0, \bar{\eta}_1, \bar{\eta}_2 \right\}$. 

\textbf{Case (2): For $\epsilon < \frac{L_0}{L_1}$:} Starting from~\cref{eq:adam-2pl-inter-2-gnus} and summing from $t = 1$ to $T$,
\begin{align*}
f(\theta_{T+1}) & \leq f(\theta_1) - \eta \, \underbrace{\sum_{t = 1}^{T} \frac{\normsq{\nabla_t}_1}{u_t}}_{:= (*)} + L_0 \,  \eta^2 \, (C^2_3 + C_1) \, T  + L_1 \, \eta^2 \, (C^2_3 + C_1) \,  \sum_{t = 1}^{T} \norm{\nabla_t}_1 
\end{align*}
In order to simplify (*), note that by Cauchy Schwarz,
\begin{align*}
\sum_{t = 1}^T \norm{\nabla_t} &= \sum_{t = 1}^T \frac{\norm{\nabla_t}}{\sqrt{u_t}} \, \sqrt{u_t} \leq \sqrt{\sum_{t = 1}^{T}  \, \frac{\normsq{\nabla_t}}{u_t}} \, \sqrt{\sum_{t = 1}^{T} u_t} \implies (*) \geq \frac{\left(\sum_{t = 1}^T \norm{\nabla_t}\right)^2}{\sum_{t = 1}^{T} u_t}
\end{align*}
Simplifying $\sum_{t = 1}^{T} u_t$,
\begin{align*}
\sum_{t = 1}^{T} u_t & = \sqrt{1 - \beta_2} \, \sum_{t = 1}^{T}  \left[ \sum_{j = 0}^{t-1} (\sqrt{\beta_2})^{j} \, \norm{\nabla_{t-j}}_1 \right] = \sqrt{1 - \beta_2} \, \sum_{t = 1}^{T} \left[ \sum_{s = 1}^{t} (\sqrt{\beta_2})^{t-s} \norm{\nabla_{s}}_1 \right] \\
& =  \sqrt{1 - \beta_2} \, \sum_{s = 1}^{T} \norm{\nabla_{s}}_1 \, \sum_{t = s}^{T} (\sqrt{\beta_2})^{t-s} = \sqrt{1 - \beta_2} \, \sum_{s = 1}^{T} \norm{\nabla_{s}}_1 \, \sum_{j = 0}^{T-s} (\sqrt{\beta_2})^{j} \tag{Since $\sum_{t = 1}^{T} \, \sum_{s = 1}^{t} = \sum_{s = 1}^{T} \, \sum_{t = s}^{T}$} \\
\implies \sum_{t = 1}^{T} u_t & \leq \frac{\sqrt{1 - \beta_2}}{1 - \sqrt{\beta_2}} \sum_{t = 1}^{T} \norm{\nabla_{t}}_1 \tag{Geometric series}
\end{align*}
Combining the above inequalities, we can conclude that $(*) \geq \frac{1 - \sqrt{\beta_2}}{\sqrt{1 - \beta_2}} \, \sum_{t = 1}^{T} \norm{\nabla_t}_1$, and therefore, if $C_4 := \frac{1 - \sqrt{\beta_2}}{\sqrt{1 - \beta_2}}$, then, 
\begin{align*}
f(\theta_{T+1}) & \leq f(\theta_1) - \eta \, \frac{C_4}{2} \, \sum_{t = 1}^{T} \norm{\nabla_t}_1 + L_0 \,  \eta^2 \, (C^2_3 + C_1) \, T  
\tag{Setting $\eta < \bar{\eta}_1 := \frac{C_4}{2 L_1 \, (C^2_3 + C_1)}$} \\
\implies \sum_{t = 1}^{T} \norm{\nabla_t}_1 & \leq \frac{2 \, f(\theta_1)}{C_4 \, \eta} + \frac{2 \, L_0 \, \eta \, (C^2_3 + C_1) \, T}{C_4} \\
\implies \min_{t \leq T} \norm{\nabla_t}_1 & \leq \frac{\sum_{t = 1}^{T} \norm{\nabla_t}_1}{T} \leq \frac{2 \, f(\theta_1)}{C_4 \, \eta \, T} + \frac{2 \, L_0 \, \eta \, (C^2_3 + C_1)}{C_4} \\
\intertext{Setting $\eta = \min\{\bar{\eta}_0, \bar{\eta}_1, \, \frac{ C_4}{4\, L_0\,(C_3^2 + C_1)} \, \epsilon \}$,}
\min_{t \leq T} \norm{\nabla_t}_1 & \leq \frac{2 \, f(\theta_1)}{C_4 \, \eta \, T} + \frac{\epsilon}{2} 
\intertext{Setting $T \geq \frac{4 f(\theta_1)}{C_4 \, \eta \, \epsilon} = O\left(\frac{1}{\epsilon^2}\right))$ is sufficient to ensure that,}
\min_{t \leq T} \norm{\nabla_t}_1 & \leq \epsilon
\end{align*}
\end{proof}

\subsection{Helper Lemmas}
\label{app:adam-helper-lemmas-gnus }
\begin{lemma}
If~\cref{assn:non-negative,assn:gnus} hold with $L_0 \geq 0$, $L_1 \geq 0$ and $\frac{1}{p} + \frac{1}{q} = 1$, then, for all $y, x$ s.t. $\norm{y - x}_p \leq \frac{1}{L_1}$, 
\begin{align*}
\norm{\nabla f(y)}_q  + \frac{L_0}{L_1} & \leq \left(\norm{\nabla f(x)}_q + \frac{L_0}{L_1} \right) \, \exp\left(L_1 \, \norm{y - x}_p \right) 
\end{align*}
\label{lemma:grad-multi-lip-gnus}
\end{lemma}
\begin{proof}
Define the function $h(\theta) := \ln(\norm{\nabla f(\theta)}_q  + \frac{L_0}{L_1})$. We will first prove that $h(\theta)$ is $L_1$-Lipschitz w.r.t the $\ell_{p}$ norm. Since $\norm{\cdot}_q$ can be non-smooth, consider a Clarke subgradient computed using Lemma \ref{lem:clarkesubdiffchainrule}
\[
\partial h(\theta) = \frac{\nabla^2 f(\theta)  \partial \|z\|_q}{\|z\|_q+\frac{L_0}{L_1}}, \qquad z =  \nabla f(\theta).
\]
\begin{align*}
\intertext{Then, for $g = \nabla^2 f(\theta) s$ where $s \in \partial(\norm{\cdot}_q)(\nabla f(\theta))$ is a subgradient of $\norm{\cdot}_{q}$ evaluated at $\nabla f(\theta)$, we write the norm of $v\in \partial h(\theta)$ as}
\norm{v}_{q} &= \frac{\norm{[\nabla^2 f(\theta)] \, s}_q}{\norm{\nabla f(\theta)}_q + \frac{L_0}{L_1}} \leq \frac{\norm{\nabla^2 f(\theta)}_{p \to q} \, \norm{s}_{p}}{\norm{\nabla f(\theta)}_q + \frac{L_0}{L_1}} \tag{By definition of the matrix norm} \\
\implies \norm{v}_{q} & \leq \frac{\norm{\nabla^2 f(\theta)}_{p \to q}}{\norm{\nabla f(\theta)}_q + \frac{L_0}{L_1}} \tag{If $s \in \partial(\norm{\cdot}_q)$, then, $\norm{s}_p \leq 1$} \\
\end{align*}
Using~\cref{assn:gnus} to simplify the numerator, and noting that lower-bounding the denominator, 
\begin{align*}
\forall v \in \partial{h}(\theta), \quad \norm{v}_q & \leq \frac{L_0 + L_1 \, \norm{\nabla f(\theta)}_q}{\norm{\nabla f(\theta)}_q+\frac{L_0}{L_1}} \leq L_1    
\end{align*}
By Lebourg's mean value theorem (\cite{clarke1998nonsmooth}, Thm. 2.4), since $h$ is Lipschitz in an open set containing $y$ and $x$, then there exists a point $u = tx+(1-t)y$, for some $t\in[0,1]$ such that 
\[
\exists g\in \partial h(u), \qquad h(y)-h(x)=\langle g,y-x\rangle.
\]
Therefore,
\begin{align*}
h(y) - h(x) &\leq \max_{t\in[0,1]}  \langle g,y-x\rangle \quad {where} \quad g \in \partial h(t \, x + (1-t) \, y) \\
& \leq \norm{y - x}_p \, \max_{t\in[0,1]}\,\|g\|_q
 \tag{Using Holder's inequality} \\
& \leq L_1 \, \norm{y - x}_p   \tag{Using the above bound on the subgradient} \\
\end{align*}
\begin{align*}
\implies \ln(\norm{\nabla f(y)}_q  + \frac{L_0}{L_1}) - \ln(\norm{\nabla f(x)}_q  + \frac{L_0}{L_1}) & \leq  L_1 \, \norm{y - x}_p  \\
\implies \norm{\nabla f(y)}_q  +\frac{L_0}{L_1} & \leq \left(\norm{\nabla f(x)}_q  + \frac{L_0}{L_1} \right) \, \exp\left(L_1 \norm{y - x}_p\right) 
\end{align*}
\end{proof}

\begin{lemma}
Under~\cref{assn:non-negative},~\cref{assn:gnus} with $L_0 \geq 0, L_1 > 0$, if for all iterations $t$, 
(i) $\norm{\tht - \thtt}_p \leq \frac{a}{L_1}$ for some constant $a$ such that $\sqrt{\beta} \exp(a) < 1$, then, 
\begin{align*}
\frac{\norm{\nabla_t}_1}{\sqrt{1 - \beta} \, \sum_{j = 0}^{t-1} (\sqrt{\beta})^{j} \, \norm{\nabla_{t-j}}_1} & \geq \frac{\norm{\nabla_t}_1}{\norm{\nabla_t}_1 + \frac{L_0}{L_1}} \, \frac{1 - \sqrt{\beta} \, \exp(a)}{\sqrt{1 - \beta}}     
\end{align*}

\label{lemma:rms-ratio-2pl-gnus}
\end{lemma}
\begin{proof}
Using part (1) of~\cref{lemma:grad-multi-lip-gnus} with $p = \infty$, $q = 1$ and $y = \theta_{t-1}$, $x = \tht$, $c > 0$, 
\begin{align*}
\norm{\nabla_{t-1}}_1 + \frac{L_0}{L_1} & \leq (\norm{\nabla_{t}}_1 + \frac{L_0}{L_1}) \, \exp\left(L_1 \,  \norm{\theta_{t-1} - \tht}_\infty \right) \\
& \leq (\norm{\nabla_{t}}_1 + \frac{L_0}{L_1}) \, \exp(a) \tag{Using assumption (i)} \\
\implies \norm{\nabla_{t-j}}_1 & \leq (\norm{\nabla_{t}}_1 + \frac{L_0}{L_1}) \, \exp(a \, j) \tag{Recursing and since $L_0/L_1 > 0$}
\end{align*}
Using the above relation to lower-bound $\sqrt{1 - \beta} \, \sum_{j = 0}^{t-1} (\sqrt{\beta})^{j} \, \norm{\nabla_{t-j}}_1$, first note that, 
\begin{align*}
\sqrt{1 - \beta} \, \sum_{j = 0}^{t-1} (\sqrt{\beta})^{j} \, \norm{\nabla_{t-j}}_1 & \leq \sqrt{1 - \beta} \, \sum_{j = 0}^{t-1} (\sqrt{\beta})^j \, (\norm{\nabla_{t}}_1 + \frac{L_0}{L_1}) \, \exp(a \, j) \\
& \leq \frac{\sqrt{1 - \beta} \; (\norm{\nabla_{t}}_1 + \frac{L_0}{L_1})}{1 - \sqrt{\beta} \, \exp(a)} \tag{Since $\sqrt{\beta} \, \exp(a) < 1$} \\
\implies \frac{\norm{\nabla_t}_1}{\sqrt{1 - \beta} \, \sum_{j = 0}^{t-1} (\sqrt{\beta})^{j} \, \norm{\nabla_{t-j}}_1} & \geq \frac{\norm{\nabla_t}_1}{\norm{\nabla_t}_1 + \frac{L_0}{L_1}} \, \frac{1 - \sqrt{\beta} \, \exp(a)}{\sqrt{1 - \beta}} 
\end{align*}

\end{proof}

\begin{lemma}
If~\cref{assn:non-negative,assn:gnus} hold, for all $y,x$ s.t. $\norm{y - x}_p \leq \frac{1}{L_1}$, 
\[
\norm{\nabla f(y) - \nabla f(x)}_{q} \leq [L_0 + L_1 \, \norm{\nabla f(x)}_q ] \, e \, \norm{y - x}_p.
\]
\label{lemma:grad-lip-gnus}
\end{lemma}
\begin{proof}
By the fundamental theorem of calculus, 
\begin{align*}
\nabla f(y) - \nabla f(x) & = \int_{t = 0}^{1} \nabla^2 f((1-t) \, x + t \, y) \, (y - x) \, dt \\
\implies \norm{\nabla f(y) - \nabla f(x)}_q & = \norm{\int_{t = 0}^{1} \nabla^2 f((1-t) \, x + t \, y) \, (y - x) \, dt}_q \\
& \leq \int_{t = 0}^{1} \, \norm{\nabla^2 f((1-t) \, x + t \, y) \, (y - x)}_{q} dt \tag{Triangle inequality} \\
& \leq \int_{t = 0}^{1} \, \norm{\nabla^2 f((1-t) \, x + t \, y)}_{p \to q} \, \norm{y - x}_p \, dt \tag{By definition of matrix norm} \\
& \leq \norm{y - x}_p \, \left[\int_{t = 0}^{1} \, L_0 + L_1 \, \norm{\nabla f((1-t) \, x + t \, y)}_q \, dt \right] \tag{Using~\cref{assn:gnus}}\\
& = \norm{y - x}_p \, \left[L_0 + L_1 \, \int_{t = 0}^{1} \, \norm{\nabla f((1-t) \, x + t \, y)}_q \, dt \right] \\
& \leq \norm{y - x}_p \, \left[ L_1 \, \int_{t = 0}^{1} \, \left(\norm{\nabla f(x)}_q + \frac{L_0}{L_1}\right)\, \exp(L_1 \, t \, \norm{y - x}) \, dt \right] \tag{Using~\cref{lemma:grad-multi-lip-gnus} with $\theta = ty + (1-t)x$ and $\theta' = x$} \\
& = \norm{y - x}_p \, \left[L_1 \left(\norm{\nabla f(x)}_q + \frac{L_0}{L_1} \right) \, \int_{t = 0}^{1} \, \exp(L_1 \, t \, \norm{y - x}) \, dt \right]  \\
& \leq \norm{y - x}_p \, \left[L_1 \left(\norm{\nabla f(x)}_q + \frac{L_0}{L_1} \right) \, \int_{t = 0}^{1} \, \exp(t) \, dt \right]  \tag{Since $\norm{y - x}_p \leq \frac{1}{L_1}$} \\
& = \left[L_1 \left(\norm{\nabla f(x)}_q + \frac{L_0}{L_1} \right) \, (e-1) \right]\, \norm{y - x}_p \\
& \leq [L_0 + L_1 \, \norm{\nabla f(x)}_q] \,  e \, \norm{y - x}_p 
\end{align*}
\end{proof}

\begin{lemma}
If~\cref{assn:non-negative,assn:gnus} hold, for all $y, x$ s.t. $\norm{y - x} \leq \frac{1}{L_1}$, 
\[
f(y) \leq f(x) + \langle \nabla f(x), y - x \rangle + \left(L_0+ L_1 \, \norm{\nabla f(x)}_q \right) \, \normsq{y - x}_p
\] 
\label{lemma:descent-ineq-gnus}
\end{lemma}
\begin{proof}
Define $u(t) = (1-t) x + t y$ and $g(t) = f(u(t))$. Use Taylor's theorem for $g$,
\begin{align*}
g(b) = g(a) + (b - a) \, g'(a) + \int_{t = a}^{b} (b - t) \, g''(t) \, dt \\
\implies g(1) = g(0) + g'(0) + \int_{t = 0}^{1} (1 - t) \, g''(t) \, dt \tag{Substituting $a = 0$ and $b = 1$} 
\end{align*}
We know that, 
\begin{align*}
g'(t) &= \frac{\partial f(u(t))}{\partial t} = \langle \nabla f(u(t)), y - x \rangle \quad \text{;} \quad g''(t) = \frac{\partial^2 f(u(t))}{\partial t^2} = (y - x)^T \nabla^2 f(u(t))  (y - x)
\end{align*}
Combining the above relations, and using that $g(1) = f(y)$, $g(0) = f(x)$, $g'(0)  = \langle \nabla f(x), y - x \rangle$. 
\begin{align*}
f(y) &= f(x) + \langle \nabla f(x), y - x \rangle + (y-x)^T \left[\int_{t = 0}^{1} (1-t) \, \nabla^2 f(t \, y + (1-t) \, x) \, dt \right]\, (y-x) \\
\intertext{Simplifying the last term,}
& (y-x)^T \left[\int_{t = 0}^{1} (1-t) \, \nabla^2 f(t \, y + (1-t) \, x) \, dt \right]\, (y-x) \\
&\leq \norm{y - x}_p \, \norm{\left[\int_{t = 0}^{1} (1-t) \, \nabla^2 f(t \, y + (1-t) \, x) \, dt \right]\, (y-x)}_{q} \tag{Holder's inequality} \\
& \leq \norm{y - x}_p \, \left[\int_{t = 0}^{1} (1-t) \, \norm{\nabla^2 f(t \, y + (1-t) \, x) \, (y-x)}_{q} \, dt \right] \tag{Triangle inequality} \\
& \leq \normsq{y - x}_p \,  \int_{t = 0}^{1} (1-t) \, \norm{\nabla^2 f(t \, y + (1-t) \, x)}_{p \to q} \, dt \tag{By definition of matrix norm} \\
& \leq \normsq{y - x}_p \,  \int_{t = 0}^{1} (1-t) \, [L_0 + L_1 \, \norm{\nabla f(t \, y + (1-t) \, x)}_q] \, dt \tag{Using~\cref{assn:gnus}} \\
& = \normsq{y - x}_p \,  \int_{t = 0}^{1} (1-t) \, L_1 \, \left [\frac{L_0}{L_1} + \, \norm{\nabla f(t \, y + (1-t) \, x)}_q\right] \, dt\\
& \leq \normsq{y - x}_p \,  \int_{t = 0}^{1} (1-t) \, L_1 \, \left[ \left(\norm{\nabla f(x)}_q + \frac{L_0}{L_1} \right) \exp(L_1 \, t \, \norm{y - x}_p) \right] \, dt \tag{Using~\cref{lemma:grad-multi-lip-gnus} with $\theta = ty + (1-t)x$ and $\theta' = x$} \\
& = \normsq{y - x}_p \,  \int_{t = 0}^{1} (1-t) \, \left((L_0 + L_1 \, \norm{\nabla f(x)}_q) \, \exp(L_1 \, t \, \norm{y - x}_p)  \right)\, dt \\
& = \normsq{y - x}_p \, (L_0 + L_1 \, \norm{\nabla f(x)}_q)\, \int_{t = 0}^{1} (1-t) \, \left( \, \exp(L_1 \, t \, \norm{y - x}_p)  \right)\, dt \\
& \leq  \normsq{y - x}_p \, (L_0 + L_1 \, \norm{\nabla f(x)}_q)\, \int_{t = 0}^{1} (1-t) \, \left( \, \exp(  t   )  \right)\, dt \tag{Since $\norm{y - x}_p \leq \frac{1}{\sqrt{L_1}}$} \\
& \leq \normsq{y - x}_p \, (L_0 + L_1 \, \norm{\nabla f(x)}_q)\, \left[\int_{t = 0}^{1} \exp(t) \, dt - \int_{t = 0}^{1} t \, \exp(t) \, dt \right]  \\
& \leq \normsq{y - x}_p \, (L_0 + L_1 \, \norm{\nabla f(x)}_q)\, (e-2)  \\
& \leq \normsq{y - x}_p \, (L_0 + L_1 \, \norm{\nabla f(x)}_q)\\
\end{align*}
Putting everything together, 
\begin{align*}
f(y) &\leq f(x) + \langle \nabla f(x), y - x \rangle + \left(L_0 + L_1 \, \norm{\nabla f(x)}_q \right) \, \normsq{y - x}_p
\end{align*}
\end{proof}

\begin{lemma}
Under~\cref{assn:non-negative},~\cref{assn:nus},~\cref{assn:gnus} with $L_0 \geq 0$ and $L_1 \geq 0$, for the coordinate-wise Adam update in~\cref{eq:adam-update-coordinate}, with $\eta \leq \frac{1}{B \, L_1}$ and $\beta_1 < \beta_2$, where $B := \sqrt{\frac{1 - \beta_1}{1 - \beta_2}} $, 
\begin{align*}
f(\thtt) & \leq f(\tht) - \eta \, \langle \nabla_t, d_t \rangle +  \left(L_0 + L_1 \, \norm{\nabla f(\tht)}_1 \right) \, \eta^2 \, B^2 \quad \text{;} \quad \norm{d_t}_\infty \leq B
\end{align*}   
\label{lemma:adam-descent-gnus}
\end{lemma}
\begin{proof}
If $\norm{\thtt - \tht}_\infty \leq \frac{1}{L_1}$, using~\cref{lemma:descent-ineq-gnus} with $p = \infty$ and $q = 1$, and denoting $\nabla_t := \nabla f(\tht)$ for convenience,
\begin{align*}
f(\thtt) & \leq f(\tht) - \eta \, \langle \nabla_t, d_t \rangle + \left(L_0 + L_1 \, \norm{\nabla f(\tht)}_1 \right) \, \eta^2 \, \normsq{d_t}_\infty
\end{align*}
Simplifying the third term on the RHS, $\norm{d_t}_\infty = \max_{i} \frac{\ti{m}}{\ti{\sqrt{v}}}$. 
\begin{align}
\ti{m} & = (1-\beta_1) \, \sum_{s = 1}^{t} \beta_1^{t-s} \si{g} \\
\implies \ti{m^2} & \leq \left(1-\beta_1\right)^2 \, \left(\sum_{s = 1}^{t} \beta_1^{t-s}\right) \left(\sum_{s = 1}^{t} \beta_1^{t-s} \si{g^2} \right) \tag{By Cauchy--Schwarz} \\
& \leq \left(1-\beta_1\right) \, \sum_{s = 1}^{t} \beta_1^{t-s} \si{g^2} \tag{By geometric series} \\
& \leq \frac{1 - \beta_1}{1 - \beta_2} \, \left[\left(1-\beta_2\right) \, \sum_{s = 1}^{t} \beta_2^{t-s} \si{g^2} \right] \tag{Since $\beta_1 \leq \beta_2$} \\
&= \frac{1 - \beta_1}{1 - \beta_2} \, \ti{v} \tag{By definition of $\ti{v}$} \\
\implies \frac{\ti{m^2}}{\ti{v}} & \leq \frac{1 - \beta_1}{1 - \beta_2}  \implies  \frac{\ti{m}}{\ti{\sqrt{v}}} \leq B := \sqrt{\frac{1 - \beta_1}{1 - \beta_2} } \\
\implies \norm{\thtt - \tht}_\infty \leq \eta \, \norm{d_t}_\infty \leq \eta \, B \label{eq:adam-1pl-thirdterm-gnus}
\end{align}
Ensuring that $\eta \leq \frac{1}{B \, L_1}$ guarantees that $\norm{\thtt - \tht}_\infty \leq \frac{1}{L_1}$. Combining the above inequalities, 
\begin{align}
f(\thtt) & \leq f(\tht) - \eta \, \langle \nabla_t, d_t \rangle +  \left(L_0 + L_1 \, \norm{\nabla f(\tht)}_1 \right) \, \eta^2 \, B^2 
\end{align}    
\end{proof}

\begin{lemma}
Under~\cref{assn:non-negative} and~\cref{assn:gnus}, if $B=\sqrt{\frac{1-\beta_1}{1-\beta_2}}$ with $L_0 \geq 0$ and $L_1 > 0$, for the Adam update in~\cref{eq:adam-update-coordinate} with $\eta \leq \frac{1}{B \, L_1}$ for all $t$,
\begin{align*}
\sum_i \frac{|\ti{g}| \, |\ti{m} - \ti{g}|}{\ti{\sqrt{v}}} &\leq C \, (L_0 + L_1 \, \norm{\nabla f(\tht)}_1) \quad \text{where,} \quad C = \frac{\beta_1\,e \, B \, \sum_{s = 1}^{t-1} (\sqrt{\beta_1})^{t-1-s} \, \eta_{s}}{\sqrt{1 - \beta_2}}. 
\end{align*}
\label{lemma:adam-mom-bound-gnus}
\end{lemma}
\begin{proof}
\begin{align}
\sum_i \frac{|\ti{g}| \, |\ti{m} - \ti{g}|}{\ti{\sqrt{v}}} &\leq \frac{1}{\sqrt{1 - \beta_2}} \, \sum_i |\ti{m} - \ti{g}| \tag{Since $\ti{v} \geq (1 - \beta_2) \, \ti{g^2}$} \\
& \leq \frac{\norm{m_t - g_t}_1}{\sqrt{1 - \beta_2}} \label{eq:adam-1pl-mom-recursion-1-gnus}
\end{align}
In order to bound $\norm{m_t - g_t}_1$, note that, 
\begin{align}
|\ti{m} - \ti{g}| &= |\beta_1 \, \tpi{m} + (1-\beta_1) \, \ti{g} - \ti{g}| = \beta_1 \, |\tpi{m} - \ti{g}| \\
& \leq \beta_1 \, |\tpi{m} - \tpi{g}| + \beta_1 \, |\tpi{g} - \ti{g}| \nonumber \\
\implies \sum_i |\ti{m} - \ti{g}| & \leq \beta_1 \, \sum_i |\tpi{m} - \tpi{g}| + \beta_1 \, \sum_i |\tpi{g} - \ti{g}| \nonumber \\
\implies \norm{m_t - g_t}_1 & \leq \beta_1 \norm{m_{t-1}  - \nabla_{t-1}}_1 + \beta_1 \norm{\nabla_t - \nabla_{t-1}}_1 \nonumber \\
\intertext{Using~\cref{lemma:grad-lip-gnus} with $p = \infty$, $q = 1$, $y = \theta_{t-1}$ and $x = \theta_t$ while ensuring that $\eta_{t-1} \leq \frac{1}{B \, L_1}$ guarantees that $\norm{\tht - \thtt}_\infty \leq \frac{1}{L_1}$} 
& \leq \beta_1 \, \norm{m_{t-1}  - \nabla_{t-1}}_1 + \beta_1 \, (L_0 + L_1 \, \norm{\nabla f(\tht)}_1) \, e \, \norm{\tht - \theta_{t-1}}_{\infty} \nonumber \\
& \leq \beta_1 \, \norm{m_{t-1}  - \nabla_{t-1}}_{1} + \beta_1 \, \underbrace{(L_0 + L_1 \, \norm{\nabla f(\tht)}_1)}_{:= \bar{L}_t} \, e \, \eta_{t-1} \, B \nonumber \\
\implies \frac{\norm{m_t - g_t}_1}{\bar{L_t}} & \leq \beta_1 \, \frac{\norm{m_{t-1}  - \nabla_{t-1}}_{1}}{\bar{L}_t} +\beta_1\, e \, \eta_{t-1} \, B \label{eq:adam-1pl-mom-recursion-2-gnus}
\end{align}
We now write $\bar{L}_t$ in terms of $\bar{L}_{t-1}$. In particular, using~\cref{lemma:grad-multi-lip-gnus} with $p = \infty$, $q = 1$ and $y = \theta_{t-1}$, $x = \tht$, 
\begin{align}
\bar{L}_t = L_0 + L_1 \, \norm{\nabla f(\tht)}_1 & \geq \left(L_0 + L_1 \, \norm{\nabla f(\theta_{t-1})}_1 \right) \, \exp\left(-L_1 \, \norm{\tht - \theta_{t-1}}_{\infty} \right)\nonumber\\
& = \bar{L}_{t-1} \, \exp\left(-L_1 \, \norm{\tht - \theta_{t-1}}_{\infty} \right) \nonumber\\
\frac{1}{\bar{L}_t} & \leq \frac{1}{\bar{L}_{t-1}} \, \exp\left(L_1 \, \norm{\tht - \theta_{t-1}}_{\infty} \right) \leq \frac{1}{\bar{L}_{t-1}} \, \frac{1}{\sqrt{\beta_1}} \tag{Using~\cref{eq:adam-1pl-thirdterm-gnus} and ensuring that $\eta_{t-1} \leq \frac{1}{B \, L_1} \, \ln(1/\sqrt{\beta_1}$}
\end{align}
Combining the above inequality with~\cref{eq:adam-1pl-mom-recursion-2-gnus}, 
\begin{align}
\frac{\norm{m_t - g_t}_1}{\bar{L_t}} & \leq \sqrt{\beta_1} \, \frac{\norm{m_{t-1}  - \nabla_{t-1}}_{1}}{\bar{L}_{t-1}} + \beta_1\,e \, \eta_{t-1} \, B \\
& \leq \beta_1\,e \, B \, \sum_{s = 1}^{t-1} (\sqrt{\beta_1})^{t-1-s} \, \eta_{s} \tag{Recursing and using that $m_{0,i} = g_{1,i}$} \\
\implies \norm{m_t - g_t}_1 & \leq (L_0 + L_1 \, \norm{f(\tht)}_1) \,\beta_1\, e \, B \, \sum_{s = 1}^{t-1} (\sqrt{\beta_1})^{t-1-s} \, \eta_{s} \label{eq:adam-1pl-mom-recursion-final-gnus}
\end{align}
Combining the above inequality with~\cref{eq:adam-1pl-mom-recursion-1-gnus}, 
\begin{align*}
\sum_i \frac{|\ti{g}| \, |\ti{m} - \ti{g}|}{\ti{\sqrt{v}}}  & \leq (L_0 + L_1 \, \norm{f(\tht)}_1) \, \underbrace{\frac{\beta_1\,e \, B \, \sum_{s = 1}^{t-1} (\sqrt{\beta_1})^{t-1-s} \, \eta_{s}}{\sqrt{1 - \beta_2}}}_{:= C} = C \, (L_0 + L_1 \, \norm{f(\tht)}_1)
\end{align*}
    
\end{proof}

\section{Simplifying and Generalizing the result in~\citet{vaswani2025armijo}}
\label{app:ls}
The update for steepest descent with Armijo line-search can be written as:
\begin{align}
\thtt &= \tht - \etat \, \norm{\nabla_t}_{q} \, d_t \quad \text{where,} \quad d_t := \argmax_{\norm{d}_p \leq 1} \langle d, \nabla_t \rangle \,. \label{eq:sdls} \\
\intertext{Given $c \in (0,1)$ and $\eta_{\max} > 0$, $\etat$ is the largest step-size that satisfies the Armijo condition at iteration $t$, i.e.,}
f(\tht - \eta \, d_t) & \leq f(\tht) - c \, \eta \,  \normsq{\nabla_t}_{q} \quad \text{;} \quad \eta \leq \eta_{\max} \label{eq:sd-armijo}
\end{align}
We will now prove the following lemma that lower-bounds the step-size returned by the Armijo line-search. 
\begin{lemma}
If $f$ satisfies~\cref{assn:non-negative,assn:nus}, at iteration $t$, the update in~\cref{eq:sdls,eq:sd-armijo}  returns a step-size $\etat \geq \min \left\{ \eta_{\max},\frac{1 - c}{H_0 + H_1 \, f(\tht)} \right\}$. 
\label{lemma:armijo-lb}
\end{lemma}
\begin{proof}
We will show that any $\eta \leq \frac{1 - c}{H_0 + H_1 \, f(\tht)}$ will satisfy the Armijo condition, and hence the back-tracking Armijo line-search will return a step-size larger than $\min \left\{ \eta_{\max},\frac{1 - c}{H_0 + H_1 \, f(\tht)} \right\}$. 

First, note that for $\eta \leq \frac{1}{H_0 + H_1 \, f(\tht)}$, 
\begin{align*}
\eta \, \norm{\nabla_t}_{q} & \leq \frac{\norm{\nabla_t}_{q}}{H_0 + H_1 \, f(\tht)} \leq \frac{\sqrt{2 H_0 \, f(\tht) + H_1 \, [f(\tht)]^2}}{H_0 + H_1 \, f(\tht)} \tag{Using~\cref{lemma:nus-grad-gen}} \\
& \leq \frac{1}{\sqrt{H_1}} \, \frac{\sqrt{2 H_0 \, H_1 \, f(\tht) + H_1^2 \, [f(\tht)]^2 + H_0^2}}{H_0 + H_1 \, f(\tht)} \tag{Since $H_0 + H_1 \, f(\tht) > 0$} \\
\implies \eta \, \norm{\nabla_t}_{q}  & \leq \frac{1}{\sqrt{H_1}}
\end{align*}
Hence, for any $\eta \leq \frac{1-c}{H_0 + H_1 \, f(\tht)} < \frac{1}{H_0 + H_1 \, f(\tht)}$, the condition required for~\cref{lemma:descent-ineq-gen} is satisfied for $y = \thtt$, $x = \tht$. Using~\cref{lemma:descent-ineq-gen},
\begin{align*}
f(\thtt) & \leq f(\tht) - \eta \, \norm{\nabla_t}_{q} \, \langle \nabla_t, d_t \rangle + \eta^2 \, (H_0 + H_1 \, f(\tht)) \, \normsq{\nabla_t}_{q} \tag{Since $\norm{d_t}_p = 1$} \\
& =f(\tht) - \eta \, \normsq{\nabla_t}_{q} + \eta^2 \, (H_0 + H_1 \, f(\tht)) \, \normsq{\nabla_t}_{q} \tag{By definition of the dual norm} \\
& \leq f(\tht) - c \, \eta  \, \normsq{\nabla_t}_{q} \tag{By choice of $\eta$} 
\end{align*}
Hence, the Armijo condition is satisfied for any $\eta \leq \frac{1-c}{H_0 + H_1 \, f(\tht)}$. Since the back-tracking line-search returns the largest step-size (smaller than $\eta_{\max}$) that satisfies the Armijo condition, the returned step-size $\etat \geq \max\left\{\eta_{\max}, \frac{1-c}{H_0 + H_1 \, f(\tht)}\right\}$. 
\end{proof}
The above lower-bound on $\etat$ holds for steepest descent and is tighter than the one derived in~\citet{vaswani2025armijo}. Given this lower-bound, the subsequent results in~\citet{vaswani2025armijo} can be derived analogously. 

\end{document}